%% file: online_robust_RL.tex
\documentclass[11pt]{article}
\usepackage{geometry}

\usepackage{fullpage}
\usepackage[usenames]{color}
\usepackage[figuresright]{rotating}
\usepackage{boxedminipage}
\usepackage[colorlinks,citecolor=blue,urlcolor=blue]{hyperref}
\usepackage{epstopdf}
\usepackage{natbib}
\usepackage{multirow}
\usepackage{enumerate}
\usepackage{enumitem}
\usepackage{verbatim}
\usepackage{bbm}
\usepackage{xcolor}
\usepackage{bm}
\usepackage{mathrsfs,color,dsfont}
\usepackage{algorithm,setspace}
\usepackage{algorithmic}
\usepackage{caption}
\usepackage{subcaption}

\usepackage{tikz}
\usetikzlibrary{arrows}

\makeatletter
\pgfdeclareshape{datastore}{
	\inheritsavedanchors[from=rectangle]
	\inheritanchorborder[from=rectangle]
	\inheritanchor[from=rectangle]{center}
	\inheritanchor[from=rectangle]{base}
	\inheritanchor[from=rectangle]{north}
	\inheritanchor[from=rectangle]{north east}
	\inheritanchor[from=rectangle]{east}
	\inheritanchor[from=rectangle]{south east}
	\inheritanchor[from=rectangle]{south}
	\inheritanchor[from=rectangle]{south west}
	\inheritanchor[from=rectangle]{west}
	\inheritanchor[from=rectangle]{north west}
	\backgroundpath{
	\southwest \pgf@xa=\pgf@x \pgf@ya=\pgf@y
	\northeast \pgf@xb=\pgf@x \pgf@yb=\pgf@y
	\pgfpathmoveto{\pgfpoint{\pgf@xa}{\pgf@ya}}
	\pgfpathlineto{\pgfpoint{\pgf@xb}{\pgf@ya}}
	\pgfpathmoveto{\pgfpoint{\pgf@xa}{\pgf@yb}}
	\pgfpathlineto{\pgfpoint{\pgf@xb}{\pgf@yb}}
	}
}
\makeatother

\usepackage{amsthm,amsmath,amssymb,amsfonts, amsbsy, mathtools, epsfig}

\renewcommand{\baselinestretch}{1.6}

\newcommand{\linsps}{\renewcommand{\baselinestretch}{1.6}}

\input{notations}

\newcommand{\supp}{\mathrm{supp}}	

\definecolor{DSgray}{cmyk}{0,1,0,0}

\definecolor{scolor}{cmyk}{0.5,2,0,0}

\newtheorem{lemma}{Lemma}
\newtheorem{theorem}{Theorem}

\newtheorem{corollary}[theorem]{Corollary}
\newtheorem{proposition}[theorem]{Proposition}
\newtheorem{remark}{Remark}

\newtheorem{condition}{Condition}

\begin{document}


\linsps

\title{\bf Online Estimation and Inference for Robust Policy Evaluation in Reinforcement Learning}
	 \author{Weidong Liu\thanks{School of Mathematical Sciences and MoE Key Lab of Artificial Intelligence, Shanghai Jiao Tong University}\quad Jiyuan Tu\thanks{School of Mathematics, Shanghai University of Finance and Economics}\quad Xi Chen\thanks{Stern School of Business, New York University. }\quad Yichen Zhang\thanks{Daniels School of Business, Purdue University}
}	
\date{}
\maketitle

\begin{abstract}
Reinforcement learning has emerged as one of the prominent topics attracting attention in modern statistical learning, with policy evaluation being a key component. Unlike the traditional machine learning literature on this topic, our work emphasizes statistical inference for the model parameters and value functions of reinforcement learning algorithms. While most existing analyses assume random rewards to follow standard distributions, we embrace the concept of robust statistics in reinforcement learning by simultaneously addressing issues of outlier contamination and heavy-tailed rewards within a unified framework. In this paper, we develop a fully online robust policy evaluation procedure, and establish the Bahadur-type representation of our estimator. Furthermore, we develop an online procedure to efficiently conduct statistical inference based on the asymptotic distribution. This paper connects robust statistics and statistical inference in reinforcement learning, offering a more versatile and reliable approach to online policy evaluation. Finally, we validate the efficacy of our algorithm through numerical experiments conducted in simulations and real-world reinforcement learning experiments.
\end{abstract}

\noindent
{\it Keywords}: Statistical inference; dependent samples; online policy evaluation; robust inference; Bahadur representation.

~\\


\section{Introduction}Reinforcement learning has offered immense success and remarkable breakthroughs in a variety of application domains, including autonomous driving, precision medicine, recommendation systems, and robotics (to name a few, e.g., \citealp{murphy2003optimal, kormushev2013reinforcement, mnih2015human, shi2018high}). From recommendation systems to mobile health (mHealth) intervention, reinforcement learning can be used to adaptively make personalized recommendations and optimize intervention strategies learned from retrospective behavioral and physiology data.  While the achievements of reinforcement learning algorithms in applications are undisputed, the reproducibility of its results and reliability is still in many ways nascent. 

Those recommendation and health applications enjoy great flexibility and affordability due to the development of reinforcement algorithms, despite calling for critical needs for a reliable and trustworthy uncertainty quantification for such implementation. The reliability of such implementations sometimes plays a life-threatening role in emerging applications. For example, in autonomous driving, it is critical to avoid deadly explorations based upon some uncertainty measures in the trial-and-error learning procedure. This substance also extends to other applications including precision medicine and autonomous robotics. From the statistical perspective, it is important to quantify the uncertainty of a point estimate with complementary hypothesis testing to reveal or justify the reliability of the learning procedure. 

Policy evaluation plays a cornerstone role in typical reinforcement learning (RL) algorithms such as Temporal Difference (TD) learning. As one of the most commonly adopted algorithms for policy evaluation in RL, TD learning provides an estimator of the value function iteratively with regard to a given policy based on samples from a Markov chain. In large-scale RL tasks where the state space is infinitely expansive, a typical procedure to provide a scalable yet efficient estimation of the value function is via linear function approximation. This procedure can be formulated in a linear stochastic approximation problem \citep{sutton.1988ml, tsitsiklis_vanroy.1997tac, sutton_etal.2009icml, ramprasad2022online}, which is designed to sequentially solve a deterministic equation $A\theta=b$ by a matrix-vector pair of a sequence of unbiased random observations of $(A_t,b_t)$ governed by an ergodic Markov chain. 

The earliest and most prototypical stochastic approximation algorithm is the Robbins-Monro algorithm introduced by \cite{robbins1951stochastic} for solving a root-finding problem, where the function is represented as an expected value, e.g., $\mbE[f(\theta)]=0$. The algorithm has generated profound interest in the field of stochastic optimization and machine learning to minimize a loss function using random samples. When referring to an optimization problem, its first-order condition can be represented as $\mbE[f(\theta)]=0$, and the corresponding Robbins-Monro algorithm is often referred to as first-order methods, or more widely known as stochastic gradient descent (SGD) in machine learning literature. 
It is well established in the literature that its averaged version \citep{ruppert1988efficient,polyak1992acceleration}, as an online first-order method, achieves optimal statistical efficiency when estimating the model parameters in statistical models, which apparently kills the interest in developing second-order methods that use additional information to help the convergence. 
That being said, it is often observed in practice that first-order algorithms entail significant accuracy loss on non-asymptotic convergence as well as severe instability in the choice of hyperparameters, specifically, the stepsizes (a.k.a. learning rate). In addition, the stepsize tuning further complicates the quantification of uncertainty associated with the algorithm output. Despite the known drawbacks above, first-order stochastic methods are historically favored in machine learning tasks due to their computational efficiency while they primarily focus on estimation. On the other hand, when the emphasis of the task lies on statistical inference, certain computation of the second-order information is generally inevitable during the inferential procedure, which shakes the supremacy of first-order methods over second-order algorithms. 

In light of that, we propose a second-order online algorithm which utilizes second-order 
information to perform the policy evaluation sequentially. Meanwhile, our algorithm can be used for conducting statistical inference in an online fashion, allowing for the characterization of uncertainty in the estimation of the value function. Such a procedure generates no extra per-unit computation and storage cost beyond $O(d^2)$, which is at least the same as typical first-order stochastic methods featuring statistical inference (see, e.g., \citealp{chen2020batchmeans} for SGD and \citealp{ramprasad2022online} for TD). More importantly, we show theoretically that the proposed algorithm converges faster in terms of the remainder term compared with first-order stochastic approximation methods, and revealed significant discrepancies in numerical experiments. In addition, the proposed algorithm is free from tuning stepsizes, which has been well-established as a substantial criticism of first-order algorithms. 

Another challenge to the reliability of reinforcement learning algorithms lies in the modeling assumptions. Most algorithms in RL have been in the \emph{optimism in the face of uncertainty}  paradigm where such procedures are vulnerable to manipulation (see some earlier exploration in e.g., \citealp{everitt2017reinforcement,wang2020reinforcement}).  In practice, it is often unrealistic to believe that rewards on the entire trajectory follow exactly the same underlying model. Indeed, non-standard behavior of the rewards happens from time to time in practice. The model-misspecification and presence of outliers are indeed very common in an RL environment, especially that with a large time horizon $T$. It is of substantial interest to design a \emph{robust} policy evaluation procedure. In pursuit of this, our proposed algorithm uses a smoothed Huber loss to replace the least-squares loss function used in classical TD learning, which is tailored to handle both outliers and heavy-tailed rewards. To model outlier observations of rewards in reinforcement learning, we bring the static $\alpha$-contamination model \citep{huber1992robust} to an online setting with dependent samples. In a static offline robust estimation problem, one aims to learn the distribution of interest, where a sample of size $n$ is drawn i.i.d. from a mixture distribution $(1-\alpha_n)P+\alpha_n Q$, and $Q$ denotes an arbitrary outlier distribution. We adopt robust estimation in an online environment, where the observations are no longer independent, and the occurrence time of outliers is unknown. 
In contrast to the offline setting, future observations cannot be used in earlier periods in an online setting. Therefore, in the earlier periods, there is very limited information to help determine whether an observation is an outlier. In addition to this discrepancy between the online decision process and offline estimation, we further allow the outlier reward models to be potentially different for different time $t$ instead of being from a fixed distribution, and such rewards may be arbitrary and even adversarially adaptive to historical information. 
In addition to the outlier model, our model also incorporates rewards with heavy-tailed distributions. This substantially relaxes the boundedness condition on the reward functions that is common in policy evaluation literature. 

We summarize the challenges and contributions of this paper in the following facets.
\begin{itemize}
\item We propose a \emph{fully online} method that leverages dependent samples to simultaneously perform policy evaluation and conduct statistical inference on the model parameters and value functions. Furthermore, we build a Bahadur-type representation of the proposed estimator, which includes the main term corresponding to the asymptotic normal distribution and a higher-order remainder term. Moreover, it shows that our algorithm matches the offline oracle and converges strictly faster to the asymptotic distribution than that of a prototypical first-order stochastic method such as TD learning. 

\item  Compared to existing reinforcement learning literature, our proposed algorithm features an online generalization of the $\alpha_n$-contamination model where the rewards contain outliers or arbitrary corruptions. Our proposed algorithm is robust to adversarial corruptions which can be adaptive to the trajectory, as well as heavy-tailed distribution of the rewards. Due to the existence of outliers, we use a smooth Huber loss where the thresholding parameter is carefully specified to change over time to accommodate the online streaming data. From a theoretical standpoint, a robust policy evaluation procedure forces the update step from $\theta_t$ to $\theta_{t+1}$ to be a non-linear function of $\theta_t$, which brings in additional technical challenges compared to the analysis of classical TD learning algorithms (see e.g., \citealp{ramprasad2022online}) based on linear stochastic approximation. 

\item Our proposed algorithm is based on a dedicated averaged version of the second-order method, where in each iteration a surrogate Hessian is obtained and used in the update step. This second-order information enables the proposed algorithm to be free from stepsize tuning while still ensuring efficient implementation. Furthermore, our proposed algorithm stands out distinctly from conventional first-order stochastic approximation approaches which fall short of attaining the optimal offline remainder rate. On the other hand, while deterministic second-order methods do excel in offline scenarios, they lack the online adaptability crucial for real-time applications.
\end{itemize}

\subsection{Related Works}\label{subsec:related}
Conducting statistical inference for model parameters in stochastic approximation has attracted great interest in the past decade, with a building foundation of the asymptotic distribution of the averaged version of stochastic approximation first established in \cite{ruppert1988efficient,polyak1992acceleration}. The established asymptotic distribution has been brought to conduct online inference. For example, \cite{fang2017scalable} presented a perturbation-based resampling procedure. \cite{chen2020batchmeans} proposed two online procedures to estimate the asymptotic covariance matrix to conduct inference. \cite{chen2021online} studied inference for zeroth-order stochastic approximation. 
\cite{ shi2020statistical} developed online inference procedures for high-dimensional problems. 
Other than those focused on the inference procedures, \cite{shao2022berry} established the rate of convergence in the remainder term, and utilize it to establish Berry-Esseen bounds of the averaged SGD estimator. \cite{givchi_palhang.2015acml, mou_pananjady_etal.2021arxiv} analyzed averaged SGD with Markovian data. Second-order stochastic algorithms were analyzed in \cite{ruppert.1985aos, schraudolph_etal.2007, byrd_hansen_etal.2016} and applied to TD learning in \cite{givchi_palhang.2015acml}.

Under the online decision-making settings including bandit algorithms and reinforcement learning, a few existing works focused on statistical inference of the model parameters or uncertainty quantification of value functions. \cite{deshpande2018accurate,hadad2021confidence,zhang2021statistical,zhang2022statistical} studied statistical inference with adaptively sampled data. 
\cite{zhan2021off,chen2021a,chen2021b, dimakopoulou2021online,chen2022online,han2022online} proposed online inference procedures for    bandit algorithms particularly. \cite{shen2024doubly} studied optimal policy evaluation and constructed valid confidence intervals for the value of the optimal policy. For reinforcement learning algorithms, \cite{thomas2015high}  proposed high-confidence off-policy evaluation based on Bernstein inequality. \cite{hanna2017bootstrapping} presented two bootstrap methods to 
compute confidence bounds for off-policy value estimates. 
\cite{dai2020coindice,feng2021non} construct confidence intervals for value functions based on optimization formulations, and \cite{jiang2020minimax} derived a minimax value interval, both with i.i.d. sample. 
\cite{shi2021statistical} proposed
inference procedures for Q functions in RL via sieve approximations. \cite{hao2021bootstrapping} studied multiplier bootstrap algorithms to offer uncertainty quantification
for exploration in   fitted
Q-evaluation. \cite{syrgkanis_zhan.2023arXiv} studied a re-weighted Z-estimator on episodic RL data and conducted inference on the structural parameter.

The most relevant literature to ours is \cite{ramprasad2022online}, who studied a bootstrap online statistical inference procedure under Markov noise using a quadratic SGD and demonstrated its application in the classical TD (and Gradient TD) algorithms in RL. Our proposed procedure and analysis differ in at least two aspects. First, our proposed estimator is a Newton-type second-order approach that enjoys a faster convergence and optimality in the remainder rate. In addition, we show both analytically and numerically that the computation cost of our procedure is typically lower than \cite{ramprasad2022online} for an inference task. Second, our proposed algorithm is a robust alternative to TD algorithms, featuring a non-quadratic loss function to handle the potential outliers and heavy-tailed rewards. There exists limited RL literature on either outliers or heavy-tailed rewards. Recent works \cite{li_sun.2023arxiv,zhu2024robust} studied online linear stochastic bandits and offline reinforcement learning in the presence of heavy-tailed rewards. However,
their studies do not apply to online statistical inference.

\subsection{Paper Organization and Notations}
The remainder of this paper is organized as follows. In Sections \ref{sec:online_rl} and \ref{sec:online_newton}, we present and discuss our proposed algorithm for robust policy evaluation in reinforcement learning. Theoretical results on convergence rates, asymptotic normality, and the Bahadur representation are presented in Section \ref{sec:online_newton}. In Section \ref{sec:online_infer}, we develop an estimator for the long-run covariance matrix to construct confidence intervals in an online fashion and provide its theoretical guarantee. 
Simulation experiments are provided in Section \ref{sec:sim} to demonstrate the effectiveness of our method. Concluding remarks are given in Section \ref{sec:conclude}. All proofs are deferred to the supplementary material.

For every vector $\vect{v}=(v_1,...,v_d)^{\tp}$, denote $|\vect{v}|_2=\sqrt{\sum_{l=1}^dv_l^2}$, $|\vect{v}|_1=\sum_{l=1}^d|v_l|$, and $|\vect{v}|_{\infty}=\max_{1\leq l\leq d}|v_l|$. For simplicity, we denote $\mbS^{d-1}$ and $\mbB^{d}$ as the unit sphere and unit ball in $\mathbb{R}^d$ centered at $\vect{0}$. Moreover, we use $\supp(\vect{v})=\{1\leq l\leq d\mid v_l\neq 0\}$ as the support of the vector $\vect{v}$. For every matrix $\vect{A}\in\mbR^{d_1\times d_2}$, define $\Norm{\vect{A}}=\sup_{|\vect{v}|_2=1}|\vect{A}\vect{v}|_2$ as the matrix operator norms,  $\Lambda_{\max}(\vect{A})$ and $\Lambda_{\min}(\vect{A})$ as the largest and smallest singular values of $\vect{A}$ respectively. The symbols $\lfloor x\rfloor$ ($\lceil x\rceil$) denote the greatest integer (the smallest integer) not larger than (not less than) $x$. We denote $(x)_{+}=\max(0,x)$. For two sequences $a_n,b_n$, we say $a_n\asymp b_n$ when $a_n=O(b_n)$ and $b_n=O(a_n)$ hold at the same time. We say $a_n\approx b_n$ if $\lim_{n\rightarrow\infty}a_n/b_n=1$. For a sequence of random variables $\{X_n\}_{n=1}^{\infty}$, we denote $X_n=O_{\mbP}(a_n)$ if there holds $\lim_{C\rightarrow\infty}\limsup_{n\rightarrow\infty}\mbP(|X_n|>Ca_n)=0$, and denote $X_n=o_{\mbP}(a_n)$ if there holds $\lim_{C\rightarrow0}\limsup_{n\rightarrow\infty}\mbP(|X_n|>Ca_n)=0$. Lastly, the generic constants are assumed to be independent of $n$ and $d$.

\section{Online Robust Policy Evaluation in Reinforcement Learning}\label{sec:online_rl}

We first review the Least-squares temporal
difference methods in RL. Consider a $4$-tuple $(\mcS, \mcU,\mcP,\mcR )$. Here $\mcS =\{1,2,...,N\}$ is the global
finite state space, $\mcU$ is the set of control (action), $\mcP$ is the transition kernel, and $\mcR $ is the reward function. One of the core steps in RL
is to estimate the cumulative reward $J^{*}$ (which is also called the state value function) for a given policy:
\begin{equation*}
J^{*}(s)=\mbE \Big{[}\sum_{k=0}^{\infty}\gamma^{k}\mcR (s_{k})\;\Big|\; s_{0}=s\Big{]},
\end{equation*}
where $\gamma\in[0,1)$ is a given discount factor and $s\in\mcS$ is any state.
Here $\{s_{k}\}$ denote the environment states which are usually modeled by a Markov chain.
In real-world RL applications, the state space is often very large such that one cannot directly compute value estimates for every state in the state space. A common approach in the modern RL is to approximate the value function $J^{*}(\cdot)$, i.e., let
\begin{equation*}
\tilde{J}(s,\vect{\theta})=\vect{\theta}^{\tp}\vect{\phi}(s)=\sum_{l=1}^{d}\theta_{l}\phi_{l}(s),
\end{equation*}
be a linear approximation of $J^{*}(\cdot)$, where $\vect{\theta}=(\theta_{1},...,\theta_{d})^{\top}\in\mathbb{R}^d$ contains the model parameters, and $\vect{\phi}(s)=(\phi_{1}(s),...,\phi_{d}(s))^{\top}\in\mathbb{R}^d$  is a set of feature vectors 
that corresponds to the state $s\in\mcS$. Here we write $\vect{\phi}_{l}=(\phi_{l}(1),...,\phi_{l}(N))^{\top}$, $1\leq l\leq d$ are linearly independent vectors in $\mbR^{N}$ and $d\ll N$. That is, we use a low dimensional linear approximation (with the basis vectors $\{\vect{\phi}_{1},...,\vect{\phi}_{d}\}$) for $J^{*}(\cdot)$. Let the matrix $\vect{\Phi}=(\vect{\phi}_{1},...,\vect{\phi}_{d})$ and then $\tilde{J}=\vect{\Phi}\vect{\theta}$.

The state value function $J^{*}$  satisfies the Bellman equation
\begin{equation}    \label{eq:bellman_eq}
J^{*}(s)=\mcR (s)+\gamma\mbE \Big{[}\sum_{k=1}^{\infty}\gamma^{k-1}\mcR (s_{k})\Big|s_{0}=s\Big{]}=\mcR (s)+\gamma\sum_{s'=1}^{N}p_{ss'}J^{*}(s'),
\end{equation}
where $s'$ is the next state transferred from $s$. Let $\mcP=(p_{ss'})_{N\times N}$ be the probability transition matrix and define the Bellman operator $\mathcal{T}$ by $\mathcal{T}(Q)=\mcR+\gamma\mcP Q$. The state function $J^{*}$ is the unique fixed point of the Bellman operator $\mathcal{T}$. When $J^{*}$ is replaced by $\tilde{J}=\vect{\Phi}\vect{\theta}$, the Bellman equation may not hold. 

The classical temporal difference attempts to find $\vect{\theta}$ such that $\tilde{J}(s,\vect{\theta})$ well approximates the value function $J^*(s)$. Particularly, the TD algorithm 
updates 
\begin{equation}\label{csr}
\vect{\theta}_{t+1}=\vect{\theta}_{t}-\eta_{t}\vect{\phi}(s_{t})[(\vect{\phi}^{\tp}(s_{t})-\gamma \vect{\phi}^{\tp}(s_{t+1}))\vect{\theta}_{t}-\mcR (s_{t})],\quad t\geq 0,
\end{equation}
where $\eta_t$ is the step size which often requires careful tuning.  We next illustrate  the key point that \eqref{csr} leads to a good estimator such that $\tilde J(s,\vect\theta_t)$ is close to $J^{*}(s)$. It can be shown that 
$\vect{\theta}_{t}$ converges to an unknown population parameter $\vect{\theta}^*$ that minimizes the expected squared difference of the Bellman equation,
\begin{equation}\label{rl_true_para}
\vect{\theta}^{*}=\argmin{\vect{u}\in \mbR^{d}}\mbE |\vect{\phi}^{\tp}(s) \vect{u}-(\mcR (s)+\gamma \vect{\phi}^{\tp}(s') \vect{\theta}^*)|^{2};
\end{equation}
It is easy to see that such $\vect{\theta}^{*}$ exists and satisfies the following first-order condition,
\begin{equation}\label{eqq}
\mbE \vect{\phi}(s)[(\vect{\phi}^{\tp}(s)-\gamma \vect{\phi}^{\tp}(s'))\vect{\theta}^{*}-\mcR (s)]=0.
\end{equation}
Under the condition that $\textbf{H}:=\mbE \vect{\phi}(s)(\vect{\phi}^{\tp}(s)-\gamma \vect{\phi}^{\tp}(s'))$ is positive definite in the sense that $x^{\tp}\textbf{H}x>0$ for all $x\in \mbR^{d}$; see \cite{tsitsiklis_vanroy.1997tac}, the solution of \eqref{eqq} writes
\begin{equation*}
\vect{\theta}^{*}=(\mbE \vect{\phi}(s)(\vect{\phi}^{\tp}(s)-\gamma \vect{\phi}^{\tp}(s')))^{-1}\mbE \vect{\phi}(s)\mcR (s).
\end{equation*}
The estimation equation (\ref{csr}) is a first-order stochastic algorithm that converges to the stochastic root-finding problem \eqref{eqq} using a sequence of observations $\{s_t,r_t\}_{t \geq 1}$. 
In this paper, we refer to it as the Least-squares temporal difference estimator. See \cite{kolter_ng.2009icml} for details on the properties of the Least-squares TD estimator.

Least-squares-based methods are oftentimes criticized due to their sensitivity to outliers in data. When there may exist outliers in some observations of the reward $\mcR (s)$, it is a natural call for interest to design a robust estimator of $\vect{\theta}^{*}$. 
Following the widely-known classical literature \citep{huber.1964, charbonnier.1994, charbonnier_etal.1997tip, hastie2009elements} 
, we replace the square loss $|\cdot|^{2}$ by a smoothed (Pseudo) Huber loss
 $f_{\tau}(x)=\tau^{2}(\sqrt{1+(x/\tau)^{2}}-1)$, parametrized by a thresholding parameter $\tau$.  We define a similar fixed point equation with the smoothed Huber loss by 
\begin{equation}\label{eq:rl_huber_para}
\vect{\theta}^*_{\tau}=\argmin{\vect{u}\in \mbR^{d}}\mbE f_{\tau}\big(\vect{\phi}^{\tp}(s) \vect{u}-(\mcR (s)+\gamma \vect{\phi}^{\tp}(s') \vect{\theta}^*_{\tau})\big).
\end{equation}
In this section, when we motivate the algorithm, we assume that the fixed point $\vect{\theta}^*_{\tau}$  exists{\footnote{Here $\vect{\theta}^*_{\tau}$ is used for the motivation only. Our algorithm and theory do not depend on the existence of $\vect{\theta}^*_{\tau}$. }}. As the thresholding parameter $\tau$ goes to infinity, the objective equation in \eqref{eq:rl_huber_para} is close to the least-squares loss in \eqref{rl_true_para}, and $\vect{\theta}_{\tau}^*$ 
should be close to $\vect{\theta}^{*}$. When $\tau$ tends to $0$, the problem \eqref{eq:rl_huber_para} becomes
$
\vect{\theta}^*_{0}=\argmin{\vect{u}\in \mbR^{d}}\mbE |\vect{\phi}^{\tp}(s) \vect{u}-(\mcR (s)+\gamma \vect{\phi}^{\tp}(s') \vect{\theta}^*_{0})|,
$ with a nonsmooth least absolute deviation (LAD) loss, which is out of the scope of this paper. In this paper, we carefully specify $\tau$ to balance the statistical efficiency and the effect of potential outliers in an online fashion (see Theorem \ref{thm:contam_rate}). 

By the first-order condition of \eqref{eq:rl_huber_para}, we obtain a similar estimation function as \eqref{eqq},
\begin{equation}\label{eq:rl_huber_est}
\mbE\big[ \vect{\phi}(s)g_{\tau}\big((\vect{\phi}^{\tp}(s)-\gamma \vect{\phi}^{\tp}(s'))\vect{\theta}_{\tau}^*-\mcR (s)\big)\big]=0,
\end{equation}
where the function $g_{\tau}(x)=f'_{\tau}(x)$ is the differential of the smoothed Huber loss $f_{\tau}(x)$. Instead of using the first-order iteration with the estimation equation \eqref{eq:rl_huber_est} as in TD (\ref{csr}), we propose a Newton-type iterative estimator, which avoids the tuning of the learning rate. Newton-type estimators are often referred to as second-order methods when discussed in convex optimization. Nonetheless, 
the equation \eqref{eq:rl_huber_para}  is just for illustration purposes, which cannot be directly optimized 
since $\vect{\theta}_{\tau}^*$ appears in the objective function for minimization. On the contrary, our proposed method is a Newton-type method for solving the root-finding problem \eqref{eq:rl_huber_est} using a sequence of the observations.

In this paper, we model two types of noise in observed rewards. The first is the classical Huber contamination model \citep{huber1992robust, huber2004robust}, where an $\alpha_n$-fraction of rewards comes from arbitrary distributions. The second is the heavy-tailed model, where the reward function may admit heavy-tailed distributions. In the following section, we show that our proposed estimator is robust to the aforementioned two noises.


\section{Online Newton-type Method for Parameter Estimation}\label{sec:online_newton}

In the following, we introduce an online Newton-type method for estimating the parameter $\vect{\theta}^*$  in the presence of outliers and heavy-tailed noise in the reward function $\mcR(s)$. For ease of presentation, we denote the observations by $\vect{X}_{i}=\vect{\phi}(s_{i})$, $\vect{Z}_{i}=\vect{\phi}(s_{i})-\gamma \vect{\phi}(s_{i+1})$, and $b_{i}=\mcR (s_{i})$, where $s_{i}$ is the state at time $i$. 
Here $(\vect{X}_i,\vect{Z}_i,b_i)\in\mathbb{R}^d\times \mathbb{R}^d\times \mathbb{R}$ for each observation in the sample. 
Our objective is to propose an online estimator  by the  estimation function \eqref{eq:rl_huber_est}, which can be rewritten as 
\begin{equation}\label{eq:rl_huber_est_rewrite}
\mbE \big[\vect{X}g_{\tau}(\vect{Z}^{\tp}\vect{\theta}_{\tau}^*-b)\big]=0,
\end{equation}based on 
a sequence of dependent observations $
\{(\vect{X}_i,\vect{Z}_i,b_i)\}_{i\geq 1}
$. 

At iteration $n+1$, our proposed Newton-type estimator updates $\hat{\vect{\theta}}_{n+1}$ by
\begin{equation}	\label{eq:online_newton}
	\hat{\vect{\theta}}_{n+1} = \frac{1}{n+1}\sum_{i=0}^n\hat{\vect{\theta}}_i-\hat{\vect{H}}_{n+1}^{-1}\frac{1}{n+1}\sum_{i=0}^n\vect{X}_{i+1}g_{\tau_{i+1}}(\vect{Z}_{i+1}^{\tp}\hat{\vect{\theta}}_i-b_{i+1}).
\end{equation}
Here $\hat{\vect{H}}_{n+1}$ is an
empirical information matrix of the estimation equation (\ref{eq:rl_huber_est_rewrite}), as
 \begin{equation}	\label{eq:Hess_fml}
	\hat{\vect{H}}_{n+1} = \frac{1}{n+1}\sum_{i=0}^n\vect{X}_{i+1}\vect{Z}_{i+1}^{\tp}g_{\tau_{i+1}}'(\vect{Z}_{i+1}^{\tp}\hat{\vect{\theta}}_i-b_{i+1}).
\end{equation}
where $\tau_{i}$ is the thresholding parameter in the Huber loss.  We let $\tau_n$ tend to infinity to eliminate the bias generated by the smoothed Huber loss.

It is noteworthy to mention that the matrix $\hat{\vect{H}}_{n+1}$ is not the Hessian matrix of the objective function on the right-hand side of \eqref{eq:rl_huber_para}. 
As discussed in the previous section, \eqref{eq:rl_huber_para} cannot be directly optimized, nor do they lead to M-estimation problems. Indeed, our proposed update \eqref{eq:online_newton} is a Newton-type method to find the root of \eqref{eq:rl_huber_est_rewrite} using a sequence of observations $\{(\vect{X}_i,\vect{Z}_i,b_i)\}_{i\geq 1}$. In the following section, we will show the use of the matrix $\hat{\vect{H}}_{n+1}$ helps for desirable convergence properties. 

It should be noted that (\ref{eq:online_newton}) can be implemented efficiently in a fully-online manner, i.e., without storage of the trajectory of historical information. Specifically, 
we write \eqref{eq:online_newton} as  $\hat{\vect{\theta}}_{n+1} =\bar{\vect{\theta}}_{n+1} -\hat{\vect{H}}_{n+1}^{-1}\vect{G}_{n+1}$, with each of the item on the right-hand side being a running average. It is easy to see that the averaged estimator $\bar{\vect{\theta}}_{n+1}=\frac{1}{n+1}( n\bar{\vect{\theta}}_{n}+\hat{\vect{\theta}}_{n})$ and 
the vector 
\begin{align}\label{eq:Giterative}
\vect{G}_{n+1}&=\frac{1}{n+1}\sum_{i=0}^n\vect{X}_{i+1}g_{\tau_{i+1}}(\vect{Z}_{i+1}^{\tp}\hat{\vect{\theta}}_i-b_{i+1})\\
&=\frac{n}{n+1}\vect{G}_{n}+\frac{1}{n+1}\vect{X}_{n+1}g_{\tau_{n+1}}(\vect{Z}_{n+1}^{\tp}\hat{\vect{\theta}}_{n}-b_{n+1})\notag
\end{align}
can both be updated online. In addition, the inverse $\hat{\vect{H}}_{n+1}^{-1}$ can be directly and efficiently computed online by the inverse recursion formulation. By the Sherman-Morrison formula, we have
\begin{equation}	\label{eq:riccati_eq}
\begin{aligned}
	\hat{\vect{H}}_{n+1}^{-1} =& \frac{n+1}{n}\hat{\vect{H}}_{n}^{-1}-\frac{n+1}{n^2}\hat{\vect{H}}_{n}^{-1}\vect{X}_{n+1}\\
	&\times\Big[\frac{1}{n}\vect{Z}_{n+1}^{\tp}\hat{\vect{H}}_{n}^{-1}\vect{X}_{n+1}+\{g_{\tau_{n+1}}'(\vect{Z}_{n+1}^{\tp}\hat{\vect{\theta}}_n-b_{n+1})\}^{-1}\Big]^{-1}\vect{Z}_{n+1}^{\tp}\hat{\vect{H}}_{n}^{-1}.
\end{aligned}
\end{equation}
Here we note that both terms $\frac{1}{n}\vect{Z}_{n+1}^{\tp}\hat{\vect{H}}_{n}^{-1}\vect{X}_{n+1}$ and $\{g_{\tau_{n+1}}'(\vect{Z}_{n+1}^{\tp}\hat{\vect{\theta}}_n-b_{n+1})\}^{-1}$ inside the brackets on the right-hand side of \eqref{eq:riccati_eq} are scalars in $\mathbb{R}^1$, not matrices. 

The complete algorithm is presented in Algorithm \ref{alg:onmat}. We refer to it 
as the Robust Online Policy Evaluation (ROPE). Compared with existing stochastic approximation algorithms for TD learning (\citealp{durmus_etal.2021colt, mou_pananjady_etal.2021arxiv, ramprasad2022online}), our ROPE algorithm does not need to tune the step size.

\begin{algorithm}[!t]
	\caption{{\small Robust Online Policy Evaluation ($\mathrm{ROPE}$).}}
	\label{alg:onmat}
	\hspace*{\algorithmicindent}   {\textbf{Input:} Online streaming data $\{(\vect{X}_i,\vect{Z}_i,b_i)\mid i\geq 1\}$ }
	\begin{algorithmic}[1]
            \STATE Compute $\hat{\vect{H}}_{n_0}^{-1}$ and $\vect{G}_{n_0}$ according to \eqref{eq:init_Hhat}  respectively.
		\FOR{$n=n_0+1, n_0+2,\dots$}
		\STATE Specify the thresholding parameter $\tau_n$.
        \STATE Compute $\bar{\vect{\theta}}_n = (n-1)\bar{\vect{\theta}}_{n-1}/n+\hat{\vect{\theta}}_{n-1}/n$, and $(\hat{\vect{H}}_n^{-1},\vect{G}_n)$ by \eqref{eq:Giterative} and \eqref{eq:riccati_eq}.
		\STATE Update the parameter by
			\begin{equation*}
				\hat{\vect{\theta}}_n = \bar{\vect{\theta}}_{n} - \hat{\vect{H}}_n^{-1}\vect{G}_n.
			\end{equation*}
		\ENDFOR
	\end{algorithmic}
	 \hspace{-31.7em}\textbf{Output:}  The final estimator $\hat{\vect{\theta}}_n$.
\end{algorithm}

Since performing iterations of $\hat{\vect{H}}_{n+1}^{-1}$ in \eqref{eq:riccati_eq} requires an initial invertible $\hat{\vect{H}}_{n_0}$, we shall compute from the first $n_0$ samples that
\begin{equation}    \label{eq:init_Hhat}
    \hat{\vect{H}}_{n_0} = \frac{1}{n_0}\sum_{i=1}^{n_0}\vect{X}_i\vect{Z}_{i}^{\top}g'_{\tau_0}(\vect{Z}_{i}^{\top}\hat{\vect{\theta}}_{0}-b_i),\quad     \vect{G}_{n_0} = \frac{1}{n_0}\sum_{i=1}^{n_0}\vect{X}_{i}g_{\tau_0}(\vect{Z}_{i}^{\top}\hat{\vect{\theta}}_{0}-b_i),
\end{equation}
which serves as the initial quantities of \eqref{eq:Giterative}--\eqref{eq:riccati_eq} in order to compute the forthcoming iterations. Here $\hat{\vect{\theta}}_0$ is a given initial parameter, and $\tau_0$ is a specified initial threshold level.


\subsection{Convergence Rate of ROPE}	\label{sec:theory}

Before illustrating how to conduct statistical inference on the parameter $\vect{\theta}^*$, we first provide theoretical results for our proposed ROPE method.

To characterize the weak dependence of the sequence of environment states, we use the concept of $\phi$-mixing. More precisely, we assume $\{s_i,i\geq 1\}$ is a sequence of $\phi$-mixing dependent variables, which satisfy
\begin{equation}	\label{eq:phi_mixing}
	|\mbP(B|A)-\mbP(B)|\leq\phi(k)
\end{equation}
for all $A\in\mcF_1^n, B\in\mcF_{n+k}^{\infty}$ and all $n,k\geq 1$, where $\mcF_a^b=\sigma(s_{i},a\leq i\leq b)$. The $\phi$-mixing dependence covers the irreducible and aperiodic Markov chain, which is typically used to model the states sampling in reinforcement learning (RL) (\citealp{rosenblatt.1956pnas, rio2017asymptotic}, and others). The mixing rate $\phi(k)$ identifies the strength of dependence of the sequence. In this paper, we assume the following condition on the mixing rate.

\begin{condition}   \label{cond:mix}
	The sequence $\{(\vect{X}_i,\vect{Z}_i)\}$ is a stationary $\phi$-mixing sequence, where the mixing rate satisfies $\phi(k)=O(\rho^k)$ for some $0<\rho<1$.
\end{condition}

Condition \ref{cond:mix} holds when $\{s_{i}\}$ is an irreducible and aperiodic Markov chain with finite state space \citep{tsitsiklis_vanroy.1997tac}. While the main techniques of our analysis could be extended to accommodate infinite and even continuous state spaces, we restrict our study to finite state spaces, since our proposed robust estimator targets  
$\vect{\theta}^*$, the fixed point of the projected Bellman equation. The deviation of this fixed point from that of the original Bellman equation, known as the approximation error, is well understood in the literature \citep{tsitsiklis_vanroy.1997tac, bhandari_etal.2018colt} for finite state spaces in TD learning. In contrast, the approximation error under continuous state spaces introduce additional challenges. Since our primary focus is inference for online policy evaluation under heavy-tailed rewards and outliers, we concentrate on establishing the statistical properties of the proposed estimator in the finite-state setting.

In addition, we also assume the following boundedness condition on $\vect{X}=\vect{\phi}(s)$ that is commonly required by the RL literature (\citealp{sutton_barto.2018, ramprasad2022online}).

\begin{condition}   \label{cond:Xbound}
	There exists a constant $C_X>0$ such that $\max\big\{|\vect{\phi}(s)|_2\;\big|\;s\in\mcS\big\}\leq C_X$.
\end{condition}

Condition \ref{cond:Xbound} imposes a uniform bound on the feature space that is independent of $d$, which ensures that our theoretical bounds scale appropriately with the feature dimension. 
Next, we assume the matrix $\vect{H} = \mbE[\vect{X}\vect{Z}^{\tp}]$ has a bounded condition number.

\begin{condition}   \label{cond:cond_num}
	There exists a constant $c>0$ such that  $	    c\leq\Lambda_{\min}(\vect{H})\leq\Lambda_{\max}(\vect{H})\leq c^{-1}$, where $\Lambda_{\max}$ and $\Lambda_{\min}$ denote the largest and the smallest singular values, respectively.
\end{condition}
When $\{s_{i}\}$ is an irreducible and aperiodic Markov chain, \cite{tsitsiklis_vanroy.1997tac} proved that $\vect{H}$ is positive definite, i.e., $x^{\tp}\vect{H}x>0$ for all $x\neq 0$. This indicates that $\vect{H}$ is nonsingular and hence \ref{cond:cond_num} holds. Finally, we provide conditions on the temporal difference error $\vect{Z}^{\tp}\vect{\theta}^*-b$.

\begin{condition}   \label{cond:Bbound}
	The temporal-difference error $\vect{Z}^{\tp}\vect{\theta}^*-b$ comes from the distribution $(1-\alpha_n)P+\alpha_n Q$, where $\alpha_n\in [0,1)$.
 The distribution $Q$ is an arbitrary distribution, and $P$ satisfies that $\mbE_{P}\big[|\vect{Z}^{\tp}\vect{\theta}^*-b|^{1+\delta}\big|\vect{X},\vect{Z}\big]\leq C_b$ holds uniformly for some constants $\delta>0$ and $C_b>0$. 
\end{condition}

 We assume that the temporal-difference error comes from the Huber contamination model. The outlier distribution $Q$ can be arbitrary and the true distribution $P$ has $(1+\delta)$-th order of moment ($\delta>0$), which does not necessarily indicate a variance exists when $\delta\in(0,1)$.
 Therefore, our assumption largely extends the boundedness condition of the reward function $\mcR(s)$ in RL literature (\citealp{sutton_barto.2018, ramprasad2022online}). In Condition \ref{cond:Bbound}, $\alpha_{n}$ controls the ratio of outliers among $n$ samples. Particularly, $\alpha_{n}=m_{n}/n$ means that there are $m_{n}$ outliers among $n$ samples $\{(\vect{X}_{i},\vect{Z}_{i},b_{i}),1\leq i\leq n\}$.
Given the above conditions, by selecting the thresholding parameter $\tau_i$ as defined in \eqref{eq:online_newton}, we can obtain the convergence rate for all the intermediate iterates $\hat{\vect{\theta}}_{i}$.
\begin{theorem} \label{thm:contam_rate}
    Suppose that \ref{cond:mix} to \ref{cond:Bbound} hold and the thresholding parameter $\tau_i =C_{\tau} \max(1,i^{\beta_1}/(\log i)^{\beta_2})$ (where $\beta_1\in[0,1),\beta_2\geq0$, and $C_{\tau}>0$). 
    Assume $n_{0}$ is sufficiently large and the initial value $|\hat{\vect{\theta}}_{0}-\vect\theta^{*}|_2\leq c_{0}$ for some $c_0<1$. Then for every $\nu>0$, there exist constants $C,c>0$ which are independent of the dimension $d$ such that
            \begin{equation*}
                \mbP\Big(\cap_{i=n_0}^n\big\{|\hat{\vect{\theta}}_{i}-\vect{\theta}^*|_2\geq C\sqrt{d}e_i\big\}\Big) \geq 1-cn_0^{-\nu} ,
            \end{equation*}
            where 
            \begin{equation}    \label{eq:conv_g1}
                e_{n} = \alpha_n\tau_{n} + \tau_{n}^{-\min(\delta,2)}+ \sqrt{\frac{\tau_{n}^{(1-\delta)_+}\log n}{n}}+\frac{\tau_n\log^2n}{n}+\frac{1}{\sqrt{d}}(c_0)^{2^{n-n_0}}.
            \end{equation}
     Here $\delta$ is defined in the moment condition in \ref{cond:Bbound}.
\end{theorem}
The error bounds $e_n$ in (\ref{eq:conv_g1}) contains five terms. In the sequel, we will discuss each of these terms individually.   %
The first term $\alpha_{n}\tau_{n}$ comes from the impact of outliers among the samples. Here the contamination ratio $\alpha_{n}=m_{n}/n$ should tend to zero; otherwise, it is impossible to obtain a consistent estimator for $\vect\theta^{*}$. 
The second term $\tau_{n}^{-\min(\delta,2)}$ is due to the bias from smoothed Huber's loss for mean estimation.
The third and the fourth term in (\ref{eq:conv_g1}) is the classical statistical rate due to the variance and boundedness incurred by thresholding, respectively. These two terms commonly appear in the Bernstein-type inequalities (see, e.g. \citealp{Bennett.1962jasa}).
As we can see from the third term, the convergence rate has a phase transition between the regimes of finite variance $\delta> 1$ and infinite variance $0<\delta\leq 1$ (see also Corollary \ref{thm:huber_rate} below). This phenomenon has also been observed in different estimation problems with Huber loss; see \cite{sun_etal.2020jasa} and \cite{fan2019adaptive}. 
In the presence of a higher moment condition, specifically $\delta\geq 5$, we can eliminate the fourth term by applying a more delicate analysis (See Proposition \ref{prop:baha_remain} below and Proposition \ref{prop:contam_rate_d3} in the supplementary material for more details).  The last term represents how quickly the proposed iterative algorithm converges.  Given an initial value $\hat{\vect{\theta}}_{0}$ with an error $c_0<1$, the proposed second-order method enjoys a local quadratic convergence, i.e., the error evolves from $c_0$ to $(c_0)^{2^{n-n_0}}$ after $n-n_0$ iterations. This term decreases super-exponentially fast, a characteristic convergence rate often encountered in the realm of second-order methods (see, e.g., \citealp{nesterov2003introductory}). Specifically, when the sample size, which is equivalent to the number of iterations, is reasonably large, 
the last term $(c_0)^{2^{n-n_0}}= O(1/\sqrt{n})$ is dominated by the statistical error in the other terms of \eqref{eq:conv_g1}. The assumption on the initial error $|\hat{\vect{\theta}}_{0}-\vect\theta^{*}|_2\leq c_{0}$ for some $c_0<1$ is mild. A valid initial value $\hat{\vect{\theta}}_{0}$ can be obtained by the solution to the following estimation equation
$\sum_{i=0}^{n_0}\vect{X}_{i}g_{\tau}(\vect{Z}^{\tp}_{i}\vect{\theta}-b_{i})=0$
using a subsample with fixed size $n_{0}$ before running Algorithm \ref{alg:onmat} with a specified constant threshold $\tau$.
The above equation can be solved efficiently by classical first-order root-finding algorithms. 

To highlight the relationship between the convergence rate and the outlier rewards, we present the following corollary which gives 
a clear statement on how the rate depends on the number of outliers $m_n$. 
\begin{corollary} \label{cor:alpha0_rate}
Suppose the conditions in Theorem \ref{thm:contam_rate} hold with $\delta\geq 2$. Let the thresholding parameter $\tau_{i}=C_{\tau}i^{\beta}$ with some $\beta, C_{\tau}>0$ and $n_0$ tend to infinity. We have
		\begin{equation*}
			 |\hat{\vect{\theta}}_{n}-\vect{\theta}^*|_2 = O_{\mbP}\left(\sqrt{d}\Big(\sqrt{\frac{\log n}{n}}+\frac{\log^2n}{n^{1-\beta}}+\frac{m_{n}}{n^{1-\beta}}+\frac{1}{n^{2\beta}}\Big)\right),
		\end{equation*}
where $m_n$ is the number of outliers among the samples of size $n$. 
\end{corollary}

The corollary shows that, 
when there are $m_{n}=o(n^{1-\beta})$ outliers, our ROPE can still estimate $\vect{\theta}^{*}$ consistently. Note that $\beta$ in $\tau_{i}$ is the exponent specified by the practitioner. A smaller $\beta$ allows more outliers $m_n$ among the samples. Furthermore, when there are $m_{n}=o(n^{\frac12-\beta})$ outliers and $\frac{1}{4}\leq\beta\leq \frac{1}{2}$, the ROPE estimator achieves the optimal rate $|\hat{\vect{\theta}}_{n}-\vect{\theta}^*|_2 = O_{\mbP}\Big(\sqrt{\frac{d\log n}{n}}\Big)$, up to a logarithm term. 

The next corollary indicates the impact of the tail of rewards $\mathcal{R}(s)$ on the convergence rate. For a clear presentation, we discuss the impact under the case without outliers, i.e., $\alpha_n=0$.

\begin{corollary} \label{thm:huber_rate}
    Suppose the conditions in Theorem \ref{thm:contam_rate} hold with the contamination rate $\alpha_n=0$ and let $n_0$ tend to infinity. 
    \begin{itemize}
    	\item When $\delta\in(0,1]$, we specify $\tau_i =C_{\tau} (i/\log i)^{1/(1+\delta)}$. Then 
		\begin{equation*}
			 |\hat{\vect{\theta}}_{n}-\vect{\theta}^*|_2 = O_{\mbP}\left(\sqrt{d}\Big(\frac{\log n}{n}\Big)^{\frac{\delta}{1+\delta}}\right).
		\end{equation*}
	\item When $\delta>1$, we specify $\tau_i =C_{\tau} (i/\log i)^{1/2}$. Then 
		\begin{equation*}
			 |\hat{\vect{\theta}}_{n}-\vect{\theta}^*|_2 = O_{\mbP}\left(\sqrt{d}\Big(\frac{\log n}{n}\Big)^{1/2}\right).
		\end{equation*}
    \end{itemize} 
\end{corollary}
Corollary \ref{thm:huber_rate} illustrates a sharp phase transition between the regimes of infinite variance $\delta\in(0,1]$ and finite variance $\delta>1$, and such transition is smooth and optimal. When $\delta\in(0,1]$, the optimal chocie of $\tau_i$ depends on $\delta$. In practice, when $\delta$ is unknown, a classical inference method on the tail behavior, such as the Hill estimator \citep{hill.1975aos}, can be employed to estimate 
$\delta$. This estimation can be conducted during an initial pilot stage of data collection using a constant thresholding parameter $\tau_0$.

The rate of convergence established in Corollary \ref{thm:huber_rate} matches the offline oracle with independent samples in the estimation of linear regression model \citep{sun_etal.2020jasa}, ignoring the logarithm term. A similar corollary can be established for this phase transition of the convergence rate when the contamination rate  $\alpha_n>0$.

\subsection{Asymptotic Normality and Bahadur Representation}	\label{sec:bahadur}
In the following theorem, we give the Bahadur representation for the proposed estimator $\hat{\vect{\theta}}_{n}$. The Bahadur representation provides a more refined rate for the estimation error.
Moreover, the asymptotic normality can be established by applying the central limit theorem to the main term in the representation.

\begin{theorem} \label{thm:asymp_norm}
	Suppose the conditions in Theorem \ref{thm:contam_rate} hold, and \ref{cond:Bbound} holds with $\delta\geq5$. Let $n_0$ tend to infinity. We have for any nonzero $\vect{v}\in\mbR^{d}$ with $|\vect{v}|_2\leq 1$,
	\begin{align}   \label{eq:bahadur_exp}
		\vect{v}^{\tp}(\hat{\vect{\theta}}_{n}-\vect{\theta}^*) =& \vect{v}^{\tp}\vect{H}^{-1}\frac{1}{n}\sum_{i\notin\mcQ_{n}}\vect{X}_{i}(\vect{Z}^{\tp}_{i}\vect{\theta}^*-b_i) \\
		&+ O_{\mbP}\Big(\sqrt{d}\tau_n\alpha_n+de_{n-1}^2\log^2n+\sqrt{d}\tau_n^{-2}+\frac{\sqrt{d}}{(n\tau_n^2)^{2/5}}\Big).\notag
	\end{align} 
 Here $e_{n}$ is defined in \eqref{eq:conv_g1}, and $\mcQ_n$ denotes the index set of the outliers.
 Moreover, if the contamination rate $\alpha_n = o(1/(\sqrt{n}\tau_n))$ and $n^{1/4}=o(\tau_n)$, then  
    \begin{equation}\label{eq:asymptotic_dist}
        \frac{\sqrt{n}}{\sigma_{\vect{v}}}\vect{v}^{\tp}(\hat{\vect{\theta}}_{n}-\vect{\theta}^{*})\xrightarrow{d} \mcN(0,1),\text{ where }\sigma_{\vect{v}}^2=\vect{v}^{\tp}\vect{H}^{-1}\vect{\Sigma} (\vect{H}^{\tp})^{-1}\vect{v},
    \end{equation}
    as $n\rightarrow \infty$. Here
    \begin{equation}    \label{eq:longrun_cov}
        \vect{\Sigma} = \sum_{k=-\infty}^{\infty}\mbE\big[  \vect{X}_{0}\vect{X}^{\tp}_{k}(\vect{Z}^{\tp}_{0}\vect{\theta}^{*}-b_{0})(\vect{Z}^{\tp}_{k}\vect{\theta}^{*}-b_{k})\big].
    \end{equation}
\end{theorem}

Theorem \ref{thm:asymp_norm} provides a Bahadur-type representation for the proposed estimator $\hat{\vect{\theta}}_{n}$,
decomposing it into a leading term representing an asymptotic linear expansion in \eqref{eq:bahadur_exp} and a higher-order remainder term. This decomposition is instrumental in establishing asymptotic distribution in \eqref{eq:asymptotic_dist} and thereby enabling the construction of valid confidence intervals in an online setting, as demonstrated in Section \ref{sec:online_infer}. Furthermore, in our framework, it facilitates the comparison of second-order convergence rates of the proposed method with those of first-order methods, as we will discuss in Remark \ref{rm:1} below, and may serve as a foundation for Berry–Esseen-type bounds on the convergence rate of the distribution, which we leave as a promising direction for further investigation. Additionally, this representation explicitly characterizes the orders of both the leading and remainder terms in relation to the model's dimensionality. While the Bahadur representation was originally introduced in the context of quantile regression, we adopt its use following the recent literature on Huber-type loss problems \citep{sun_etal.2020jasa}, motivated by its role in robustness in reinforcement learning.

To achieve asymptotic normality, additional conditions $\alpha_n = o(1/(\sqrt{n}\tau_n))$ and $n^{1/4}=o(\tau_n)$ are required on the specification of the thresholding parameters $\tau_n$. These conditions easily hold when the practitioner specifies $\tau_{n}=C_{\tau}n^{\beta}$ with $\beta>1/4$ and the number of outliers satisfies $m_{n}=o(n^{1/2-\beta})$. 
We further note that to establish the Bahadur representation in the above theorem, we require the sixth moment to exist (i.e., $\delta=5$). This condition may be weakened and we leave this theoretical question to future investigation. It is worth noting that even for the case where there is no contamination, i.e., $\alpha_{n}=0$, our result is still new. To our knowledge, there is no literature that establishes the Bahadur representation for the TD method. To highlight the rate of convergence concisely in the remainder term, we provide the following corollary under $\alpha_{n}=0$. 

\begin{proposition}   \label{prop:baha_remain}
    Suppose the conditions of Theorem \ref{thm:asymp_norm} hold with $\alpha_n=0$ and $\tau_i =C_{\tau}i^{\beta}$ for $\beta\geq 3/4$  and some $C_{\tau}>0$, we have for any nonzero $\vect{v}\in\mbR^{d}$ with $|\vect{v}|_2\leq 1$,
    \begin{equation*}
        \vect{v}^{\tp}(\hat{\vect{\theta}}_{n}-\vect{\theta}^*) = \vect{v}^{\tp}\vect{H}^{-1}\frac{1}{n}\sum_{i=1}^{n}\vect{X}_{i}(\vect{Z}^{\tp}_{i}\vect{\theta}^*-b_i) + O_{\mbP}\Big(\frac{d\log n}{n}\Big).
    \end{equation*}
\end{proposition}

In this proposition, we specify $\beta\geq 3/4$ to ensure that the remainder term possesses an order of $O_{\mbP}\big(d\log n/n\big)$. Notably, this result is not a direct corollary of Theorems  \ref{thm:contam_rate} and \ref{thm:asymp_norm}, since the rate of $e_{n-1}$ in \eqref{eq:conv_g1} of Theorem \ref{thm:contam_rate} is not sharp enough to guarantee the term $de_{n-1}^2\log^2n$ in \eqref{eq:bahadur_exp} converges as fast as $O\big(d\log n/n\big)$ in Proposition \ref{prop:baha_remain}. To reach this fast rate of the remainder, we establish an improved rate of Theorem \ref{thm:contam_rate} under the stronger assumptions, achieved by eliminating the fourth term of \eqref{eq:conv_g1}. The detailed formal result of the improved rate of $e_n$ is relegated to Proposition \ref{prop:contam_rate_d3} in the supplementary material.

\begin{remark}[Acceleration of Convergence]\label{rm:1}
    We discuss the order of the remainder term in the Bahadur representation and further demonstrate that our proposed method converges strictly faster to the asymptotic distribution than that of a prototypical first-order stochastic method such as TD learning. Existing analysis on the TD learning \citep{ramprasad2022online} does not provide a Bahadur representation. Therefore, we compare it with its i.i.d. version instead. For estimators obtained with i.i.d. sample using SGD, the proof of \cite{polyak1992acceleration} indicates that the remainder term of SGD with learning $cn^{-\eta}$, $0<\eta\leq 1$ is at the order of $O_{\mbP}(1/n^{1-\eta/2}+1/n^{\eta})$, which is upper bounded by $O_{\mbP}(n^{-\frac{2}{3}})$, when the dimension $d$ is fixed. Our proposed method achieves a strictly faster rate of remainder, $O(n^{-1}\log n)$. 
\end{remark}


\section{Estimation of Long-Run Covariance Matrix and Online Statistical Inference}\label{sec:online_infer}

In the above section, we provide an online Newton-type algorithm to estimate the parameter $\vect{\theta}^*$. As we can see in Theorem \ref{thm:asymp_norm},
the proposed estimator has an asymptotic variance with a sandwich structure $\vect{H}^{-1}\vect{\Sigma} (\vect{H}^{\tp})^{-1}$. To conduct statistical inference simultaneously with ROPE method we proposed in Section \ref{sec:online_newton}, we propose an online plug-in estimator for this sandwich structure.  Note that the online estimator $\widehat{\vect{H}}^{-1}_{n}$ of $\vect{H}^{-1}$ is proposed and utilized in \eqref{eq:riccati_eq} of the estimation algorithm. It remains to construct an online estimator for $\vect{\Sigma}$ in \eqref{eq:longrun_cov}.

Developing an online estimator of $\vect{\Sigma}$ with dependent samples is intricate compared to the case of independent ones. As the samples are dependent, the covariance matrix $\vect{\Sigma}$ is an infinite sum of the series of time-lag covariance matrices. 
To estimate the above long-run covariance matrix $\vect{\Sigma}$, we first rewrite the definition of $\vect{\Sigma}$ in \eqref{eq:longrun_cov} into the following form
\begin{align*}
\vect{\Sigma} =&\mbE\big[  \vect{X}_{0}\vect{X}^{\tp}_{0}(\vect{Z}^{\tp}_{0}\vect{\theta}^{*}-b_{0})^2\big]+\sum_{k=-\infty}^{-1}\mbE\big[  \vect{X}_{0}\vect{X}^{\tp}_{k}(\vect{Z}^{\tp}_{0}\vect{\theta}^{*}-b_{0})(\vect{Z}^{\tp}_{k}\vect{\theta}^{*}-b_{k})\big]\\
&+\sum_{k=1}^{\infty}\mbE\big[  \vect{X}_{0}\vect{X}^{\tp}_{k}(\vect{Z}^{\tp}_{0}\vect{\theta}^{*}-b_{0})(\vect{Z}^{\tp}_{k}\vect{\theta}^{*}-b_{k})\big].
\end{align*}
Next, we replace each term $\mbE\big[  \vect{X}_{0}\vect{X}^{\tp}_{k}(\vect{Z}^{\tp}_{0}\vect{\theta}^{*}-b_{0})(\vect{Z}^{\tp}_{k}\vect{\theta}^{*}-b_{k})\big]$ in \eqref{eq:longrun_cov} with its empirical counterpart using the $n$ samples. Meanwhile, to handle outliers, we replace $(\vect{Z}^{\tp}_{i}\vect{\theta}^{*}-b_{i})$ by $g_{\tau_i}(\vect{Z}^{\tp}_{i}\hat{\vect{\theta}}_{i-1}-b_{i})$. In addition, the infinite sum in \eqref{eq:longrun_cov} is not feasible for direct estimation. Nonetheless,  Condition \ref{cond:mix} on the mixing rate of $\{(\vect{X}_t,\vect{Z}_t)\}$ allows us to approximate it by effectively estimating the first $\lceil\lambda\log n\rceil$ terms only, where $\lambda$ is a pre-specified constant. In summary, we can construct the following estimator for the covariance matrix $\vect{\Sigma}$,
\begin{align}
\hat{\vect{\Sigma}}_{n}=&\frac{1}{n}\sum_{i=1}^{n}\vect{X}_{i}\vect{X}^{\tp}_{i}\big(g_{\tau_i}(\vect{Z}^{\tp}_{i}\hat{\vect{\theta}}_{i-1}-b_{i})\big)^2\label{eq:hatsigma}\\
 &+\frac{1}{n}\sum_{i=1}^{n}\sum_{k=1}^{\lceil\lambda\log i\rceil\wedge (i-1)}\vect{X}_{i}\vect{X}^{\tp}_{i-k}g_{\tau_i}(\vect{Z}^{\tp}_{i}\hat{\vect{\theta}}_{i-1}-b_{i})g_{\tau_{i-k}}(\vect{Z}^{\tp}_{i-k}\hat{\vect{\theta}}_{i-k-1}-b_{i-k})\notag\\
&+\frac{1}{n}\sum_{i=1}^{n}\sum_{k=1}^{\lceil\lambda\log i\rceil\wedge (i-1)}\vect{X}_{i-k}\vect{X}^{\tp}_{i}g_{\tau_{i-k}}(\vect{Z}^{\tp}_{i-k}\hat{\vect{\theta}}_{i-k-1}-b_{i-k})g_{\tau_i}(\vect{Z}^{\tp}_{i}\hat{\vect{\theta}}_{i-1}-b_{i}).\notag
\end{align}

To  perform a fully-online update for the estimator $\hat{\vect{\Sigma}}_{n}$, we define
\begin{equation*}
\vect{S}_{j}=\sum_{i=1}^{j}\vect{X}^{\tp}_{i}g_{\tau_i}(\vect{Z}^{\tp}_{i}\hat{\vect{\theta}}_{i-1}-b_{i}),\quad j\geq 1,
\end{equation*}
and the above equation \eqref{eq:hatsigma} can be rewritten as, 
\begin{align*}
\hat{\vect{\Sigma}}_{n}=&\frac{1}{n}\sum_{i=1}^{n}\vect{X}_{i}\vect{X}^{\tp}_{i}g_{\tau_i}^{2}(\vect{Z}^{\tp}_{i}\hat{\vect{\theta}}_{i-1}-b_{i})+\frac{1}{n}\sum_{i=1}^{n}\vect{X}_{i}g_{\tau_i}(\vect{Z}^{\tp}_{i}\hat{\vect{\theta}}_{i-1}-b_{i})(\vect{S}_{i-1}-\vect{S}_{i-\lceil\lambda\log i\rceil\wedge(i-1)})\\
&+\frac{1}{n}\sum_{i=1}^{n}(\vect{S}_{i-1}-\vect{S}_{i-\lceil\lambda\log i\rceil\wedge(i-1)})^{\tp}\vect{X}^{\tp}_{i}g_{\tau_i}(\vect{Z}^{\tp}_{i}\hat{\vect{\theta}}_{i-1}-b_{i}).
\end{align*}
In practice, at the $n$-th step, we keep only the terms $\{\vect{S}_{n},...,\vect{S}_{n-\lceil\lambda\log n\rceil\wedge(n-1)}\}$ in memory, and the covariance matrix can be updated by 
\begin{align*}
\hat{\vect{\Sigma}}_{n}=\frac{1}{n}\Big{[}&\vect{X}_{n}\vect{X}^{\tp}_{n}g_{\tau_n}^{2}(\vect{Z}^{\tp}_{n}\hat{\vect{\theta}}_{n-1}-b_{n})+\vect{X}_{n}g_{\tau_n}(\vect{Z}^{\tp}_{n}\hat{\vect{\theta}}_{n-1}-b_{n})(\vect{S}_{n-1}-\vect{S}_{n-\lceil\lambda\log n\rceil\wedge(n-1)})\\
&+(\vect{S}_{n-1}-\vect{S}_{n-\lceil\lambda\log n\rceil\wedge(n-1)})^{\tp}\vect{X}^{\tp}_{n}g_{\tau_n}(\vect{Z}^{\tp}_{n}\hat{\vect{\theta}}_{n-1}-b_{n})
\Big{]}+\frac{n-1}{n}\hat{\vect{\Sigma}}_{n-1}.\stepcounter{equation}\tag{\theequation}\label{eq:sigma_hat_est}
\end{align*}
The proposed estimator $\hat{\vect{\Sigma}}_{n}$ complements the estimator $\hat{\vect{H}}_{n}^{-1}$ in \eqref{eq:riccati_eq} to construct an estimator of the sandwich form $\vect{H}^{-1}\vect{\Sigma} (\vect{H}^{\tp})^{-1}$ that appears in the asymptotic distribution in Theorem \ref{thm:asymp_norm}. 

Specifically, we can construct the confidence interval for the parameter $\vect{v}^{\tp} \vect{\theta}^*$ using the asymptotic distribution in \eqref{eq:asymptotic_dist} in the following way. For any unit vector $\vect{v}$, a confidence interval with nominal level $(1-\xi)$ is
\begin{equation}	\label{eq:conf_intv}
	\Big[\vect{v}^{\tp}\hat{\vect{\theta}}_n-q_{1-\xi/2}\hat{\sigma}_{\vect{v}},\vect{v}^{\tp}\hat{\vect{\theta}}_n+q_{1-\xi/2}\hat{\sigma}_{\vect{v}}\Big],
\end{equation}
where $\hat{\sigma}_{\vect{v}}^2=\vect{v}^{\tp}\hat{\vect{H}}_{n}^{-1}\hat{\vect{\Sigma}}_n (\hat{\vect{H}}_n^{\tp})^{-1}\vect{v}$  and $q_{1-\xi/2}$ denotes the $(1-\xi/2)$-th quantile of a standard normal distribution.

\begin{remark}
We provide some insights into the computational complexity of our inference procedure. Both estimators $\hat{\vect{H}}_{n}^{-1}$ and $\hat{\vect{\Sigma}}_{n}$ can be computed in an $O(d^2)$ per-iteration computation cost, as discussed in \eqref{eq:riccati_eq}. In addition, the computation $\hat{\sigma}_{\vect{v}}^2=\vect{v}^{\tp}\hat{\vect{H}}_{n}^{-1}\hat{\vect{\Sigma}}_n (\hat{\vect{H}}_n^{\tp})^{-1}\vect{v}$ requires three vector-matrix multiplication with the same computation complexity $O(d^2)$. On the contrary, the online bootstrap method proposed in \cite{ramprasad2022online} has a per-iteration complexity of at least $O(Bd)$, where the resampling size $B$ in the bootstrap is usually much larger than $d$ in practice.
\end{remark}

The construction in \eqref{eq:conf_intv} serves also for the uncertainty quantification of the value function $\vect{\phi}(s)^\top\vect{\theta}^*$ if one specifies $\vect{v}=\vect{\phi}(s)$. Furthermore, the constructed confidence interval can be utilized to implement early stopping, reducing costs associated with data collection. One potential approach is to adopt the early stopping rule outlined in \cite{xia_khamaru_etal.2023jmlr}, which is based on predefined tolerance probability, stopping checkpoints and target accuracy level. This allows the robust policy evaluation algorithm to terminate efficiently once the desired precision is achieved, optimizing resource usage.

In the following theorem, we provide the theoretical guarantee of the confidence interval construction by showing that our online estimator of the covariance matrix, $\hat{\vect{\Sigma}}_{n}$, is consistent.

\begin{theorem} \label{prop:cov_est}
Suppose the conditions in Theorem \ref{thm:asymp_norm} hold.  The covariance estimator $\hat{\vect{\Sigma}}_n$ defined in \eqref{eq:sigma_hat_est} satisfies
\begin{equation*}
\|\hat{\vect{\Sigma}}_{n}-\vect{\Sigma}\|=O_{\mbP}\left(\sqrt{d}\tau_n^2\alpha_n+  \sqrt{d}\tau_n^{-1}+\tau_n\sqrt{\frac{d\log n}{n}}+ \frac{d\tau_n^2\log^2n}{n}\right).
\end{equation*}
\end{theorem}
Theorem \ref{prop:cov_est} provides an upper bound on the estimation error $\hat{\vect{\Sigma}}_{n}-\vect{\Sigma}$. To achieve the consistency on $\hat{\vect{\Sigma}}_{n}$, we specify the thresholding parameter  $\tau_n = C_{\tau}n^{\beta}$ for $1/4<\beta<1/2$, such that
\[
\|\hat{\vect{\Sigma}}_{n}-\vect{\Sigma}\|=O_{\mbP}\left(\sqrt{d}\tau_n^2\alpha_n+  \frac{\sqrt{d\log n}}{n^{1/2-\beta}}\right).
\] 
In this case, as long as the fraction of outliers $\alpha_n$ satisfies $\alpha_n = o(1/(n^{2\beta}\sqrt{d}))$, we obtain a consistent estimator of the matrix $\vect{\Sigma}$. In other words, our proposed robust algorithm allows $o(n^{1-2\beta})$ outliers, ignoring the dimension.
We summarize the result in the following corollary.

\begin{corollary} \label{cor:conf_intv}
    Suppose the conditions of Theorem \ref{thm:asymp_norm} hold, and the fraction of outliers satisfies $\alpha_n = o(1/(n^{2\beta}\sqrt{d}))$, where we specify $\tau_i = C_{\tau}i^{\beta}$ and $1/4<\beta<1/2$.  Then given a vector $\vect{v}\in\mbR^d$ and a pre-specified confidence level $1-\xi$, we have
    \begin{equation*}
        \lim_{n\rightarrow\infty}\mbP\left(\vect{v}^{\tp}\vect{\theta}^*\in\Big[\vect{v}^{\tp}\hat{\vect{\theta}}_n-q_{1-\xi/2}\hat{\sigma}_{\vect{v}},\vect{v}^{\tp}\hat{\vect{\theta}}_n+q_{1-\xi/2}\hat{\sigma}_{\vect{v}}\Big]\right)=1-\xi,
    \end{equation*}
    where $\hat{\sigma}_{\vect{v}}^2=\vect{v}^{\tp}\hat{\vect{H}}_{n}^{-1}\hat{\vect{\Sigma}}_n (\hat{\vect{H}}_n^{\tp})^{-1}\vect{v}$  and $q_{1-\xi/2}$ denotes the $(1-\xi/2)$-th quantile of a standard normal distribution.
\end{corollary}


\section{Numerical Experiments}	\label{sec:sim}
In this section, we assess the performance of our ROPE algorithm in numerical experiments. We construct all confidence intervals with a nominal coverage of $95\%$, and compare our method with the online bootstrap method with the linear stochastic approximation proposed in \cite{ramprasad2022online}, which we refer to as LSA. 

\subsection{Parameter Inference for Infinite-Horizon MDP}	\label{sec:ihmdp}

In the first experiment, we focus on an infinite-horizon Markov Decision Process (MDP) setting. Specifically, we create an environment with a state space of size $50$ and an action space of size $5$. The dimension of the features is fixed at $10$. The transition probability kernel of the MDP, the state features are randomly generated from $\mcN(\vect{0}_p,\mbI_p)$, and the policy under evaluation is generated by picking an action uniformly from the action space for each state. By employing the Bellman equation \eqref{eq:bellman_eq}, we can compute the expected rewards at each state under the policy. To introduce variability, we add noise to the expected rewards, drawn from different distributions such as the standard normal distribution and $t$ distribution with a degree of freedom of $2.25$. The main objective of this experiment is parameter inference. Particularly, we construct a confidence interval for the first coordinate of the true parameter $\vect{\theta}^*$ to study the effect of the first feature on the value function.

We set the thresholding level as $\tau_i = C(i/(\log i)^2)^{\beta} $ where $\beta$ is a positive constant. We begin with an investigation of the influence of the parameters $C$ and $\beta$ on the coverage probability and width of the confidence interval in our ROPE algorithm. Figure \ref{fig:ihmdp_inner} displays the results, revealing that the performance is relatively robust to variations in $C$ and $\beta$, irrespective of the type of noise applied. Notably, although our theoretical guarantees are based on the noises with a $6$-th order moment (See Theorem \ref{thm:asymp_norm}), the experiment indicates that our method exhibits a wider range of applicability.

In subsequent experiments, we specify $C=0.5$ and $\beta=1/3$ for the ROPE algorithm and compare it with the LSA method. Alongside the coverage probability and width of the confidence interval, we also assess the average running time for both methods. The findings are depicted in Figure \ref{fig:ihmdp_outer}. It is evident that, under a light-tailed noise that admits normal distribution, the two methods yield comparable results, with the ROPE method consistently outperforming LSA in terms of the confidence interval width. When the noise exhibits heavy-tailed characteristics, the confidence interval width of LSA is notably larger than that of ROPE. Additionally, the runtime of ROPE is much shorter than that of LSA. 

\begin{figure}[t]
\begin{subfigure}[t]{0.42\textwidth}
    \includegraphics[width=\linewidth]{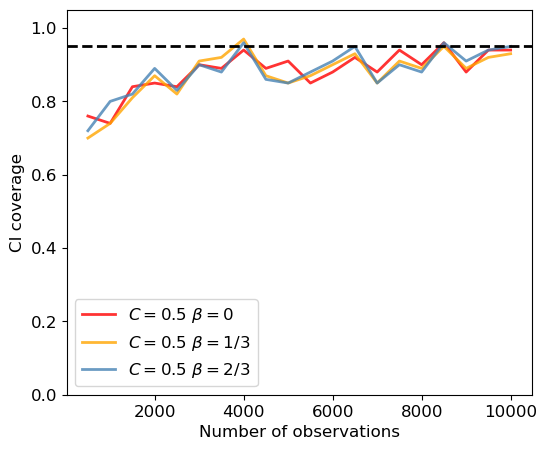}
\end{subfigure}
\begin{subfigure}[t]{0.42\textwidth}
  \includegraphics[width=\linewidth]{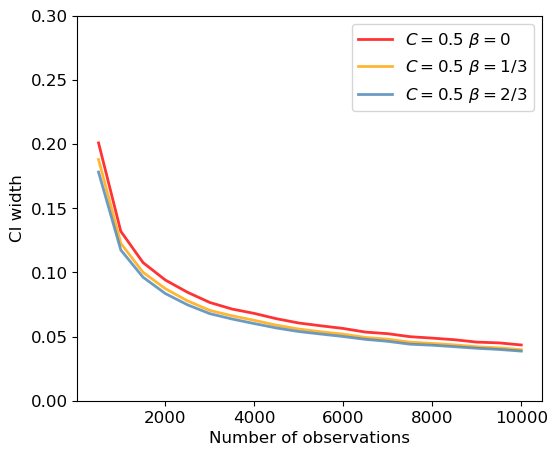}
\end{subfigure}\hfill

\begin{subfigure}[t]{0.42\textwidth}
    \includegraphics[width=\linewidth]{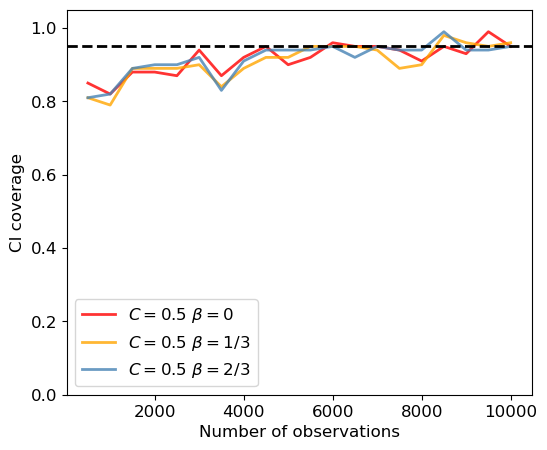}
\end{subfigure}
\begin{subfigure}[t]{0.42\textwidth}
    \includegraphics[width=\linewidth]{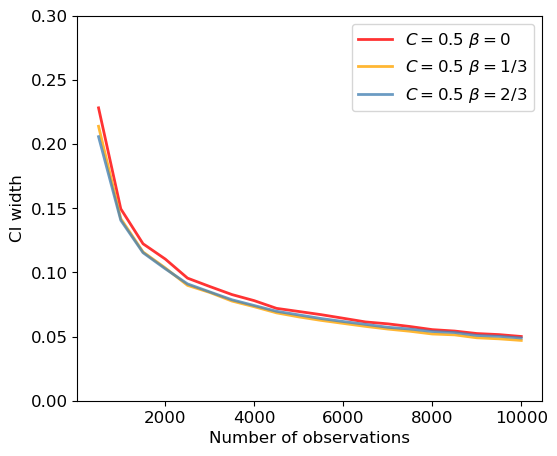}
\end{subfigure}\hfill

\caption{Coverage probability (the first column) and the width of confidence interval (the second column) of $\mathrm{ROPE}$ of various $C$ and $\beta$. We specify the noise distribution as standard normal (the first row) and Student's $t_{2.25}$ (the second row).}
\label{fig:ihmdp_inner}
\end{figure}

In the LSA method, the step size is determined by the expression $\alpha i^{-\eta}$, which relies on two positive hyperparameters $\alpha$ and $\eta$. In this experiment, we investigate the sensitivity of LSA method on these parameters. For comparison, we also present the result of our ROPE algorithm, which does not require any hyperparameter tuning. In this particular experiment, we specify the noise distribution to be standard normal, and directly apply the online Newton step on the square loss (as opposed to the pseudo-Huber loss). Notably, Figure \ref{fig:ihmdp_lsasens} illustrates that the LSA method is sensitive to the parameter $\alpha$. Specifically, when $\alpha$ takes on larger values, the LSA method generates significantly wider confidence intervals than that of LSA with well-tuned hyperparameters, which is undesirable in practical applications. Meanwhile, our ROPE method always generates confidence intervals with a valid width and comparable coverage rate. 

\subsection{Value Inference for FrozenLake RL Environment}	\label{sec:fl}

In this section, we consider the FrozenLake environment provided by OpenAI gym, which involves a character navigating an $8\times8$ grid world. The character's objective is to start from the first tile of the grid and reach the end tile within each episode. If the character successfully reaches the target tile, a reward of $1$ is obtained; otherwise, the reward is $0$. In our setup, we generate the state features uniformly from $[0,1]^p$, with a dimensionality of $p=4$. The policy under evaluation is pre-trained using $Q$-learning, and the true parameter can be explicitly computed using the transition probability matrix. Under a contaminated reward model, we introduce random perturbations by replacing the true reward with a value uniformly sampled from the range $[0,100]$ with probability $\alpha\in\{0,n^{-1},0.05n^{-1/2}\}$. Our goal is to infer the initial state value.
\begin{figure}[htp]
\begin{subfigure}[t]{0.325\textwidth}
    \includegraphics[width=\linewidth]{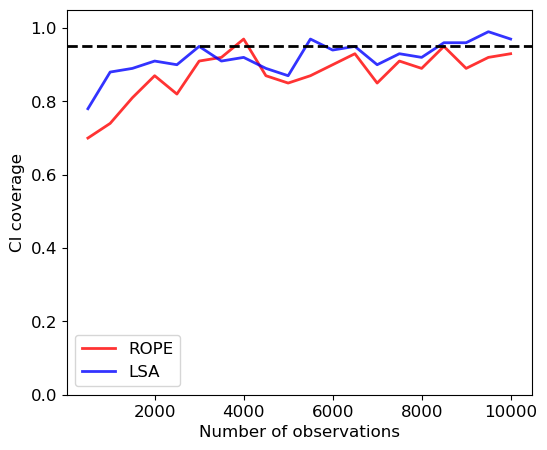}
\end{subfigure}\hfill
\begin{subfigure}[t]{0.325\textwidth}
  \includegraphics[width=\linewidth]{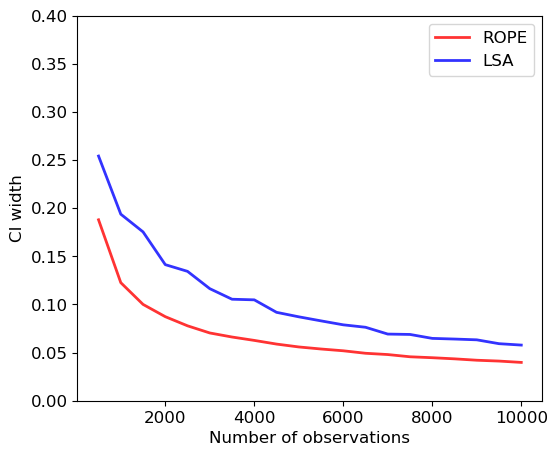}
\end{subfigure}\hfill
\begin{subfigure}[t]{0.325\textwidth}
    \includegraphics[width=\linewidth]{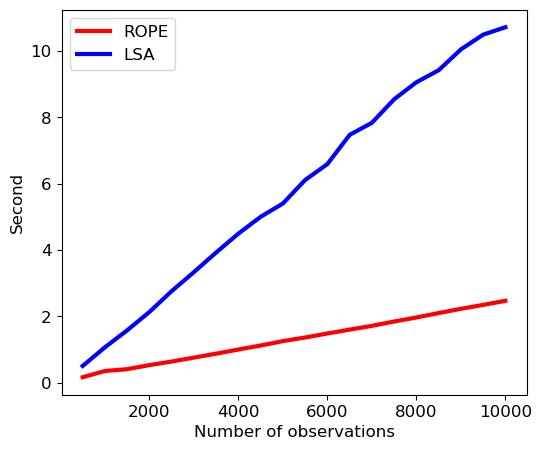}
\end{subfigure}\hfill

\begin{subfigure}[t]{0.325\textwidth}
    \includegraphics[width=\linewidth]{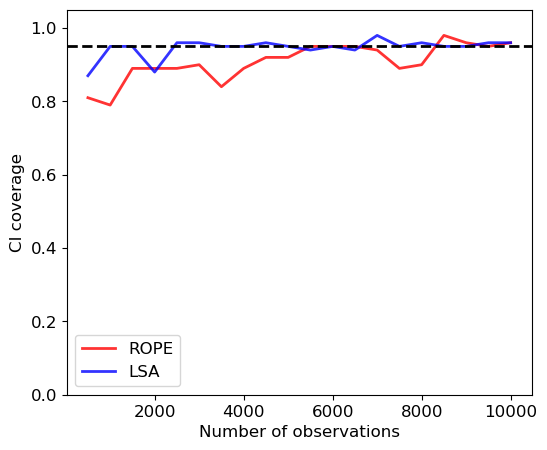}
\end{subfigure}\hfill
\begin{subfigure}[t]{0.325\textwidth}
    \includegraphics[width=\linewidth]{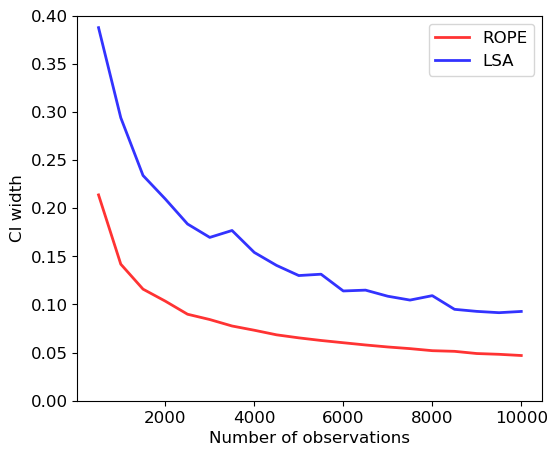}
\end{subfigure}\hfill
\begin{subfigure}[t]{0.325\textwidth}
    \includegraphics[width=\linewidth]{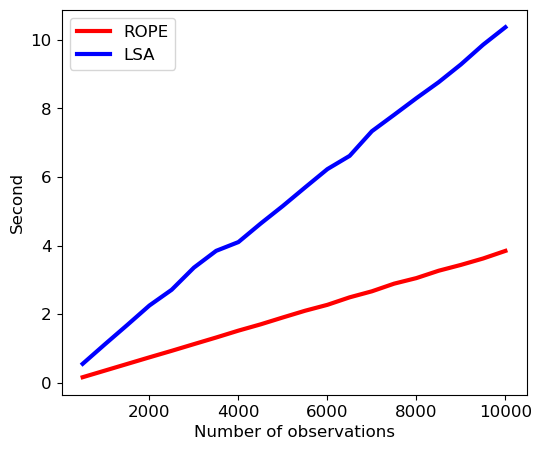}
\end{subfigure}\hfill

\caption{Coverage probability (the first column), the width of confidence interval (the second column), and computing time (the third column) of $\mathrm{ROPE}$ and $\mathrm{LSA}$. We specify the noise distribution as standard normal (the first row) and Student's $t_{2.25}$ (the second row).}
\label{fig:ihmdp_outer}
\end{figure}

\begin{figure}[htp]
\begin{subfigure}[t]{0.49\textwidth}
    \includegraphics[width=\linewidth]{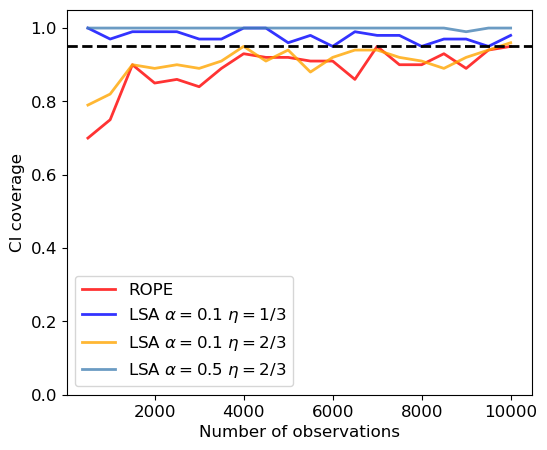}
\end{subfigure}
\begin{subfigure}[t]{0.49\textwidth}
  \includegraphics[width=\linewidth]{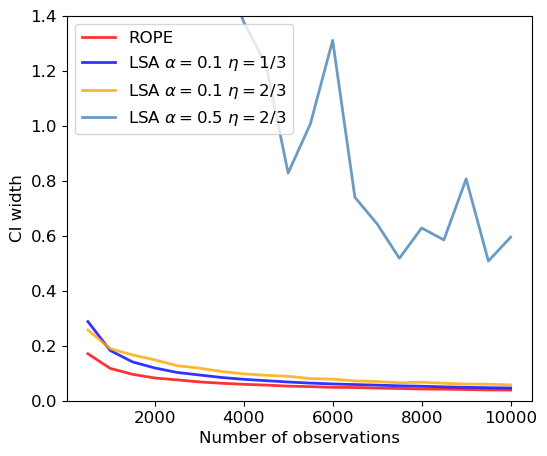}
\end{subfigure}\hfill

\caption{Coverage probability (left), and the width of confidence interval (right) of $\mathrm{ROPE}$ and $\mathrm{LSA}$ of various $\alpha$ and $\eta$. We specify the noise distribution as standard normal.}
\label{fig:ihmdp_lsasens}
\end{figure}

\begin{figure}[htp]
\begin{subfigure}[t]{0.325\textwidth}
    \includegraphics[width=\linewidth]{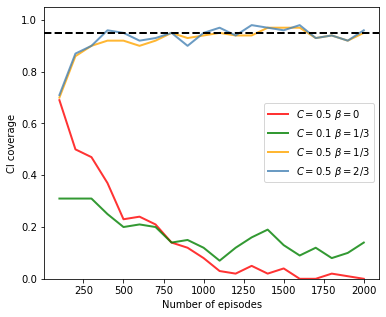}
\end{subfigure}\hfill
\begin{subfigure}[t]{0.325\textwidth}
  \includegraphics[width=\linewidth]{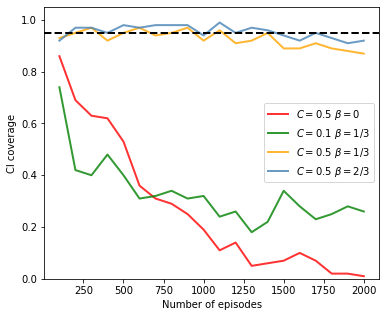}
\end{subfigure}\hfill
\begin{subfigure}[t]{0.325\textwidth}
    \includegraphics[width=\linewidth]{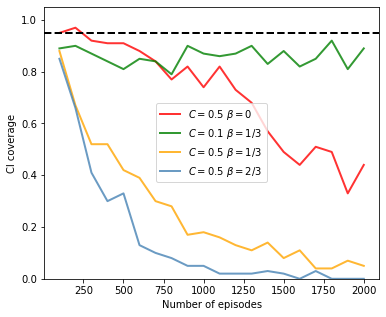}
\end{subfigure}

\begin{subfigure}[t]{0.325\textwidth}
    \includegraphics[width=\linewidth]{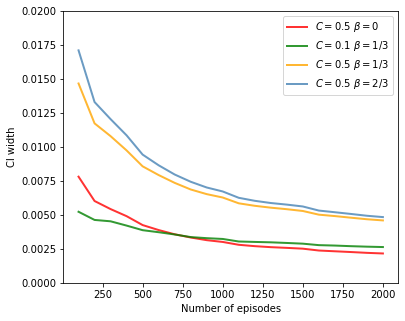}
\end{subfigure}\hfill
\begin{subfigure}[t]{0.325\textwidth}
    \includegraphics[width=\linewidth]{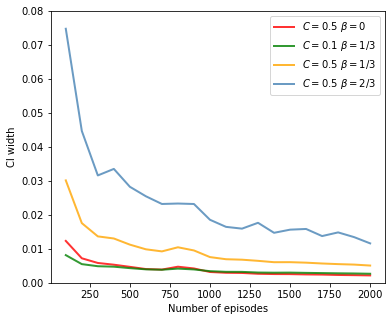}
\end{subfigure}\hfill
\begin{subfigure}[t]{0.325\textwidth}
    \includegraphics[width=\textwidth]{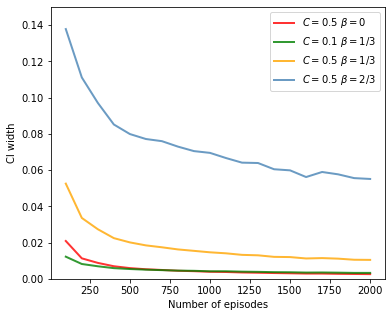}
\end{subfigure}

\caption{Coverage probability (the first row) and the width of confidence interval (the second row) of $\mathrm{ROPE}$ of various $C$ and $\beta$. We set the contamination rate to be $0$ (the first column), $n^{-1}$ (the second column), and $0.05n^{-1/2}$ (the third column), respectively.}
\label{fig:fl_inner}
\end{figure}

\begin{figure}[htp]
\begin{subfigure}[t]{0.325\textwidth}
    \includegraphics[width=\linewidth]{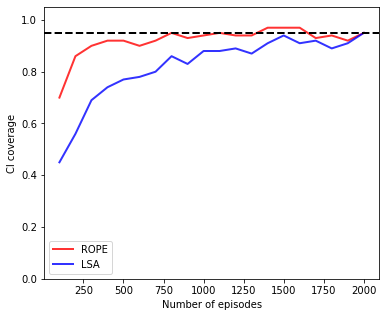}
\end{subfigure}\hfill
\begin{subfigure}[t]{0.325\textwidth}
  \includegraphics[width=\linewidth]{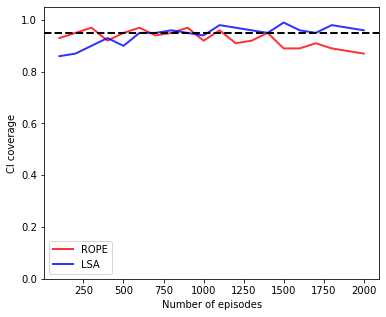}
\end{subfigure}\hfill
\begin{subfigure}[t]{0.325\textwidth}
    \includegraphics[width=\linewidth]{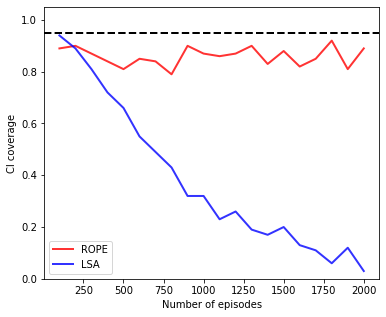}
\end{subfigure}

\begin{subfigure}[t]{0.325\textwidth}
    \includegraphics[width=\linewidth]{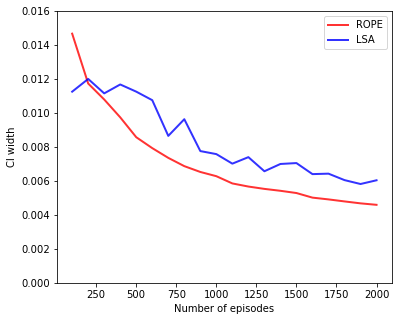}
\end{subfigure}\hfill
\begin{subfigure}[t]{0.325\textwidth}
    \includegraphics[width=\linewidth]{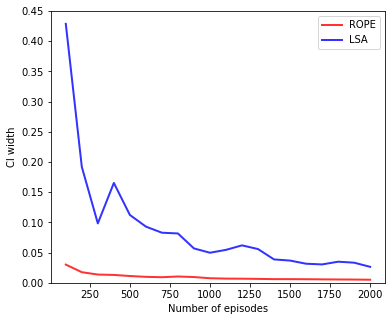}
\end{subfigure}\hfill
\begin{subfigure}[t]{0.325\textwidth}
    \includegraphics[width=\textwidth]{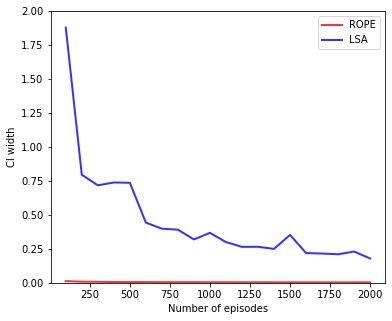}
\end{subfigure}

\begin{subfigure}[t]{0.325\textwidth}
    \includegraphics[width=\linewidth]{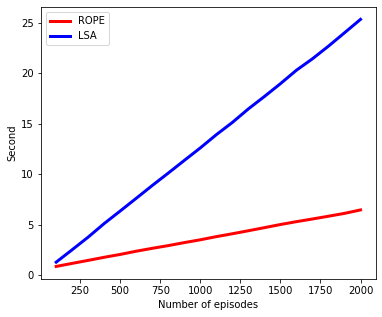}
\end{subfigure}\hfill
\begin{subfigure}[t]{0.325\textwidth}
    \includegraphics[width=\linewidth]{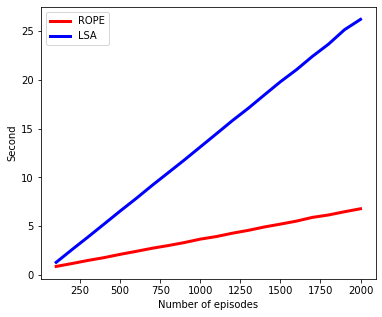}
\end{subfigure}\hfill
\begin{subfigure}[t]{0.325\textwidth}
    \includegraphics[width=\linewidth]{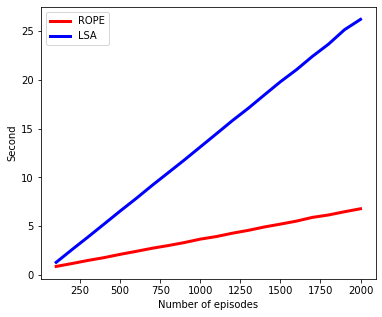}
\end{subfigure}
\caption{Coverage probability (the first row), the width of confidence interval (the second row), and computing time (the third row) of $\mathrm{ROPE}$ and $\mathrm{LSA}$. We set the contamination rate to be $0$ (the first column), $n^{-1}$ (the second column), and $0.05n^{-1/2}$ (the third column), respectively.}
\label{fig:fl_outer}
\end{figure}

\begin{figure}[htp]
\begin{subfigure}[t]{0.42\textwidth}
  \includegraphics[width=\linewidth]{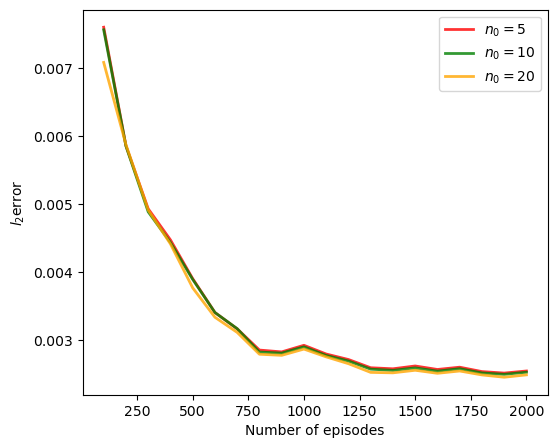}
\end{subfigure}
\begin{subfigure}[t]{0.42\textwidth}
  \includegraphics[width=\linewidth]{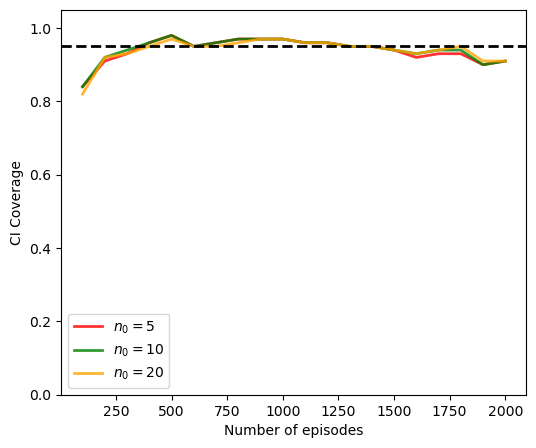}
\end{subfigure}
\hfill
\caption{The $\ell_2$ error (left) and coverage probability (right) of $\mathrm{ROPE}$ with different initialization using 5, 10, 20 episodes. We set the contamination rate to be $0.05n^{-3/4}$.}
\label{fig:fl_init}
\end{figure}

Following the approach in Section \ref{sec:ihmdp}, we proceed to examine the sensitivity of ROPE to the thresholding parameters $C$ and $\beta$. We present the results in Figure \ref{fig:fl_inner}, and observe that the performance of ROPE is influenced by both the contamination rate $\alpha_n$ and the chosen thresholding parameter. Specifically, when $\alpha_n$ is small, a larger thresholding parameter tends to yield better outcomes. However, as $\alpha_n$ increases, the use of a large thresholding parameter can negatively impact performance.

Subsequently, we set the values $C=0.5$, $\beta=1/3$ when $\alpha_n=0$ or $n^{-1}$, and $C=0.1,\beta=1/3$ when $\alpha_n=0.05n^{-1/2}$ for ROPE algorithm, and compare its performance with that of the LSA method. Figure \ref{fig:fl_outer} displays the results, which clearly demonstrates that our ROPE method consistently outperforms LSA in terms of coverage rates, CI widths, and running time. The advantage of ROPE becomes more pronounced as the contamination rate $\alpha_n$ increases.

In the next experiment, we investigate the impact of the initialization value $n_0$ on the performance of the ROPE algorithm. We consider the corruption rate $0.05n^{-3/4}$, while varying the number of episodes in the initialization by $\{5,10,20\}$. Each episode contains an average of $38$ observations, leading to corresponding values of $n_0$ approximately $\{190,380,760\}$. The parameters are set to $C=0.5$, $\beta=1/3$. The results in Figure \ref{fig:fl_init} indicate that the value of $n_0$ has a minimal effect on $\ell_2$-errors and coverage rates.

\subsection{Online Policy Evaluation on MIMIC-III Dataset}
In this section, we demonstrate the effectiveness of the proposed ROPE algorithm by applying it to the Medical Information Mart for Intensive Care version III (MIMIC-III) database \citep{johnson_etal.2016}. MIMIC-III is a freely accessible database containing de-identified critical care data collected from 2001 to 2012 across six ICUs at a Boston teaching hospital. Following \cite{Komorowski_etal.2018}, we select the data pertaining to adult sepsis patients and extract a set of 43 features to characterize each patient, including demographics, Elixhauser premorbid status, vital signs, and laboratory values. We then apply feature clustering, resulting in 752 distinct states. Each state is assigned a feature vector, computed as the mean of the features within that state. We follow the approach of \cite{Komorowski_etal.2018} and \cite{prasad2017reinforcement} to assign rewards to each state based on patient health metrics and mortality outcomes. The reward at the end of each episode reflects hospital mortality, while the rewards assigned to intermediate observations incorporate physiological parameters such as heart rate, respiratory rate, and arterial pH \citep{prasad2017reinforcement}. The terminal reward has a substantially larger magnitude than the intermediate rewards, naturally inducing a heavy-tailed reward distribution. This design is common in healthcare applications, as it emphasizes critical but infrequent outcomes of great clinical importance. Our objective is to perform \emph{on-policy} evaluation and construct confidence intervals for the value function corresponding to the initial state.

\begin{figure}[htp]
\begin{subfigure}[t]{0.325\textwidth}
    \includegraphics[width=\linewidth]{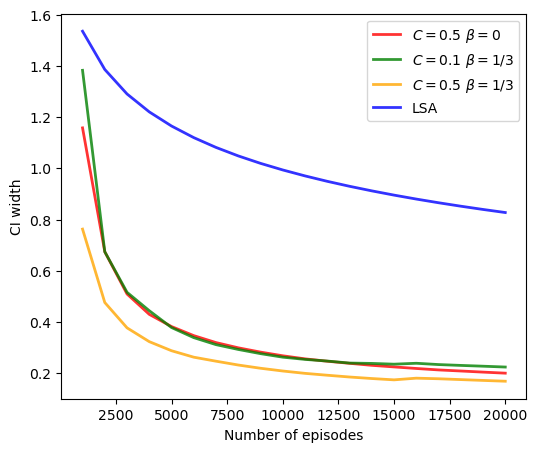}
\end{subfigure}\hfill
\begin{subfigure}[t]{0.325\textwidth}
  \includegraphics[width=\linewidth]{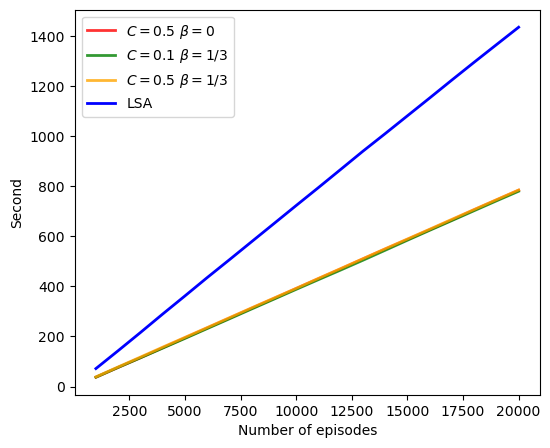}
\end{subfigure}\hfill
\begin{subfigure}[t]{0.325\textwidth}
    \includegraphics[width=\linewidth]{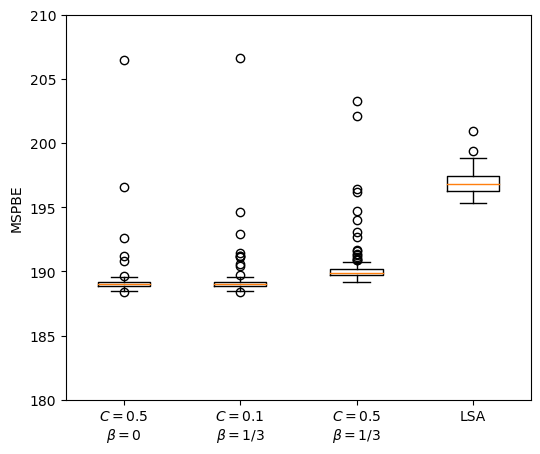}
\end{subfigure}

\caption{The width of confidence interval (the first column), computational time (the second column), and MSPBE (the third column) of $\mathrm{ROPE}$ and $\mathrm{LSA}$ for the MIMIC-III dataset.}
\label{fig:mimic}
\end{figure}

The final processed data contains 20943 episodes, comprising a total 278598 observations. We compare the LSA method with our ROPE method using various $(C,\beta)$ pairs: $(0.5,0),(0.1,1/3)$, and $(0.5,1/3)$, where ROPE is initialized with $100$ episodes. The discount factor is set to $\gamma = 0.99$, and the experiment is repeated 100 times. Since the true value is unknown in real-data analysis, we report only the evolution of confidence interval width and computational time as the number of episodes increases. Additionally, we present a box plot of the Mean-Squared Projected Bellman Error (MSPBE) in the final episode. The results, shown in Figure \ref{fig:mimic}, demonstrates that our proposed ROPE method outperforms LSA across all metrics. In particular, we observe that the unrobustified LSA method yields considerably wider confidence intervals than our ROPE method, and the dispersion of MSPBE values is also larger for LSA. These findings underscore the necessity of robust method in the presence of heavy-tailed noise.

\section{Concluding Remarks}	\label{sec:conclude}

This paper introduces a robust online Newton-type algorithm designed for policy evaluation in reinforcement learning under the influence of heavy-tailed noise and outliers. We demonstrate the estimation consistency as well as the asymptotic normality of our estimator. We further establish the Bahadur representation and show that it converges strictly faster than that of a prototypical first-order method such as TD. To further conduct statistical inference, we propose an online approach for constructing confidence intervals. Our experimental results highlight the efficiency and robustness of our approach when compared to the existing online method. 

While our current work focuses on TD learning for on-policy evaluation, our robust framework can be extended to Gradient Temporal-Difference (GTD) algorithms for off-policy evaluation. There is also potential to apply our approach to multi-step TD methods and TD($\lambda$), though this may introduce challenges due to increased dependencies. Additionally, we identify an extension to fitted-Q evaluation as a promising direction for future research. Beyond policy evaluation, extending our methodology to the reinforcement learning algorithms with policy optimization (e.g., deep Q-learning) presents significant potential but also introduces challenges, which we leave as open questions for future investigation.

\bibliographystyle{asa}
\bibliography{online_robust_RL}

\newpage

\section{Appendix}


\subsection{Experiment of the Effect of Thresholding parameters}

In this section, we extend the analysis from the infinite-horizon Markov Decision Process setting discussed in Section \ref{sec:ihmdp} by examining the effect of thresholding parameter $\tau$ on our ROPE method. As stated in Theorem \ref{thm:contam_rate}, the thresholding parameter is defind as $\tau = C_{\tau} \max(1,i^{\beta_1}/(\log i)^{\beta_2}),$ and is determined by three factors $C_{\tau},\beta_1$ and $\beta_2$. In our experiments, we investigate how variations in each of these factors affect performance. We adopt the ideal parameter triple $(C_{\tau},\beta_1,\beta_2) = (0.5,1/3,2/3)$ and add only $t_{2.25}$ noise to the reward. In each experiment, one parameter is varied while keep the other two remain fixed, and we report the corresponding coverage probability and confidence interval width. The results are presented in Figure \ref{fig:ihmdp_thred} below.

\begin{figure}[htp]
\begin{subfigure}[t]{0.325\textwidth}
    \includegraphics[width=\linewidth]{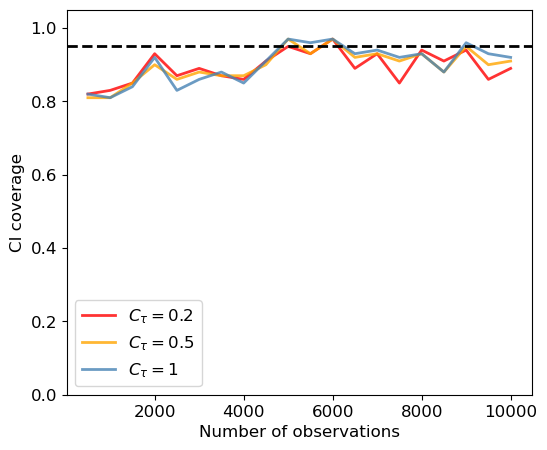}
\end{subfigure}\hfill
\begin{subfigure}[t]{0.325\textwidth}
  \includegraphics[width=\linewidth]{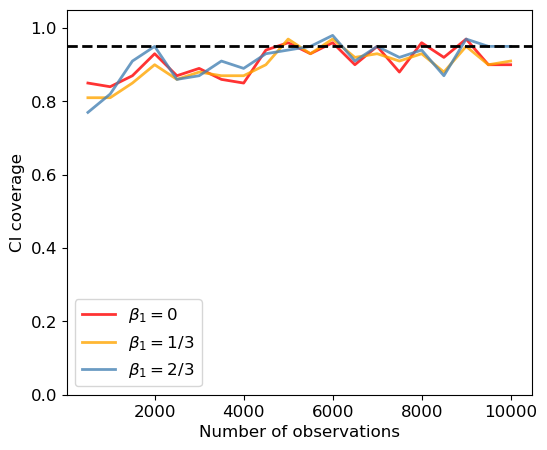}
\end{subfigure}\hfill
\begin{subfigure}[t]{0.325\textwidth}
    \includegraphics[width=\linewidth]{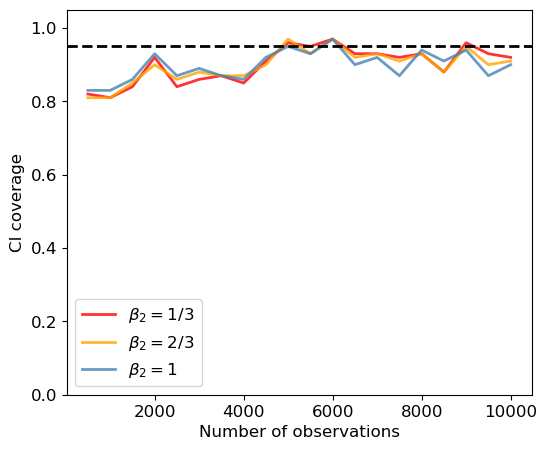}
\end{subfigure}\hfill

\begin{subfigure}[t]{0.325\textwidth}
    \includegraphics[width=\linewidth]{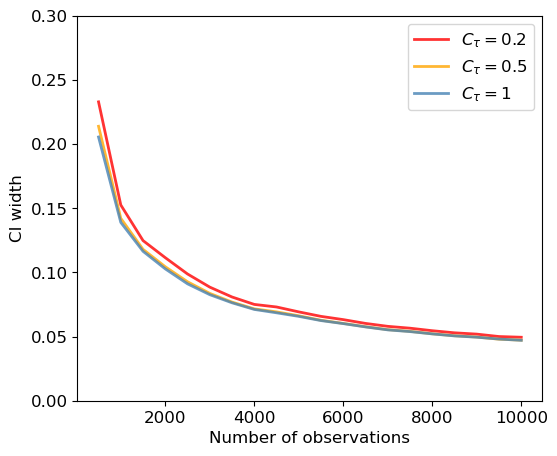}
\end{subfigure}\hfill
\begin{subfigure}[t]{0.325\textwidth}
    \includegraphics[width=\linewidth]{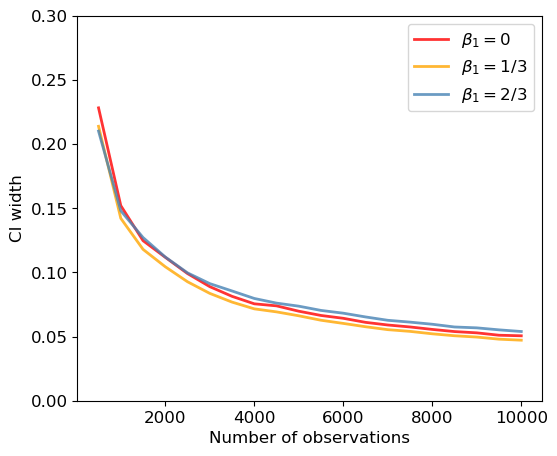}
\end{subfigure}\hfill
\begin{subfigure}[t]{0.325\textwidth}
    \includegraphics[width=\linewidth]{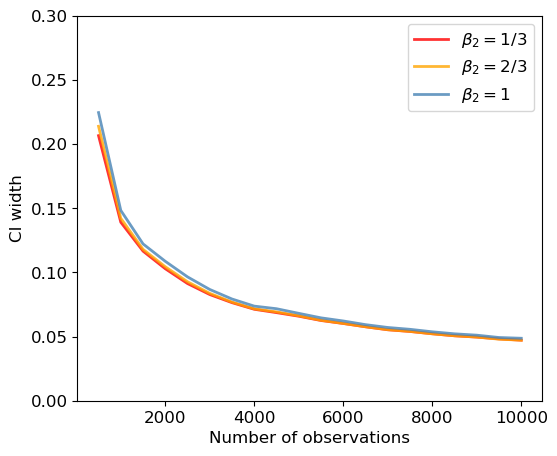}
\end{subfigure}\hfill

\caption{Coverage probability (the first row) and the width of confidence interval (the second row) $\mathrm{ROPE}$ under different thresholding parameters. We vary $C_{\tau}$ (the first column), $\beta_1$ (the second column) and $\beta_2$ (the third column) while other factors are held constant.}
\label{fig:ihmdp_thred}
\end{figure}

From our experimental results, we observe that the factors $C_{\tau},\beta_1$ and $\beta_2$ have a relatively limited impact on the performance under $t_{2.25}$ noise. In particular, as illustrated in the first the third columns of Figure \ref{fig:ihmdp_thred}, a smaller thresholding parameter leads to a narrower confidence interval width when $\beta_1$ is held constant. Moreover, the results in the second column indicate that the confidence interval width does not vary monotonically with respect to $\beta_1$, with $\beta_1=1/3$ emerging as the optimal choice.


\subsection{Technical Lemmas}

\begin{lemma}\label{lem:concen_mix} 
Let $Y_{1},...,Y_{n}$ be $\phi$-mixing sequence with $\phi(k)=O(\rho^{k})$. Assume that $|Y_{i}|\leq M$ and $\mbE Y_{i}=0$. Then for any fixed $\nu>0$ and $1\leq k\leq n$, we have some constant $C$ such that
\begin{equation*}
\begin{aligned}
&\mbP\Bigg(\Big|\frac{\sum_{i=1}^{k}Y_{i}}{k}\Big|\geq C\sqrt{\frac{\log n}{k}}+C\frac{\log^2 n}{k}\Bigg)=O(n^{-\nu}).
\end{aligned}
\end{equation*}
\end{lemma}

\begin{proof}
    We divide the $k$-tuple $(1,\dots,k)$ into $m_k$ different subsets $H_1,\dots,H_{m_k}$, where $m_k = \lceil k/\lceil \lambda\log n\rceil\rceil$. Here $|H_i|=\lceil\lambda\log n\rceil$ for $1\leq i\leq m_k-1$, and $|H_{m_j}|\leq\lceil\lambda\log n\rceil$. $\lambda$ is a sufficiently large constant which will be specified later. Then we have $m_k\approx k/(\lambda\log n)$. Without loss of generality, we assume that $m_k$ is an even integer.
    
    Let $\xi_l=\frac{1}{\lceil\lambda\log n\rceil}\sum_{i\in H_{2l-1}}Y_i$, then we know $|\xi_l|\leq M$ and $\mbE[\xi_l]=0$ holds. For all $B_1\in\sigma(\xi_1,\dots,\xi_l)$ and $B_2\in\sigma(\xi_{l+1})$, there holds $|\mbP(B_2|B_1)-\mbP(B_2)|\leq\phi(\lceil\lambda\log n\rceil)$ for all $l$. By Theorem 2 of \cite{berkes_philipp.1979aop}, there exists a sequence of independent variables $\eta_l$, $l\geq 1$ with $\eta_l$ having the same distribution as $\xi_l$, and
    \begin{equation*}
        \mbP\Big(|\xi_{l}-\eta_{l}|\geq 6\phi(\lceil\lambda\log n\rceil)\Big)\leq  6\phi(\lceil\lambda\log n\rceil).
    \end{equation*}
    For any $\nu>0$, we can take $\lambda\geq(\nu +1)/|\log\rho|$ so that
    \begin{equation}    \label{eq:berke_rv}
    \begin{aligned}
        &\mbP\Big(\Big|\frac{2}{ m_k}\sum_{l=1}^{ m_{k}/2}(\xi_{l}-\eta_{l})\Big|\geq C_1n^{-\nu}\Big)\\
        \leq&\mbP\Big(\Big|\frac{2}{ m_k}\sum_{l=1}^{ m_{k}/2}(\xi_{l}-\eta_{l})\Big|\geq 6\phi(\lceil\lambda\log n\rceil)\Big)\leq 3m_{k}\phi(\lceil\lambda\log n\rceil)\leq C_{1}n^{-\nu},
    \end{aligned}
    \end{equation}
    where $C_1>0$ is some constant. Next, we bound the variance of $\xi_l$. From equation (20.23) of \cite{billingsley1968convergence} we know that for arbitrary $i,j$, there is
    \begin{equation*}
        \big|\mbE[Y_iY_j]-\mbE[Y_i]\mbE[Y_j]\big|\leq 2\sqrt{\phi(|i-j|)}\sqrt{\mbE[Y_i^2]\mbE[Y_j^2]}\leq 2\rho^{|i-j|/2}M^2.
    \end{equation*}
    Then we compute that
    \begin{align*}
        \var(\eta_l)=\var(\xi_l) =& \frac{1}{\lceil\lambda\log n\rceil^2}\Big[\sum_{i,j\in H_{2l-1}}\{\mbE(Y_iY_j)-\mbE(Y_i)\mbE(Y_j)\}\Big]\stepcounter{equation}\tag{\theequation}\label{eq:var_eta}\\
        \leq&\frac{2}{\lceil\lambda\log n\rceil^2}\sum_{i,j=1}^{\lceil\lambda\log n\rceil}\rho^{|i-j|/2}M^2\\
        \leq&\frac{2\lceil\lambda\log n\rceil}{\lceil\lambda\log n\rceil^2}\Big(\frac{2}{1-\sqrt{\rho}}-1\Big)M^2\leq\frac{4}{(1-\sqrt{\rho})\lceil\lambda\log n\rceil}M^2.
    \end{align*}
    Notice that $\eta_l$'s are all uniformly bounded by $M$, we can apply Bernstein's inequality (\citealp{Bennett.1962jasa}) and yield
    \begin{align*}
        &\mbP\Big(\Big|\frac{2}{m_k}\sum_{l=1}^{m_k/2}\eta_l\Big|\geq t\Big)\leq 2\exp\Big(-\frac{m_kt^2/4}{\var(\eta_l)/2+Mt/3}\Big)\\
        \leq&\exp\Big(-\frac{m_k\lceil\lambda\log n\rceil t^2/4}{2M^2/(1-\sqrt{\rho})+\lceil\lambda\log n\rceil Mt/3}\Big).
    \end{align*}
    We take $t=C(\sqrt{\log n/k}+\log^2n/k)$ for some $C$ large enough (note that $k\approx m_k\lceil\lambda\log n \rceil$). Then we have that
    \begin{equation}    \label{eq:ind_bern}
        \mbP\Big(\Big|\frac{2}{m_k}\sum_{l=1}^{m_k/2}\eta_l\Big|\geq t\Big)=O(n^{-\nu}).
    \end{equation}
    Combining \eqref{eq:berke_rv} and \eqref{eq:ind_bern} we have that
    \begin{equation}    \label{eq:odd_avg}
    \begin{aligned}
        &\mbP\Big(\Big|\frac{2}{m_k}\sum_{l=1}^{m_k/2}\xi_l\Big|\geq C\sqrt{\frac{\log n}{k}}+C\frac{\log^2n}{k}\Big)=O(n^{-\nu}). 
    \end{aligned}
    \end{equation}
    Next we denote $\tilde{\xi}=\frac{1}{\lceil\lambda\log n\rceil}\sum_{i\in H_{2l}}Y_i$. We can similarly show that
    \begin{equation}    \label{eq:even_avg}
    \begin{aligned}
        &\mbP\Big(\Big|\frac{2}{m_k}\sum_{l=1}^{m_k/2}\tilde{\xi}_l\Big|\geq C\sqrt{\frac{\log n}{k}}+C\frac{\log^2n}{k}\Big)=O(n^{-\nu}).
    \end{aligned}
    \end{equation}
    Combining \eqref{eq:odd_avg} and \eqref{eq:even_avg} we prove the desired result.
\end{proof}

\begin{lemma}	\label{lem:mart_diff_concen}
	Let $\vect{Y}_1,\dots,\vect{Y}_n$ be random vectors in $\mbR^p$. Let $\mcG_i=\sigma(\vect{Y}_1,\dots,\vect{Y}_i)$. Assume that
	\begin{align*}
		&\mbE[\vect{Y}_i|\mcG_{i-1}]=\vect{0},\quad\max_{1\leq i\leq n}\sup_{\vect{v}\in\mbS^p}|\vect{v}^{\tp}\vect{Y}_i|\leq 1,\\
		&\sup_{\vect{v}\in\mbS^p}\sum_{i=1}^n\var[\vect{v}^{\tp}\vect{Y}_i|\mcG_{i-1}]\leq b_n.
	\end{align*}
	Then for every $\nu>0$, there is
	\begin{equation*}
		\mbP\Big(\big|\sum_{i=1}^n\vect{Y}_i\big|_2\geq C\big(d\log n+\sqrt{db_n\log n }\big)\Big)=O(n^{-\nu d}),
	\end{equation*}
	for some $C$ sufficiently large.
\end{lemma}

\begin{proof}
	Let $\mfN$ be the $1/2$-net of the unit ball $\mbS^{d-1}$. By Lemma 5.2 of \cite{vershynin.2010} we know that $|\mfN|\leq 5^d$. Then we have that
	\begin{align*}
		\Big|\sum_{i=1}^n\vect{Y}_i\Big|_2 =& \sup_{\vect{v}\in\mbS^{d-1}}\Big|\sum_{i=1}^n\vect{v}^{\tp}\vect{Y}_i\Big|\\
		\leq&\sup_{\tilde{\vect{v}}\in\mfN}\Big|\sum_{i=1}^n\tilde{\vect{v}}^{\tp}\vect{Y}_i\Big|+\sup_{|\vect{v}-\tilde{\vect{v}}|_2\leq 1/2}\Big|\sum_{i=1}^n(\vect{v}-\tilde{\vect{v}})^{\tp}\vect{Y}_i\Big|,\\
		\Rightarrow \Big|\sum_{i=1}^n\vect{Y}_i\Big|_2\leq& 2\sup_{\tilde{\vect{v}}\in\mfN}\Big|\sum_{i=1}^n\tilde{\vect{v}}^{\tp}\vect{Y}_i\Big|.
	\end{align*}
	By Theorem 1.6 of \cite{freedman.1975aop} we know that
	\begin{equation*}
		\sup_{\tilde{\vect{v}}\in\mbS^{d-1}}\mbP\Big(\big|\sum_{i=1}^n\tilde{\vect{v}}^{\tp}\vect{Y}_i\big|\geq C_1\big(d\log n+\sqrt{db_n\log n }\big)\Big)=O(n^{-(\nu+\log 5)d})
	\end{equation*}
	for some large constant $C'>0$. Therefore
	\begin{align*}
		&\mbP\Big(\big|\sum_{i=1}^n\vect{Y}_i\big|_2\geq 2C_1\big(d\log n+\sqrt{db_n\log n }\big)\Big)\\
		\leq&\mbP\Big(\sup_{\tilde{\vect{v}}\in\mfN}\big|\sum_{i=1}^n\tilde{\vect{v}}^{\tp}\vect{Y}_i\big|\geq C_1\big(d\log n+\sqrt{db_n\log n }\big)\Big)\\
		\leq&5^d\sup_{\tilde{\vect{v}}\in\mfN}\mbP\Big(\big|\sum_{i=1}^n\tilde{\vect{v}}^{\tp}\vect{Y}_i\big|\geq C_1\big(d\log n+\sqrt{db_n\log n }\big)\Big)=O(n^{-\nu d}),
	\end{align*}
	which proves the lemma.
\end{proof}

\begin{lemma}   \label{lem:mix_norm}
    Let $\vect{Y}$ be a $d\times d$ random matrix with $\mbE[\vect{Y}]=\vect{0}$, $\|\vect{Y}\|\leq 1$, $\mcF$ be the $\sigma$-field generated by $\vect{Y}$. Then for any $\sigma$-field $\mcG$, there holds
    \begin{equation*}
        \mbE\Big[\|\mbE[\vect{Y}|\mcG]\|\Big]\leq d\sqrt{2\pi\phi},
    \end{equation*}
    where
    \begin{equation*}
        \phi = \sup_{A\in\mcF,B\in\mcG}|\mbP(AB)-\mbP(A)\mbP(B)|.
    \end{equation*}
\end{lemma}

\begin{proof}
    By elementary inequality of matrix, we know that
    \begin{equation*}
        |\vect{Y}|_{\infty}\leq\|\vect{Y}\|\leq 1, \quad\text{ and }\quad\|\vect{Y}\|\leq \|\vect{Y}\|_F.
    \end{equation*}
    For every element $Y_{i,j}$ of $\vect{Y}$, by Lemma 4.4.1 of \cite{berkes_philipp.1979aop} we have that
    \begin{equation*}
        \mbE\Big[\big|\mbE[Y_{ij}|\mcG]\big|^2\Big]\leq \mbE\Big[\big|\mbE[Y_{ij}|\mcG]\big|\Big]\leq 2\pi\phi.
    \end{equation*}
    Therefore, we have that
    \begin{align*}
        &\mbE\Big[\|\mbE[\vect{Y}|\mcG]\|\Big]\leq \mbE\Big[\|\mbE[\vect{Y}|\mcG]\|_F\Big],\\
        \leq&\Big\{\mbE\Big[\|\mbE[\vect{Y}|\mcG]\|^2_F\Big]\Big\}^{1/2}\\
        \leq&\Big\{\mbE\Big[\sum_{i,j=1}^d\big|\mbE[Y_{ij}|\mcG]\big|^2\Big]\Big\}^{1/2}\leq\sqrt{d^22\pi\phi}=d\sqrt{2\pi\phi}.
    \end{align*}
    The proof is complete.
\end{proof}

\begin{lemma}   \label{lem:huber_grad_approx}
(Approximation of pseudo-Huber gradient) 
The gradient of pseudo-Huber loss $g_{\tau}(x)$ satisfies
\begin{equation}    \label{eq:huber_grad_approx_in}
\begin{aligned}
    |g_{\tau}(x)-x|&\leq\frac{1}{2}\tau^{-\delta}|x|^{1+\delta},\quad \text{ for }\delta\in(0,2]\\
    |g'_{\tau}(x)-1|&\leq \frac{5}{2}\tau^{-1-\delta}|x|^{1+\delta}\quad \text{ for }\delta\in(0,1]
\end{aligned}
\end{equation}
uniformly for $|x|\leq \tau$. And
\begin{equation}    \label{eq:huber_grad_approx_out}
    |g_{\tau}(x)-x|\leq|x|,\quad |g'_{\tau}(x)-1|\leq1,
\end{equation}
holds for all $x$.
\end{lemma}

\begin{proof}
    It is easy to compute that 
    \begin{equation*}
        g_{\tau}(x) = \frac{x}{\sqrt{1+x^2/\tau^2}},\quad g'_{\tau}(x)=\frac{1}{(1+x^2/\tau^2)^{3/2}}.
    \end{equation*}
    It is not hard to see that
    \begin{equation*}    
        |g_{\tau}(x)-x|\leq|x|,\quad |g'_{\tau}(x)-1|\leq1.
    \end{equation*}
This proves \eqref{eq:huber_grad_approx_out}. Furthermore, when $|x|\leq \tau$, we have
    \begin{align*}
        |g_{\tau}(x)-x|=&\frac{x^2/\tau^2}{(1+\sqrt{1+x^2/\tau^2})\sqrt{1+x^2/\tau}}\cdot|x|\\
        \leq&\frac{1}{2}\tau^{-\delta}|x|^{1+\delta}
    \end{align*}
    for all $0\leq \delta\leq 2$.
 Similarly,
    \begin{align*}
        |g'_{\tau}(x)-1|=&\frac{x^2/\tau^2(2+\sqrt{1+x^2/\tau^2}+x^2/\tau^2)}{(1+\sqrt{1+x^2/\tau^2})(1+x^2/\tau^2)^{3/2}}\\
        \leq&\frac{5}{2}\tau^{-1-\delta}|x|^{1+\delta},
    \end{align*}
    for $|x|\leq\tau$. The proof is complete.
\end{proof}

\begin{lemma}   \label{lem:seq_bound}
    For any sequence $\{a_{k}\}$ where $a_k=(\log k)^{\beta_1}/k^{\beta_2}$ and $\beta_1\geq0$, there hold
    \begin{equation*}
        n_0a_{n_0}+\sum_{k=n_0+1}^{n}a_k\leq 
        \begin{cases}
            C na_{n}\log n\quad&\text{if }\beta_2\leq 1;\\
            C(\log n)^{\beta_1}/n_0^{\beta_2-1}\quad&\text{if }\beta_2>1.
        \end{cases}
    \end{equation*}
    Here the constant $C$ only depends on the parameter $\beta_2$.
\end{lemma}

\begin{proof}
    Directly compute that
    \begin{equation}    \label{eq:sum_bound}
        n_0a_{n_0}+\sum_{k=n_0+1}^na_k\leq (\log n)^{\beta_1}\Big(\frac{n_0}{n_0^{\beta_2}}+\sum_{k=n_0+1}^{n}\frac{1}{k^{\beta_2}}\Big).
    \end{equation}
    We know that
    \begin{equation*}
        \sum_{k=n_0+1}^{n}\frac{1}{k^{\beta_2}}\leq\int_{n_0}^n\frac{1}{x^{\beta_2}}\diff x=\frac{1}{1-\beta_2}\Big(n^{1-\beta_2}-n_0^{1-\beta_2}\Big).
    \end{equation*}
    When $\beta_2< 1$, \eqref{eq:sum_bound} can be bounded by
    \begin{align*}
        n_0a_{n_0}+\sum_{k=n_0+1}^na_k\leq& (\log n)^{\beta_1}\Big(n_0^{1-\beta_2}+\frac{1}{1-\beta_2}(n^{1-\beta_2}-n_0^{1-\beta_2})\Big)\\
        \leq&(\log n)^{\beta_1}n^{1-\beta_2}\Big(-\frac{\beta_2}{1-\beta_2}\Big(\frac{n_0}{n}\Big)^{1-\beta_2}+\frac{1}{1-\beta_2}\Big)\\
        \leq&\frac{1}{1-\beta_2}na_n.
    \end{align*}
    When $\beta_2>1$, \eqref{eq:sum_bound} can be bounded by
    \begin{align*}
        n_0a_{n_0}+\sum_{k=n_0+1}^na_k\leq& (\log n)^{\beta_1}\Big(\frac{\beta_2}{\beta_2-1}\times\frac{1}{n_0^{\beta_2-1}}-\frac{1}{\beta_2-1}\times\frac{1}{n^{\beta_2-1}}\Big)\\
        \leq&\frac{\beta_2}{\beta_2-1}(\log n)^{\beta_1}\frac{1}{n_0^{\beta_2-1}}.
    \end{align*}
    When $\beta_2=1$, \eqref{eq:sum_bound} can be bounded by
    \begin{align*}
        n_0a_{n_0}+\sum_{k=n_0+1}^na_k\leq& (\log n)^{\beta_1}\big(1+\log n\big)\\
        \leq&2na_n\log n.
    \end{align*}
\end{proof}


\subsection{Proof of Results in Section \ref{sec:online_rl}}

\begin{proof}[Proof of Theorem \ref{thm:contam_rate}]
     For each $k\geq n_0$, when $\delta>0$, we denote
     \begin{equation}   \label{eq:ek_case1}
         e_k = \alpha_k\tau_k + \tau_k^{-\min(\delta,2)}+ \sqrt{\frac{\tau_k^{(1-\delta)_+}\log k}{k}}+\frac{\log^2k\tau_k}{k}+\frac{1}{\sqrt{d}}(c_0)^{2^{k-n_0}},
     \end{equation}
     where $\delta' = \min\{2,\delta\}$. Given a sufficiently large constant $\Psi>0$, we define the event
    \begin{equation}    \label{eq:huber_ind_event}
    \begin{aligned}
        E_k =& \Big\{|\hat{\vect{\theta}}_k-\vect{\theta}^*|_2\leq \Psi\sqrt{d}e_k\Big\} \cap\Big\{\|\hat{\vect{H}}_{k}-\vect{H}\|\leq C_0\Psi\sqrt{d}e_k\log k\Big\}\\
        &\cap\Big\{\Big|\frac{1}{k}\sum_{i=1}^{k}\vect{X}_{i}g_{\tau_i}(\vect{Z}_{i}^{\tp}\vect{\theta}^*-b_{i})\Big|_2\leq C_0\sqrt{d}e_{k}\Big\},
    \end{aligned}
    \end{equation}
    where the $C_0$ is a constant that does not depend on $\Psi$. We will show that $\mbP(E_k^c,\cap_{i=n_0}^{k-1}E_i)\leq C_0k^{-\nu}$, where $\nu>1$ is some constant and $C_0>0$ is some constant only depends on $\nu$. Let us prove it by induction on $k$. By the choice of initial parameter $\hat{\vect{\theta}}_{0}$, we know \eqref{eq:huber_ind_event} holds for $k=n_0$. 
    For $n = k+1$, from \eqref{eq:online_newton} we have that
    \begin{align*}	
	    \hat{\vect{\theta}}_{k+1}-\vect{\theta}^* =& \frac{1}{k+1}\sum_{i=1}^{k+1}(\hat{\vect{\theta}}_{i-1}-\vect{\theta}^*)-\hat{\vect{H}}_{k+1}^{-1}\frac{1}{k+1}\sum_{i=1}^{k+1}\big\{\vect{X}_{i}g_{\tau_{i}}(\vect{Z}_{i}^{\tp}\hat{\vect{\theta}}_{i-1}-b_{i})-\vect{X}_{i}g_{\tau_{i}}(\vect{Z}_{i}^{\tp}\vect{\theta}^*-b_{i})\big\}\\
	    &-\hat{\vect{H}}_{k+1}^{-1}\frac{1}{k+1}\sum_{i=1}^{k+1}\big\{\vect{X}_{i}g_{\tau_{i}}(\vect{Z}_{i}^{\tp}\vect{\theta}^*-b_{i})\big\}.\stepcounter{equation}\tag{\theequation}\label{eq:huber_prim_expand}
    \end{align*}
    For the second term on the right-hand side of \eqref{eq:huber_prim_expand}, we have that
    \begin{eqnarray*}
        &&\frac{1}{k+1}\sum_{i=1}^{k+1}\big\{\vect{X}_{i}g_{\tau_i}(\vect{Z}_{i}^{\tp}\hat{\vect{\theta}}_{i-1}-b_{i})-\vect{X}_{i}g_{\tau_i}(\vect{Z}_{i}^{\tp}\vect{\theta}^*-b_{i})\big\}\cr
        =&&\frac{1}{k+1}\sum_{i=1}^{k+1}\vect{X}_{i}\vect{Z}_{i}^{\tp}g_{\tau_i}'(\vect{Z}_{i}^{\tp}\vect{\theta}^*-b_{i})(\hat{\vect{\theta}}_{i-1}-\vect{\theta}^*)\cr
        &&+\frac{1}{k+1}\sum_{i=1}^{k+1}\int_0^1(1-t)\vect{X}_{i}g_{\tau_i}''\big\{\vect{Z}_{i}^{\tp}(\vect{\theta}^*+t(\hat{\vect{\theta}}_{i-1}-\vect{\theta}^*))-b_{i}\big\}(\vect{Z}_{i}^{\tp}(\hat{\vect{\theta}}_{i-1}-\vect{\theta}^*))^2\diff t\cr
        =&&\frac{1}{k+1}\sum_{i=1}^{k+1}\vect{X}_{i}\vect{Z}_{i}^{\tp}g_{\tau_i}'(\vect{Z}_{i}^{\tp}\vect{\theta}^*-b_{i})(\hat{\vect{\theta}}_{i-1}-\vect{\theta}^*)+\frac{1}{k+1}\sum_{i=1}^{k+1}O(1)|\hat{\vect{\theta}}_{i-1}-\vect{\theta}^*|_2^2\cr
        =&&\frac{1}{k+1}\sum_{i=1}^{k+1}\vect{X}_{i}\vect{Z}_{i}^{\tp}g_{\tau_i}'(\vect{Z}_{i}^{\tp}\vect{\theta}^*-b_{i})(\hat{\vect{\theta}}_{i-1}-\vect{\theta}^*)+O(1)\frac{1}{k+1}\sum_{i=1}^{k+1}|\hat{\vect{\theta}}_{i-1}-\vect{\theta}^*|_2^2,
    \end{eqnarray*}
    where the second equality holds because of the fact that $|g_{\tau_{i}}''(x)|\leq 3/\tau_{i}=O(1)$, the last equality uses the inductive hypothesis. Substitute it into \eqref{eq:huber_prim_expand}, we have
    \begin{align*}	
	    \hat{\vect{\theta}}_{k+1}-\vect{\theta}^* =& (\vect{I}-\hat{\vect{H}}_{k+1}^{-1}\vect{H})\frac{1}{k+1}\sum_{i=1}^{k+1}(\hat{\vect{\theta}}_{i-1}-\vect{\theta}^*)-\hat{\vect{H}}_{k+1}^{-1}\frac{1}{k+1}\sum_{i=1}^{k+1}(\vect{A}_{i}-\vect{H})(\hat{\vect{\theta}}_{i-1}-\vect{\theta}^*)\\
	    &-\hat{\vect{H}}_{k+1}^{-1}\frac{1}{k+1}\sum_{i=1}^{k+1}\big\{\vect{X}_{i}g_{\tau_{i}}(\vect{Z}_{i}^{\tp}\vect{\theta}^*-b_{i})\big\}+O(1)\frac{1}{k+1}\sum_{i=1}^{k+1}|\hat{\vect{\theta}}_{i-1}-\vect{\theta}^*|_2^2\stepcounter{equation}\tag{\theequation}\label{eq:huber_key_expand},
    \end{align*}
    where $\vect{A}_i=\vect{X}_i\vect{Z}_i^{\tp}g_{\tau_{i}}'(\vect{Z}_i^{\tp}\vect{\theta}^*-b_i)$ and $\vect{H} = \mbE[\vect{X}\vect{Z}^{\tp}]$.
    
    Under event $\cap_{i=n_0}^kE_i$, by Lemma \ref{lem:seq_bound} there exists a constant $C_1>0$ such that
    \begin{align*}   
        &\mbP\Big(\frac{1}{k+1}\sum_{i=1}^{k+1}|\hat{\vect{\theta}}_{i-1}-\vect{\theta}^*|_2^2\geq C_1\Psi^2de_k^2\log k,\cap_{i=n_0}^kE_i\Big)\\
        \geq& \mbP\Big(\frac{1}{k+1}\sum_{i=1}^{k+1}|\hat{\vect{\theta}}_{i-1}-\vect{\theta}^*|_2^2\geq \frac{1}{k+1}\sum_{i=0}^k\Psi^2de_i^2,\cap_{i=n_0}^kE_i\Big) =0\stepcounter{equation}\tag{\theequation}\label{eq:theta_error_bound}.
    \end{align*}
    Similarly, by Lemma \ref{lem:seq_bound} there exists a constant $C_2>0$ such that
    \begin{align*}   
        &\mbP\Big(\frac{1}{k+1}\sum_{i=1}^{k+1}|\hat{\vect{\theta}}_{i-1}-\vect{\theta}^*|_2\geq C_2\Psi\sqrt{d}e_k\log k,\cap_{i=n_0}^kE_i\Big)\\
        \geq& \mbP\Big(\frac{1}{k+1}\sum_{i=1}^{k+1}|\hat{\vect{\theta}}_{i-1}-\vect{\theta}^*|_2\geq \frac{1}{k+1}\sum_{i=0}^k\Psi\sqrt{d}e_i,\cap_{i=n_0}^kE_i\Big) =0\stepcounter{equation}\tag{\theequation}\label{eq:theta_1error_bound}.
    \end{align*}
    
    By Lemma \ref{lem:huber_concen_hess}, \eqref{eq:theta_1error_bound} and \eqref{eq:huber_inv_bound} we know that under event $\cap_{i=n_0}^kE_i$, for every $\nu>0$ there exist constants $C_1,C_3>0$ (which only depend on $\nu$), such that
    \begin{align*}
        &\mbP\Big(\Big|(\vect{I}-\hat{\vect{H}}_{k+1}^{-1}\vect{H})\frac{1}{k+1}\sum_{i=0}^k(\hat{\vect{\theta}}_i-\vect{\theta}^*)\Big|_2\geq C_1\Psi^2de^2_k\log^2k, \cap_{i=n_0}^kE_i\Big)\\
        \leq&\mbP\Big(\|\hat{\vect{H}}_{k+1}^{-1}\|\cdot\|\hat{\vect{H}}_{k+1}-\vect{H}\|\cdot\frac{1}{k+1}\sum_{i=0}^k|\hat{\vect{\theta}}_i-\vect{\theta}^*|_2\geq C_1\Psi^2de_k^2\log^2k, \cap_{i=n_0}^kE_i\Big)\leq C_3(k+1)^{-\nu}.
    \end{align*}
    By Lemma \ref{lem:mix_bound} we know that
    \begin{equation}
        \mbP\Big( \Big|\frac{1}{k+1}\sum_{i=1}^{k+1}(\vect{A}_i-\vect{H})(\hat{\vect{\theta}}_{i-1}-\vect{\theta}^*)\Big|_2\geq C_1(\alpha_k+\Psi^2de_k^2\log k), \cap_{i=n_0}^kE_i \Big)\leq C_3(k+1)^{-\nu}.
    \end{equation}
        
    In this part, we mainly consider the case where $\delta>0$ and $\tau_k$ can be arbitrary. A special case where $\delta>4$ and $\sqrt{i/\log^3i}=O(\tau_i)$ will be presented in Proposition \ref{prop:contam_rate_d3} below. By Lemma \ref{lem:huber_grad_concen} we have that
    \begin{equation*}
        \mbP\Big(\Big|\hat{\vect{H}}_{k+1}^{-1}\frac{1}{k+1}\sum_{i=1}^{k+1}\big\{\vect{X}_{i}g(\vect{Z}_{i}^{\tp}\vect{\theta}^*-b_{i})\big\}\Big|_2\geq C_1\sqrt{d}e_{k+1}, \cap_{i=n_0}^kE_i\Big)\leq C_3(k+1)^{-\nu};
    \end{equation*}
    
    By substituting the above inequalities into \eqref{eq:huber_key_expand}, we have that
    \begin{align*}
        &\mbP\Big(|\hat{\vect{\theta}}_{k+1}-\vect{\theta}^*|_2\geq \Psi\sqrt{d}e_{k+1},\cap_{i=n_0}^kE_i\Big)\\
        \leq& \mbP\Big(|\hat{\vect{\theta}}_{k+1}-\vect{\theta}^*|_2\geq C_1\sqrt{d}e_{k+1}+C_1\alpha_k+3C_1\Psi^2de_k^2\log^2k,\cap_{i=n_0}^kE_i\Big)\\
        \leq& 3C_3(k+1)^{-\nu},\stepcounter{equation}\tag{\theequation}\label{eq:ind_complt}
    \end{align*}
    where $\Psi$ can be taken to be sufficiently large, given $C_1$ fixed. Here the second inequality holds because $\sqrt{d}e_k$ can be sufficiently small for $k\geq n_0$ by taking $n_0$ sufficiently large. Meanwhile, as $\alpha_k\leq (k+1)\alpha_{k+1}/k\leq (n_0+1)\alpha_{k+1}/n_0=o(\tau_{k+1}\alpha_{k+1})=o(e_{k+1})$, the last two terms both have order of $o(e_{k+1})$. Obviously, the rest two events of \eqref{eq:huber_ind_event} are contained in \eqref{eq:ind_complt}. Therefore, the inductive hypothesis that $\mbP(E_{k+1}^c,\cap_{j=n_0}^{k}E_j)\leq 3C_3(k+1)^{-\nu}$, is proved.
    
    For $k\geq n_0$, there holds
    \begin{align*}
        &\mbP\Big(\cap_{i=n_0}^kE_i\Big)\\
        \geq&\mbP\Big(E_{n_0}\Big)-\sum_{i=n_0+1}^k\mbP\Big(E_i^c,\cap_{j=n_0}^{i-1}E_j\Big) \\
        \geq&1-\sum_{i=n_0+1}^k3C_3i^{-\nu}\geq 1-\frac{3C_3}{\nu-1}\frac{1}{n_0^{\nu-1}},
    \end{align*}
    which proves the theorem.
\end{proof}

\begin{lemma}   \label{lem:huber_concen_hess}
    (Bound of the pseudo-Huber Hessian matrix) Under the same assumptions as in Theorem \ref{thm:contam_rate}, there exists uniform constants $C$ and $c$, such that the Hessian matrix $\hat{\vect{H}}_{k+1}$ satisfies
    \begin{equation*}
        \mbP\Big(\|\hat{\vect{H}}_{k+1}-\vect{H}\|\geq C\Psi\sqrt{d}e_k\log k,\cap_{i=n_0}^kE_i\Big)\leq c(k+1)^{-\nu d}.
    \end{equation*}
\end{lemma}

\begin{proof}
    We note that
    \begin{align*}
        \|\hat{\vect{H}}_{k+1}-\vect{H}\|\leq\frac{1}{k+1}\Big\|\sum_{i\in\mcQ_k}\vect{X}_i\vect{Z}_ig'_{\tau_i}(\vect{Z}_i^{\tp}\hat{\vect{\theta}}_{i-1}-b_i)-\vect{H}\Big\|+\frac{1}{k+1}\Big\|\sum_{i\notin\mcQ_k}\vect{X}_i\vect{Z}_ig'_{\tau_i}(\vect{Z}_i^{\tp}\hat{\vect{\theta}}_{i-1}-b_i)-\vect{H}\Big\|.
    \end{align*}
    For the first term, there is
    \begin{align*}
        &\Big\|\vect{X}\vect{Z}^{\tp}g'_{\tau}(\vect{Z}^{\tp}\vect{\theta}-b)\Big\|\\
        =&\sup_{\vect{u},\vect{v}\in\mbS^{d-1}}\vect{u}^{\tp}\vect{X}\vect{Z}^{\tp}\vect{v}|g'_{\tau}(\vect{Z}^{\tp}\vect{\theta}-b)|\leq M^2,
    \end{align*}
    holds for all $\vect{\theta}$ and $\tau\geq1$. Therefore we have that
    \begin{align*}
        &\frac{1}{k+1}\Big\|\sum_{i\in\mcQ_k}\vect{X}_i\vect{Z}_ig'_{\tau_i}(\vect{Z}_i^{\tp}\hat{\vect{\theta}}_{i-1}-b_i)-\vect{H}\Big\|\stepcounter{equation}\tag{\theequation}\label{eq:contam_hess_term2}\\
        \leq&\frac{1}{k+1}\sum_{i\in\mcQ_k}\|\vect{X}_i\vect{Z}_ig'_{\tau_i}(\vect{Z}_i^{\tp}\hat{\vect{\theta}}_{i-1}-b_i)\|+\frac{1}{k+1}\sum_{i\in\mcQ_i}\|\vect{H}\|\leq 2M^2\alpha_k.
    \end{align*}
    For the second term, we first prove that
    \begin{equation}    \label{eq:huber_center_hess}
        \mbP\Big(\Big\|\frac{1}{k+1}\sum_{i\notin\mcQ_k}(\vect{A}_i-\mbE[\vect{A}_i])\Big\|\geq C_1\sqrt{\frac{d\log k}{k+1}}\Big)=O((k+1)^{-\nu d}).
    \end{equation}
    Let $\mfN$ be the $1/4$-net of the unit ball $\mbS^{d-1}$. By Lemma 5.2 of \cite{vershynin.2010} we know that $|\mfN|\leq 9^{d}$. Then we have that
    \begin{eqnarray*}
        \Big\|\frac{1}{k+1}\sum_{i\notin\mcQ_k}(\vect{A}_i-\mbE[\vect{A}_i])\Big\|=&&\sup_{\vect{u},\vect{v}\in\mbS^{d-1}}\Big|\frac{1}{k+1}\sum_{i\notin\mcQ_k}\vect{u}^{\tp}(\vect{A}_i-\mbE[\vect{A}_i])\vect{v}\Big|\cr
        \leq&&\sup_{\tilde{\vect{u}},\tilde{\vect{v}}\in\mfN}\Big|\frac{1}{k+1}\sum_{i\notin\mcQ_k}\tilde{\vect{u}}^{\tp}(\vect{A}_i-\mbE[\vect{A}_i])\tilde{\vect{v}}\Big|\cr
        &&+\sup_{\tilde{\vect{u}}\in\mfN}\sup_{|\vect{v}-\tilde{\vect{v}}|_2\leq 1/4}\Big|\frac{1}{k+1}\sum_{i\notin\mcQ_k}\tilde{\vect{u}}^{\tp}(\vect{A}_i-\mbE[\vect{A}_i])(\vect{v}-\tilde{\vect{v}})\Big|_2\cr
        &&+\sup_{\vect{v}\in\mbS^{d-1}}\sup_{|\vect{u}-\tilde{\vect{u}}|_2\leq 1/4}\Big|\frac{1}{k+1}\sum_{i\notin\mcQ_k}(\vect{u}-\tilde{\vect{u}})^{\tp}(\vect{A}_i-\mbE[\vect{A}_i])\vect{v}\Big|_2\cr
        \Rightarrow\Big\|\frac{1}{k+1}\sum_{i\notin\mcQ_k}(\vect{A}_i-\mbE[\vect{A}_i])\Big\|\leq&&2\sup_{\tilde{\vect{u}},\tilde{\vect{v}}\in\mfN}\Big|\frac{1}{k+1}\sum_{i\notin\mcQ_k}\tilde{\vect{u}}^{\tp}(\vect{A}_i-\mbE[\vect{A}_i])\tilde{\vect{v}}\Big|.
    \end{eqnarray*}
    Applying Lemma \ref{lem:concen_mix} to every $\frac{1}{k+1}\sum_{i\notin\mcQ_k}\tilde{\vect{u}}^{\tp}(\vect{A}_i-\mbE[\vect{A}_i])\tilde{\vect{v}}$ we can obtain that
    \begin{equation*}
        \sup_{\tilde{u},\tilde{\vect{v}}\in\mfN}\mbP\Big(\Big|\frac{1}{k+1}\sum_{i\notin\mcQ_k}\tilde{\vect{u}}^{\tp}(\vect{A}_i-\mbE[\vect{A}_i])\tilde{\vect{v}}\Big|_2\geq C_1\sqrt{\frac{d\log k}{k+1}}\Big)=O((k+1)^{-(\nu+\log 9)d})
    \end{equation*}
    for some constant $C_1>0$. 
    Therefore
    \begin{align*}
        &\mbP\Big(\Big\|\frac{1}{k+1}\sum_{i\notin\mcQ_k}(\vect{A}_i-\mbE[\vect{A}_i])\Big\|\geq 2C_1\sqrt{\frac{d\log k}{k+1}}\Big)\\
        \leq& \mbP\Big(\sup_{\tilde{u},\tilde{\vect{v}}\in\mfN}\Big|\frac{1}{k+1}\sum_{i\notin\mcQ_k}\tilde{\vect{u}}^{\tp}(\vect{A}_i-\mbE[\vect{A}_i])\tilde{\vect{v}}\Big|_2\geq C_1\sqrt{\frac{d\log k}{k+1}}\Big)\\
        \leq&9^{2d}\sup_{\tilde{\vect{u}},\tilde{\vect{v}}\in\mfN}\mbP\Big(\Big|\frac{1}{k+1}\sum_{i\notin\mcQ_k}\tilde{\vect{u}}^{\tp}(\vect{A}_i-\mbE[\vect{A}_i])\tilde{\vect{v}}\Big|_2\geq C_1\sqrt{\frac{d\log k}{k+1}}\Big)=O((k+1)^{-\nu d}),
    \end{align*}
    which proves \eqref{eq:huber_center_hess}.  Next we prove that when $(\vect{X},\vect{Z},b)$ are normal data, for every $\tau>0$,
    \begin{equation*}
    	\Big\|\mbE[\vect{X}\vect{Z}^{\tp}g'_{\tau}(\vect{Z}^{\tp}\vect{\theta}^*-b)]-\mbE[\vect{X}\vect{Z}^{\tp}]\Big\|\leq  C_2\tau^{-\delta'}.
    \end{equation*}
    Indeed, denote $\epsilon = \vect{Z}^{\tp}\vect{\theta}^*-b$, we have that
    \begin{align*}
    	&\Big\|\mbE[\vect{X}\vect{Z}^{\tp}g'_{\tau}(\vect{Z}^{\tp}\vect{\theta}^*-b)]-\mbE[\vect{X}\vect{Z}^{\tp}]\Big\|\\
	\leq&\sup_{\vect{u},\vect{v}}\mbE[\vect{u}^{\tp}\vect{X}\vect{Z}^{\tp}\vect{v}|g'_{\tau}(\epsilon)-1|]\\
	\leq&\sup_{\vect{u},\vect{v}}\mbE[\vect{u}^{\tp}\vect{X}\vect{Z}^{\tp}\vect{v}\mbI(|\epsilon|>\tau)+\frac{5}{2}\vect{u}^{\tp}\vect{X}\vect{Z}^{\tp}\vect{v}\tau^{-1-\delta''}|\epsilon|^{1+\delta''}\mbI(|\epsilon|\leq\tau)]\\
	\leq&\sup_{\vect{u},\vect{v}}\mbE\big[\vect{u}^{\tp}\vect{X}\vect{Z}^{\tp}\vect{v}\mbP(|\epsilon|>\tau|\vect{X},\vect{Z})\big]+\frac{5}{2}\tau^{-1-\delta''}\sup_{\vect{u},\vect{v}}\mbE\big[\vect{u}^{\tp}\vect{X}\vect{Z}^{\tp}\vect{v}\mbE[|\epsilon|^{1+\delta''}|\vect{X},\vect{Z}]\big]\leq C_2\tau^{-1-\delta''}\leq C_2\tau^{-\delta'},
    \end{align*}
    where the third line uses Lemma \ref{lem:huber_grad_approx}, and $\delta'' = \min(1,\delta)$. Then clearly we have that $1+\delta''\geq \delta'$. Therefore, by the choice of $\tau_i$ and Lemma \ref{lem:seq_bound}, we have
    \begin{equation}	\label{eq:EAH_diff}
    	\Big\|\frac{1}{k+1}\sum_{i\notin\mcQ_k}\mbE[\vect{A}_i]-\vect{H}\Big\|\leq\frac{1}{k+1}\sum_{i\notin\mcQ_k}\Big\|\mbE[\vect{A}_i]-\vect{H}\Big\|\leq C_2\tau_{k+1}^{-\delta'}.
    \end{equation}
    
    Under event $\cap_{i=n_0}^kE_i$, from \eqref{eq:theta_1error_bound} we know there is
    \begin{equation}    \label{eq:huber_hess_diff}
    \begin{aligned}
        \Big\|\frac{1}{k+1}\sum_{i\notin\mcQ_k}\vect{A}_i-\hat{\vect{H}}_k\Big\|=& \Big\|\frac{1}{k+1}\sum_{i\notin\mcQ_k}\vect{X}_i\vect{Z}_i^{\tp}\big\{g_{\tau_i}'(\vect{Z}_i^{\tp}\vect{\theta}^*-b_{i})-g_{\tau_i}'(\vect{Z}_i^{\tp}\hat{\vect{\theta}}_{i-1}-b_{i})\big\}\Big\|\\
        \leq&M^2\frac{1}{k+1}\sum_{i\notin\mcQ_k}|\hat{\vect{\theta}}_{i-1}-\vect{\theta}^*|_2\leq C_3\Psi\sqrt{d}e_k\log k,
    \end{aligned}
    \end{equation}
    Therefore combining \eqref{eq:huber_center_hess}, \eqref{eq:EAH_diff} and \eqref{eq:huber_hess_diff} we have that
    \begin{equation}	\label{eq:contam_hess_term1}
    \begin{aligned}
    .	&\mbP\Bigg(\frac{1}{k+1}\Big\|\sum_{i\notin\mcQ_k}\vect{X}_i\vect{Z}_ig'_{\tau_i}(\vect{Z}_i^{\tp}\hat{\vect{\theta}}_{i-1}-b_i)-\vect{H}\Big\| \geq C_4\Big(\sqrt{\frac{d\log k}{k+1}}+\tau_{k+1}^{-\delta'}+\Psi\sqrt{d}e_k\log k\Big)\Bigg)\\
    =&O((k+1)^{-\nu d}).
    \end{aligned}
    \end{equation}
    Combining \eqref{eq:contam_hess_term2} and \eqref{eq:contam_hess_term1}, we have that
    \begin{equation*}
    \begin{aligned}
        \mbP\Bigg(\|\hat{\vect{H}}_{k+1}-\vect{H}\|\geq C_5\Big(\alpha_k+\sqrt{\frac{d\log k}{k+1}}+\tau_{k+1}^{-\delta'}+\sqrt{d}e_k\log k\Big)\Bigg)=O((k+1)^{-\nu d}),
    \end{aligned}
    \end{equation*}
    which proves the lemma.  
    
    As a corollary, we can show that under event $\cap_{i=n_0}^kE_i$,
	\begin{equation}   \label{eq:huber_inv_bound}
		\|\hat{\vect{H}}_{k+1}^{-1}\|=(\Lambda_{\min}(\hat{\vect{H}}_{k+1}))^{-1}\leq \big(\Lambda_{\min}(\vect{H})-\|\hat{\vect{H}}_{k+1}-\vect{H}\|\big)^{-1}\leq 2/\Lambda_{\min}(\vect{H}).
	\end{equation}

\end{proof}

\begin{lemma}	\label{lem:huber_grad_concen}
	Under the same assumptions as in Theorem \ref{thm:contam_rate}, for every $\nu>0$, there exist constants $C$ and $c$ such that:
    \begin{itemize}
        \item[i)] If $\delta>0$ and $\tau_k$ can be arbitrary, then 
        \begin{equation*}
        		\mbP\Big(\Big|\frac{1}{k+1}\sum_{i=1}^{k+1}\big\{\vect{X}_{i}g_{\tau_i}(\vect{Z}_{i}^{\tp}\vect{\theta}^*-b_{i})\big\}\Big|_2\geq C\sqrt{d}e_{k+1}\Big)\leq c(k+1)^{-\nu}.
	\end{equation*}
        Here $e_{k+1}$ is defined in \eqref{eq:ek_case1}.
        \item[ii)] If $\delta>1$ and $\sqrt{k/\log^3k}=O(\tau_k)$, then
        \begin{equation*}
        		\mbP\Big(\Big|\frac{1}{k+1}\sum_{i=1}^{k+1}\big\{\vect{X}_{i}g_{\tau_i}(\vect{Z}_{i}^{\tp}\vect{\theta}^*-b_{i})\big\}\Big|_2\geq C\sqrt{d}e_{k+1}\Big)\leq c(k+1)^{-\min\{\nu, (\delta-1)/3\}}.
	\end{equation*}
        Here $e_{k+1}$ is defined in \eqref{eq:conv_g2}.
    \end{itemize}
\end{lemma}

\begin{proof}
	\noindent\textbf{Proof of i).} We prove the bound coordinate-wisely. When $(\vect{X},\vect{Z},b)$ has no outlier, for the $l$-th coordinate, we first prove that for every $\tau>0$, there is
	\begin{equation*}
		|\mbE[X_lg_{\tau}(\vect{Z}^{\tp}\vect{\theta}^*-b)]|\leq C_1\tau^{-\delta'}.
	\end{equation*}
	Indeed,  $\epsilon=\vect{Z}^{\tp}\vect{\theta}^*-b$, we have that
	\begin{align*}
		|\mbE[X_lg_{\tau}(\epsilon)]|\leq& |\mbE[X_l\epsilon]|+ \big|\mbE[X_l|\epsilon - g_{\tau}(\epsilon)|\mbI(|\epsilon|\leq\tau)]\big| +\big|\mbE[X_l|g_{\tau}(\epsilon)-\epsilon|\mbI(|\epsilon|>\tau)]\big|\\
		\leq&\frac{1}{2}\tau^{-\delta'}\big|\mbE[X_l|\epsilon|^{1+\delta'}\mbI(|\epsilon|\leq\tau)]\big| +\big|\mbE[X_l|\epsilon|\mbI(|\epsilon|>\tau)]\big|\\
		\leq&\frac{1}{2}\tau^{-\delta'}\big|\mbE\big[X_l\mbE[|\epsilon|^{1+\delta'}\mbI(|\epsilon|\leq\tau)|\vect{X},\vect{Z}]\big]\big| +\big|\mbE\big[X_l\mbE[|\epsilon|\mbI(|\epsilon|>\tau)|\vect{X},\vect{Z}]\big]\big|\\
		\leq&C_1\tau^{-\delta'}.
	\end{align*}
	Therefore, by the choice of $\tau_i$ and Lemma \ref{lem:seq_bound}, we have
	\begin{equation}	\label{eq:huber_grad_exp}
		\Big|\frac{1}{k+1}\sum_{i=1}^{k+1}\mbE[X_{i,l}g_{\tau_i}(\vect{Z}^{\tp}_i\vect{\theta}^*-b_i)]\Big|\leq \frac{1}{k+1}\sum_{i=1}^{k+1}\Big|\mbE[X_{i,l}g_{\tau_i}(\vect{Z}^{\tp}_i\vect{\theta}^*-b_i)]\Big|\leq C_2\tau_{k+1}^{-\delta'}.
	\end{equation}
	Next, when all data are normal, we prove the rate of
	\begin{equation*}	
        		\Big|\frac{1}{k+1}\sum_{i=1}^{k+1}\big\{X_{i,l}g_{\tau_i}(\epsilon_i)-\mbE[X_{i,l}g_{\tau_i}(\epsilon_i)]\big\}\Big|.
	\end{equation*}
	We basically rehash the proof in Lemma \ref{lem:concen_mix}. 
	
	\noindent\textbf{Case 1: $\delta\in(0,1]$.} We divide the $k$-tuple $(1,\dots,k)$ into $m_k$ different subsets $H_1,\dots,H_{m_k}$, where $m_k = \lceil k/\lceil \lambda\log k\rceil\rceil$. Here $|H_i|=\lceil\lambda\log k\rceil$ for $1\leq i\leq m_k-1$, and $|H_{m_j}|\leq\lceil\lambda\log k\rceil$. Then we have $m_k\approx k/(\lambda\log k)$. Without loss of generality, we assume that $m_k$ is an even integer.
    
    Let $Y_i=X_{i,l}g_{\tau_i}(\epsilon_i)-\mbE[X_{i,l}g_{\tau_i}(\epsilon_i)]$ and $\xi_q=\frac{1}{\lceil\lambda\log k\rceil}\sum_{i\in H_{2q-1}}Y_i$, then we know $|\xi_q|\leq 2M\tau_{k+1}$ (since $X_{i,l}g_{\tau_i}(\epsilon)\leq |\vect{X}|_2\tau_i\leq M\tau_{k+1}$) and $\mbE[\xi_l]=0$ holds. For all $B_1\in\sigma(\xi_1,\dots,\xi_q)$ and $B_2\in\sigma(\xi_{q+1})$, there holds $|\mbP(B_2|B_1)-\mbP(B_2)|\leq\phi(\lceil\lambda\log k\rceil)$ for all $q$. By Theorem 2 of \cite{berkes_philipp.1979aop}, there exists a sequence of independent variables $\eta_l$, $l\geq 1$ with $\eta_l$ having the same distribution as $\xi_l$, and
    \begin{equation*}
        \mbP\Big(|\xi_{q}-\eta_{q}|\geq 6\phi(\lceil\lambda\log k\rceil)\Big)\leq  6\phi(\lceil\lambda\log k\rceil).
    \end{equation*}
    For any $\nu>0$, we can take $\lambda\geq(\nu +1)/|\log\rho|$ so that
    \begin{equation}    \label{eq:huber_berke_rv}
    \begin{aligned}
        &\mbP\Big(\Big|\frac{2}{ m_k}\sum_{q=1}^{ m_{k}/2}(\xi_{q}-\eta_{q})\Big|\geq C_3k^{-\nu}\Big)\\
        \leq&\mbP\Big(\Big|\frac{2}{ m_k}\sum_{q=1}^{ m_{k}/2}(\xi_{q}-\eta_{q})\Big|\geq 6\phi(\lceil\lambda\log k\rceil)\Big)\leq 3m_{k}\phi(\lceil\lambda\log k\rceil)\leq C_{3}k^{-\nu},
    \end{aligned}
    \end{equation}
    where $C_3>0$ is some constant. Next, we bound the variance of $\xi_q$. From equation (20.23) of \cite{billingsley1968convergence} we know that for arbitrary $i,j$, there is
    \begin{align*}
        &\big|\mbE\big[Y_iY_j\big]\big|\leq 2\sqrt{\phi(|i-j|)}\sqrt{\mbE[Y_i^2]\mbE[Y_j^2]} \leq 2C_4\rho^{|i-j|/2}M^2\tau_{k+1}^{1-\delta}.
    \end{align*}
    Then we compute that
    \begin{align*}
        \var(\eta_q)=\var(\xi_q) =& \frac{1}{\lceil\lambda\log k\rceil^2}\Big[\sum_{i,j\in H_{2q-1}}\{\mbE(Y_iY_j)\}\Big]\\
        \leq&\frac{2C_4}{\lceil\lambda\log k\rceil^2}\sum_{i,j=1}^{\lceil\lambda\log k\rceil}\rho^{|i-j|/2}M^2\tau_{k+1}^{1-\delta}\\
        \leq&\frac{2\lceil\lambda\log k\rceil C_4}{\lceil\lambda\log k\rceil^2}\Big(\frac{2}{1-\sqrt{\rho}}-1\Big)M^2\tau_{k+1}^{1-\delta}\leq\frac{4C_4}{(1-\sqrt{\rho})\lceil\lambda\log k\rceil}M^2\tau_{k+1}^{1-\delta}.
    \end{align*}
    Notice that $\eta_q$'s are all uniformly bounded by $2M\tau_{k+1}$, we can apply Bernstein's inequality (\citealp{Bennett.1962jasa}) and yield
    \begin{align*}
        &\mbP\Big(\Big|\frac{2}{m_k}\sum_{q=1}^{m_k/2}\eta_q\Big|\geq t\Big)\leq 2\exp\Big(-\frac{m_kt^2/4}{\var(\eta_q)/2+M\tau_kt/3}\Big)\\
        \leq&\exp\Big(-\frac{m_k\lceil\lambda\log k\rceil t^2/4}{2M^2C_{\epsilon}\tau_{k+1}^{1-\delta}/(1-\sqrt{\rho})+2\lceil\lambda\log k\rceil M\tau_{k+1}t/3}\Big).
    \end{align*}
    We take $t=C(\sqrt{\tau_{k+1}^{1-\delta}\log k/(k+1)}+\tau_{k+1}\log^2 k/(k+1))$ for some $C$ large enough (note that $k\approx m_k\lceil\lambda\log k \rceil$). Then we have that
    \begin{equation}    \label{eq:huber_ind_bern}
        \mbP\Big(\Big|\frac{2}{m_k}\sum_{q=1}^{m_k/2}\eta_q\Big|\geq t\Big)=O(k^{-\nu}).
    \end{equation}
    Combining \eqref{eq:huber_berke_rv} and \eqref{eq:huber_ind_bern} we have that
    \begin{equation}    \label{eq:huber_odd_avg}
    \begin{aligned}
        &\mbP\Big(\Big|\frac{2}{m_k}\sum_{q=1}^{m_k/2}\xi_q\Big|\geq C\sqrt{\frac{\tau^{1-\delta}_{k+1}\log k}{k+1}}+C\frac{\log^2 k\tau_{k+1}}{k+1}\Big)=O(k^{-\nu}).
    \end{aligned}
    \end{equation}
    Next we denote $\tilde{\xi}=\frac{1}{\lceil\lambda\log k\rceil}\sum_{i\in H_{2q}}Y_i$, then we can similarly show that
    \begin{equation}    \label{eq:huber_even_avg}
    \begin{aligned}
        &\mbP\Big(\Big|\frac{2}{m_k}\sum_{q=1}^{m_k/2}\tilde{\xi}_q\Big|\geq C\sqrt{\frac{\tau^{1-\delta}_{k+1}\log k}{k+1}}+C\frac{\log^2 k\tau_{k+1}}{k+1}\Big)=O(k^{-\nu}),
    \end{aligned}
    \end{equation}
    Combining \eqref{eq:huber_odd_avg} and \eqref{eq:huber_even_avg} we can prove that
    \begin{equation}	\label{eq:recentered_rate}
    	\mbP\Bigg(\Big|\frac{1}{k+1}\sum_{i=1}^{k+1}\big\{X_{i,l}g_{\tau_i}(\epsilon_i)-\mbE[X_{i,l}g_{\tau_i}(\epsilon_i)]\big\}\Big|\geq C\Big(\sqrt{\frac{\tau^{1-\delta}_{k+1}\log k}{k+1}}+\frac{\log^2 k\tau_{k+1}}{k+1}\Big)\Bigg)=O(k^{-\nu}).
    \end{equation}
    By combining it with \eqref{eq:huber_grad_exp} we have that
    \begin{equation*}
    	\mbP\Bigg(\Big|\frac{1}{k+1}\sum_{i=1}^{k+1}\big\{\vect{X}_{i}g_{\tau_i}(\vect{Z}_{i}^{\tp}\vect{\theta}^*-b_{i})\big\}\Big|_2 \geq C \Big(\sqrt{d}\tau_{k+1}^{-\delta'}+\sqrt{\frac{d\tau^{1-\delta}_{k+1}\log k}{k+1}}+\frac{\sqrt{d}\log^2 k\tau_{k+1}}{k+1}\Big)\Bigg)=O(k^{-\nu}).
    \end{equation*}
    
    \noindent\textbf{Case 2: $\delta>1$}. Similarly, we can obtain the rate 
    \begin{equation*}
    	\mbP\Bigg(\Big|\frac{1}{k+1}\sum_{i=1}^{k+1}\big\{X_{i,l}g_{\tau_i}(\epsilon_i)-\mbE[X_{i,l}g_{\tau_i}(\epsilon_i)]\big\}\Big|\geq C\Big(\sqrt{\frac{\log k}{k+1}}+\frac{\log^2 k\tau_{k+1}}{k+1}\Big)\Bigg)=O(k^{-\nu}),
    \end{equation*}    
    by directly plugging $\tau_k$ into \eqref{eq:recentered_rate}. Then we have that
    \begin{equation*}
    	\mbP\Bigg(\Big|\frac{1}{k+1}\sum_{i=1}^{k+1}\big\{\vect{X}_{i}g_{\tau_i}(\vect{Z}_{i}^{\tp}\vect{\theta}^*-b_{i})\big\}\Big|_2 \geq C\Big(\tau_{k+1}^{-\delta'}+\sqrt{\frac{\log k}{k+1}}+\frac{\log^2 k\tau_{k+1}}{k+1}\Big)\Bigg)=O(k^{-\nu}).
    \end{equation*}
    
    \noindent\textbf{Proof of ii).} When $\delta>1$ and $\sqrt{k/\log^3k} = O(\tau_k)$, we define the thresholding level $\tilde{\tau}=C_{\tau}\sqrt{k+1}/\log^{3/2}k$, and consider the truncated random variables $\vect{X}_ig_{\tau_i}(\epsilon_i)\mbI(|\epsilon_i|\leq\tilde{\tau})$. Here $C_{\tau}$ is a sufficiently large constant. Then we have that
    \begin{align*}
    	&\mbP\Big(\sum_{i=1}^{k+1}\vect{X}_ig_{\tau_i}(\epsilon_i)\neq \sum_{i=1}^{k+1}\vect{X}_ig_{\tau_i}(\epsilon_i)\mbI(|\epsilon_i|\leq\tilde{\tau})\Big)\\
	\leq&\mbP\Big(\cup_{i=1}^{k+1}\{\vect{X}_ig_{\tau_i}(\epsilon_i)\neq \vect{X}_ig_{\tau_i}(\epsilon_i)\mbI(|\epsilon_i|\leq\tilde{\tau})\}\Big)\\
	\leq&(k+1)\max_{1\leq i\leq k+1}\mbP\Big(|\epsilon_i|>\tilde{\tau}\Big) = O((k+1)^{-(\delta-1)/3}).\stepcounter{equation}\tag{\theequation}\label{eq:thred_coincide}
    \end{align*}
    Similarly as in \eqref{eq:huber_grad_exp}, we have that
    \begin{equation}	\label{eq:huberT_grad_exp}
    	\Big|\frac{1}{k+1}\sum_{i=1}^{k+1}\mbE[X_{i,l}g_{\tau_i}(\epsilon_i)\mbI(|\epsilon_i|\leq\tilde{\tau})]\Big|\leq \frac{1}{k+1}\sum_{i=1}^{k+1}\Big|\mbE[X_{i,l}g_{\tau_i}(\epsilon_i)\mbI(|\epsilon_i|\leq\tilde{\tau})]\Big|\leq C(\tau_{k+1}^{-\delta'}+\tilde{\tau}^{-\delta}).
    \end{equation}
    Similarly as in the proof of \eqref{eq:recentered_rate}, for every $\nu>0$, there exists $C>0$ such that
    \begin{align*}
        &\mbP\Bigg(\Big|\frac{1}{k+1}\sum_{i=1}^{k+1}\big\{X_{i,l}g_{\tau_i}(\epsilon_i)\mbI(|\epsilon_i|\leq\tilde{\tau})-\mbE[X_{i,l}g_{\tau_i}(\epsilon_i)\mbI(|\epsilon_i|\leq\tilde{\tau})]\big\}\Big|\geq C\big(\sqrt{\frac{\log k}{k+1}}+\frac{\log^2k\tilde{\tau}}{k+1}\big)\Bigg)=O(k^{-\nu})\stepcounter{equation}\tag{\theequation}\label{eq:Trecentered_rate}
    \end{align*}
    By the choice of $\tilde{\tau}$, we know that 
    \begin{equation*}
        \sqrt{\frac{\log k}{k+1}}+\frac{\log^2 k\tilde{\tau}}{k+1}=O\Big(\sqrt{\frac{\log k}{k+1}}\Big).
    \end{equation*}
    Combining \eqref{eq:thred_coincide}, \eqref{eq:huberT_grad_exp} and \eqref{eq:Trecentered_rate} we have that
    \begin{equation*}
    	\mbP\Bigg(\Big|\frac{1}{k+1}\sum_{i=1}^{k+1}\big\{\vect{X}_{i}g_{\tau_i}(\vect{Z}_{i}^{\tp}\vect{\theta}^*-b_{i})\big\}\Big|_2 \geq C\sqrt{d}\Big(\tau_{k+1}^{-\delta'}+\tilde{\tau}^{-\delta}+\sqrt{\frac{\log k}{k+1}}\Big)\Bigg)=O((k+1)^{-\min\{\nu,(\delta-1)/3\}}).
    \end{equation*}    
    When there are $\alpha_k$ fraction of outliers, denote $\mcQ_k$ as the index set, we have
    \begin{align*}
        &\Big|\frac{1}{k+1}\sum_{i=1}^{k+1}\big\{\vect{X}_{i}g_{\tau_i}(\vect{Z}_{i}^{\tp}\vect{\theta}^*-b_{i})\big\}\Big|_2\\
        \leq&\frac{1}{k+1}\Big|\sum_{i\in\mcQ_k}\big\{\vect{X}_{i}g_{\tau_i}(\vect{Z}_{i}^{\tp}\vect{\theta}^*-b_{i})\big\}\Big|_2+\frac{1}{k+1}\Big|\sum_{i\notin\mcQ_k}\big\{\vect{X}_{i}g_{\tau_i}(\vect{Z}_{i}^{\tp}\vect{\theta}^*-b_{i})\big\}\Big|_2\\
        =&O\big(\sqrt{d}\alpha_k\tau_{k+1}\big) + \frac{1}{k+1}\Big|\sum_{i\notin\mcQ_k}\big\{\vect{X}_{i}g_{\tau_i}(\vect{Z}_{i}^{\tp}\vect{\theta}^*-b_{i})\big\}\Big|_2,
    \end{align*}
    which proves the lemma.
\end{proof}

\begin{lemma}   \label{lem:mix_bound}
    (Bound of the mixed term) Under the same assumptions as in Theorem \ref{thm:contam_rate}, for every $\nu>0$, there exist constants $C$ and $c$, such that
    \begin{equation*}
        \mbP\Big(\Big|\frac{1}{k+1}\sum_{i=1}^{k+1}(\vect{A}_i-\vect{H})(\hat{\vect{\theta}}_{i-1}-\vect{\theta}^*)\Big|_2\geq C(\alpha_k+\Psi^2de_k^2\log k),\cap_{i=n_0}^kE_i\Big)\leq c(k+1)^{-\nu}.
    \end{equation*}
\end{lemma}

\begin{proof}
    Firstly, from \eqref{eq:EAH_diff} and Lemma \ref{lem:seq_bound}, under the event $\cap_{i=n_0}^kE_i$, we can bound the term
    \begin{align*}
        &\Big|\frac{1}{k+1}\sum_{i=1}^{k+1}(\mbE[\vect{A}_i]-\vect{H})(\hat{\vect{\theta}}_{i-1}-\vect{\theta}^*)\Big|_2\\
        \leq&\frac{1}{k+1}\sum_{i=1}^{k+1}\|\mbE[\vect{A}_i]-\vect{H}\|\cdot|\hat{\vect{\theta}}_{i-1}-\vect{\theta}^*|_2\\
        \leq&\frac{1}{k+1}\sum_{i=1}^{k+1}C\Psi\tau_{i}^{-\delta'}e_{i-1}=O(\tau_{k+1}^{-\delta'}\sqrt{d}e_k\log k)=O(\sqrt{d}e_k^2\log k).\stepcounter{equation}\tag{\theequation}\label{eq:eA_H}
    \end{align*}
    For ease of notation, we first consider the case when there is no outlier. Similar as in the proof of Lemma \ref{lem:concen_mix}, for each $k$, we evenly divide the tuple $\{n_0,\dots,k\}$ into $m_k$ subsets $H_1,\dots,H_{m_k}$, where $m_k = \lceil (k-n_0)/\lceil \lambda\log k\rceil\rceil$. Here $|H_q|=\lceil\lambda\log k\rceil$ for $1\leq q\leq m_k-1$, and $|H_{m_k}|\leq\lceil\lambda\log k\rceil$. $\lambda$ is a sufficiently large constant which will be specified later. Then we have $m_k\approx (k-n_0)/(\lambda\log k)$. We further denote $H_0=\{1,\dots,n_0\}$. Without loss of generality, we assume that $m_k$ is an even integer. For each $i\in\{n_0,\dots,k\}$, suppose $i\in H_{l_i}$, we construct the following random variable 
    \begin{equation}    \label{eq:theta_tilde}
        \tilde{\vect{\theta}}_i=\frac{1}{i}\sum_{q=0}^{l_i-2}\sum_{j\in H_q}(\hat{\vect{\theta}}_j-\vect{\theta}^*)+\vect{H}^{-1}\frac{1}{i}\sum_{q=0}^{l_i-2}\sum_{j\in H_q}\vect{X}_{j+1}g_{\tau_{j+1}}(\vect{Z}_{j+1}^{\tp}\hat{\vect{\theta}}_{j}-b_{j+1}).
    \end{equation}
    When $l_j=1$, we take the sum for the terms in $H_0$. For $i\in H_0$, we take $\tilde{\vect{\theta}}_i=\hat{\vect{\theta}}_i-\vect{\theta}^* = \hat{\vect{\theta}}_0-\vect{\theta}^*$. 
    Then by Lemma \ref{lem:mix_remainder} below we have that,
    \begin{equation}   `   \label{eq:main_approx} 
    \begin{aligned}
        &\mbP\Big(\Big|\frac{1}{k+1}\sum_{i=1}^{k+1}(\vect{A}_i-\mbE[\vect{A}_i])(\hat{\vect{\theta}}_{i-1}-\vect{\theta}^*)-\frac{1}{k+1}\sum_{i=1}^{k+1}(\vect{A}_i-\mbE[\vect{A}_i])\tilde{\vect{\theta}}_{i-1}\Big|_2\geq C\Psi^2de_k^2,\cap_{i=n_0}^kE_i\Big)\\
        =& O((k+1)^{-\nu}).
    \end{aligned}
    \end{equation}
    Moreover, from the proof of Lemma \ref{lem:mix_remainder}, we can obtain that under $\cap_{i=n_0}^kE_i$, there is 
    \begin{equation}    \label{eq:tilde_approx}
        |\hat{\vect{\theta}}_i-\vect{\theta}^*-\tilde{\vect{\theta}}_i|_2\leq \Psi\sqrt{d}e_i.
    \end{equation}
    It left to bound the term $\frac{1}{k+1}\sum_{j=0}^k(\vect{A}_i-\mbE[\vect{A}_i])\tilde{\vect{\theta}}_i$. To be more precise, we define the $\sigma$-field $\mcG_l=\sigma((\vect{X}_i,\vect{Z}_i,b_i):i\in\cup_{j=1}^{2l+1}H_j)$, and construct the random variable
    \begin{equation*}
        \vect{\xi}_l = \sum_{i\in H_{2l+1}}(\vect{A}_{i+1}-\mbE[\vect{A}_{i+1}])\tilde{\vect{\theta}}_{i}
    \end{equation*}
    Then by \eqref{eq:tilde_approx} we know that
    \begin{equation*}
        \mbP\Big(\cup_i\Big\{|\tilde{\vect{\theta}}_i|_2\geq 2\Psi\sqrt{d}e_i\Big\},\cap_{j=n_0}^k E_j\Big)=O((k+1)^{-\nu}).
    \end{equation*}
    We further set
    \begin{equation*}
        \tilde{\vect{\xi}}_l = \sum_{i\in H_{2l+1}}(\vect{A}_{i+1}-\mbE[\vect{A}_{i+1}])\tilde{\vect{\theta}}_i\mbI\Big\{|\tilde{\vect{\theta}}_i|_2\leq 2\Psi\sqrt{d}e_i\Big\},
    \end{equation*}
    then there is
    \begin{equation}    \label{eq:xi_coin}          
    \mbP\Big(\sum_{l=1}^{m_k/2}\vect{\xi}_l\neq \sum_{l=1}^{m_k/2}\tilde{\vect{\xi}}_l,\cap_{i=n_0}^kE_i\Big)=O(k^{-\nu}).
    \end{equation}
    Notice that $\{\tilde{\vect{\xi}}_l-\mbE[\tilde{\vect{\xi}}_l|\mcG_{l-1}],l\geq 1\}$ are martingale differences, and there is $\mbE[\tilde{\vect{\theta}}_{i}|\mcG_{l-1}]=\tilde{\vect{\theta}}_i$ for $i\in H_{2l+1}$. Therefore we have
    \begin{equation*}
        \mbE[\tilde{\vect{\xi}}_l|\mcG_{l-1}]=\sum_{i\in H_{2l+1}}\mbE[(\vect{A}_{i+1}-\mbE[\vect{A}_{i+1}])|\mcG_{l-1}]\tilde{\vect{\theta}}_i\mbI\Big\{|\tilde{\vect{\theta}}_i|_2\leq 2\Psi\sqrt{d}e_i\Big\}.
    \end{equation*}
    By Lemma \ref{lem:mix_norm} there is
    \begin{equation*}
        \mbE\Big[\big\|\mbE[\vect{A}_{i+1}-\mbE[\vect{A}_{i+1}]|\mcG_{l-1}]\big\|\Big]\leq Cd\sqrt{\phi(\lceil\lambda\log k\rceil)}=O(k^{-\nu-2}),
    \end{equation*}
    for some $\lambda$ large enough. Then Markov's inequality yields
    \begin{equation}    \label{eq:xi_exp_bound}
        \mbP\Big(\Big|\frac{1}{k}\sum_{l=1}^{m_k/2}\mbE[\tilde{\vect{\xi}}_l|\mcG_{l-1}]\Big|_2\geq C\frac{\sqrt{d}e_k}{k},\cap_{i=n_0}^kE_i\Big)=O(k^{-\nu}).
    \end{equation}
    It is direct to verify that
    \begin{align*}
        &\sup_{\vect{v}\in\mbS^{d-1}}|\vect{v}^{\tp}(\tilde{\vect{\xi}}_l-\mbE[\tilde{\vect{\xi}}_l|\mcG_{l-1}])|\leq C\Psi\lambda\log k,\\
        &\sup_{\vect{v}\in\mbS^{d-1}}\sum_{l=1}^{m_k/2}\var[|\vect{v}^{\tp}\tilde{\vect{\xi}}_l||\mcG_{l-1}]\leq\sup_{\vect{v}\in\mbS^{d-1}}\sum_{l=1}^{m_k/2}\mbE[|\vect{v}^{\tp}\tilde{\vect{\xi}}_l|^2|\mcG_{l-1}]\leq C\Psi^2kde_k^2\log k.
    \end{align*}
    Then we can apply Lemma \ref{lem:mart_diff_concen} and yield
    \begin{align*}
        &\mbP\Big(\Big|\sum_{l=1}^{m_k/2}(\tilde{\vect{\xi}}_l-\mbE[\tilde{\vect{\xi}}_l|\mcG_{l-1}])\Big|_2\geq C\Psi(d\log^2k+de_k\sqrt{k\log^2 k})\Big)=O(k^{-\nu d}).
    \end{align*}
    Combining it with \eqref{eq:xi_exp_bound} and \eqref{eq:xi_coin}, we have that
    \begin{equation*}
        \mbP\Big(\Big|\frac{1}{k+1}\sum_{l=1}^{m_k/2}\vect{\xi}_l\Big|_2\geq C\Psi\Big(\frac{d\log^2k}{k}+de_k\sqrt{\frac{\log^2 k}{k}}\Big),\cap_{i=n_0}^kE_i\Big)=O((k+1)^{-\nu})
    \end{equation*}
    A similar result holds for the even term. Note that $\frac{d\log^2k}{k}+de_k\sqrt{\frac{\log^2 k}{k}}=O(de_k^2)$, and together with \eqref{eq:main_approx} we have 
    \begin{equation*}
        \mbP\Big(\Big|\frac{1}{k+1}\sum_{j=0}^k(\vect{A}_{i+1}-\mbE[\vect{A}_{i+1}])(\hat{\vect{\theta}}_i-\vect{\theta}^*)\Big|_2\geq C\Psi^2de_k^2\log k,\cap_{i=n_0}^kE_i\Big)=O((k+1)^{-\nu}).
    \end{equation*}
    When taking the outliers into consideration, we have that
    \begin{align*}
    	&\Big|\frac{1}{k+1}\sum_{j=0}^k(\vect{A}_{i+1}-\mbE[\vect{A}_{i+1}])(\hat{\vect{\theta}}_i-\vect{\theta}^*)\Big|_2\\
	=&\Big|\frac{1}{k+1}\sum_{j\in\mcQ_k}(\vect{A}_{i+1}-\mbE[\vect{A}_{i+1}])(\hat{\vect{\theta}}_i-\vect{\theta}^*)\Big|_2+\Big|\frac{1}{k+1}\sum_{j\notin\mcQ_k}(\vect{A}_{i+1}-\mbE[\vect{A}_{i+1}])(\hat{\vect{\theta}}_i-\vect{\theta}^*)\Big|_2\\
	=&\Big|\frac{1}{k+1}\sum_{j\notin\mcQ_k}(\vect{A}_{i+1}-\mbE[\vect{A}_{i+1}])(\hat{\vect{\theta}}_i-\vect{\theta}^*)\Big|_2+O\big(\alpha_k),
    \end{align*}
    under the event $\cap_{i=n_0}^kE_i$, which proves the lemma by combining it with \eqref{eq:eA_H}.
\end{proof}

\begin{lemma}   \label{lem:mix_remainder}
    Let $\tilde{\vect{\theta}}_{i}$ be defined in \eqref{eq:theta_tilde}, for every $\nu>0$, there exist constants $C$ and $c$, such that
    \begin{equation*}    
    \begin{aligned}
        &\mbP\Big(\Big|\frac{1}{k+1}\sum_{i=1}^{k+1}(\vect{A}_i-\mbE[\vect{A}_i])(\hat{\vect{\theta}}_{i-1}-\vect{\theta}^*)-\frac{1}{k+1}\sum_{i=1}^{k+1}(\vect{A}_i-\mbE[\vect{A}_i])\tilde{\vect{\theta}}_{i-1}\Big|_2\geq C\Psi^2de_k^2,\cap_{i=n_0}^kE_i\Big)\\
        \leq& c(k+1)^{-\nu}.
    \end{aligned}
    \end{equation*}
\end{lemma}

\begin{proof}
    Under event $\cap_{i=n_0}^kE_i$, using Lemma \ref{lem:seq_bound}, we have
    \begin{align*}
        &\big|\hat{\vect{\theta}}_{i}-\vect{\theta}^*-\tilde{\vect{\theta}}_i\big|_2\\
        =&\Big|\frac{1}{i}\sum_{j=(l_i-2)\lceil\lambda\log k\rceil+n_0+1}^{i-1}(\hat{\vect{\theta}}_j-\vect{\theta}^*)\Big|_2+\Big|\vect{H}^{-1}\frac{1}{i}\sum_{j=(l_i-2)\lceil\lambda\log k\rceil+n_0+1}^{i-1}\vect{X}_{j+1}g_{\tau_{j+1}}(\vect{Z}_{j+1}^{\tp}\hat{\vect{\theta}}_j-b_{j+1})\Big|_2\\
        &+\Big|(\hat{\vect{H}}_{i-1}^{-1}-\vect{H}^{-1})\frac{1}{i}\sum_{j=0}^{i-1}\vect{X}_{j+1}g_{\tau_{j+1}}(\vect{Z}_{j+1}^{\tp}\hat{\vect{\theta}}_j-b_{j+1})\Big|_2\\
        =&\Big|(\hat{\vect{H}}_{i-1}^{-1}-\vect{H}^{-1})\frac{1}{i}\sum_{j=0}^{i-1}\vect{X}_{j+1}g_{\tau_{j+1}}(\vect{Z}_{j+1}^{\tp}\hat{\vect{\theta}}_j-b_{j+1})\Big|_2\\
        &+\Big|\vect{H}^{-1}\frac{1}{i}\sum_{j=(l_i-2)\lceil\lambda\log k\rceil+n_0+1}^{i-1}\vect{X}_{j+1}g_{\tau_{j+1}}(\vect{Z}_{j+1}^{\tp}\vect{\theta}^*-b_{j+1})\Big|_2+O\Big(\frac{e_i\log k}{i}\Big).\stepcounter{equation}\tag{\theequation}\label{eq:tilde_diff}
    \end{align*}
    For the first term, by Lemma \ref{lem:huber_concen_hess} and \eqref{eq:huber_inv_bound} we have that under event $\cap_{i=n_0}^kE_i$,
    \begin{equation*}
        \|\hat{\vect{H}}_{i-1}^{-1}-\vect{H}^{-1}\|\leq \|\hat{\vect{H}}_{i-1}^{-1}\|\cdot\|\hat{\vect{H}}_{i-1}-\vect{H}\|\cdot\|\vect{H}^{-1}\|\leq C\Psi\sqrt{d}e_i.
    \end{equation*}
    On the other hand, there is
    \begin{align*}
        &\Big|\frac{1}{i}\sum_{j=0}^{i-1}\vect{X}_{j+1}g_{\tau_{j+1}}(\vect{Z}_{j+1}^{\tp}\hat{\vect{\theta}}_j-b_{j+1})\Big|_2\\
        \leq&\Big|\frac{1}{i}\sum_{j=0}^{i-1}\vect{X}_{j+1}g_{\tau_{j+1}}(\vect{Z}_{j+1}^{\tp}\vect{\theta}^*-b_{j+1})\Big|_2+M^2\frac{1}{i}\sum_{j=0}^{i-1}|\hat{\vect{\theta}}_j-\vect{\theta}^*|_2\leq C\Psi\sqrt{d}e_i\log i.
    \end{align*}
    Therefore, we have that
    \begin{align*}
        &\frac{1}{k+1}\sum_{i=1}^{k+1}(\vect{A}_i-\mbE[\vect{A}_i])(\hat{\vect{\theta}}_{i-1}-\vect{\theta}^*)-\frac{1}{k+1}\sum_{i=1}^{k+1}(\vect{A}_i-\mbE[\vect{A}_i])\tilde{\vect{\theta}}_{i-1}\\
        =&\frac{1}{k+1}\sum_{i=1}^{k+1}(\vect{A}_i-\mbE[\vect{A}_i])\vect{H}^{-1}\frac{1}{i}\sum_{j=(l_i-2)\lceil\lambda\log k\rceil+n_0+1}^{i-1}\vect{X}_{j+1}g_{\tau_{j+1}}(\vect{Z}_{j+1}^{\tp}\vect{\theta}^*-b_{j+1})+O(\Psi^2de_i^2)\\
        =&\frac{1}{k+1}\sum_{i=1}^{k+1}\Big\{\sum_{j=(l_i-2)\lceil\lambda\log k\rceil+n_0+1}^{i}\frac{1}{j}(\vect{A}_j-\mbE[\vect{A}_j])\Big\}\vect{H}^{-1}\vect{X}_{i}g_{\tau_{i}}(\vect{Z}_{i}^{\tp}\vect{\theta}^*-b_{i})+O(\Psi^2de_i^2).\stepcounter{equation}\tag{\theequation}\label{eq:remain_main}
    \end{align*}
    Denote
    \begin{equation*}
        \vect{Y}_i = \Big\{\sum_{j=(l_i-2)\lceil\lambda\log k\rceil+n_0+1}^{i}\frac{1}{j}(\vect{A}_j-\mbE[\vect{A}_j])\Big\}\vect{H}^{-1}\vect{X}_{i}g_{\tau_{i}}(\vect{Z}_{i}^{\tp}\vect{\theta}^*-b_{i}).
    \end{equation*}
    It is direct to verify that
    \begin{align*}
        |\mbE[Y_i]|_2\leq& \sqrt{\mbE\Big\|\sum_{j=(l_i-2)\lceil\lambda\log k\rceil+n_0+1}^{i}\frac{1}{j}(\vect{A}_j-\mbE[\vect{A}_j])\Big\|^2\mbE\big|\vect{H}^{-1}\vect{X}_{i}g_{\tau_{i}}(\vect{Z}_{i}^{\tp}\vect{\theta}^*-b_{i})\big|_2^2}\\
        =&O\Big(\frac{\sqrt{d}\log k}{i}\tau_i^{(1-\delta)_+/2}\Big)=O(\sqrt{d}e_i^2).\stepcounter{equation}\tag{\theequation}\label{eq:exp_remain_bound}
    \end{align*}
    Therefore, it left bound
    \begin{align*}
        \Big|\frac{1}{k+1}\sum_{i=1}^{k+1}(Y_i-\mbE[Y_i])\Big|_2.
    \end{align*}
    To this end, we further denote
    \begin{align*}
        \vect{\xi}_l = \sum_{i\in H_{4l+1}}(\vect{Y}_i-\mbE[\vect{Y}_i]),
    \end{align*}
    We define the $\sigma$-field $\mcG_l=\sigma((\vect{X}_i,\vect{Z}_i,b_i):i\in\cup_{j=1}^{4l+1}H_j)$, by Lemma \ref{lem:mix_norm} we have that
    \begin{equation*}
        \mbE\Big[\Big|\mbE[\vect{\xi}_l|\mcG_{l-1}]\Big|_2\Big]=O\Big(\frac{\sqrt{d}\tau_{l}\log k}{lk^{\nu+2}}\Big).
    \end{equation*}
    Therefore, by Markov's inequality
    \begin{equation}    \label{eq:xi_remain_exp_bound}
        \mbP\Big(\Big|\frac{1}{k}\sum_{l=1}^{m_k/4}\mbE[\vect{\xi}_l|\mcG_{l-1}]\Big|_2\geq C\frac{\sqrt{d}\tau_k\log k}{k^2},\cap_{i=n_0}^kE_i\Big)=O(k^{-\nu}).
    \end{equation}
    It is direct to verify that
    \begin{align*}
        &\sup_{\vect{v}\in\mbS^{d-1}}|\vect{v}^{\tp}(\vect{\xi}_l-\mbE[\vect{\xi}_l|\mcG_{l-1}])|\leq C\sqrt{d}\log k,\\
        &\sup_{\vect{v}\in\mbS^{d-1}}\sum_{l=1}^{m_k/2}\var[|\vect{v}^{\tp}\vect{\xi}_l||\mcG_{l-1}]\leq\sup_{\vect{v}\in\mbS^{d-1}}\sum_{l=1}^{m_k/2}\mbE[|\vect{v}^{\tp}\vect{\xi}_l|^2|\mcG_{l-1}]\\
        \leq& C_1\sum_{l=1}^{m_k/4}d\tau_l^{(1-\delta)_+}\log^2 k/l^2\leq C d\log^3 k.
    \end{align*}
    Then we can apply Lemma \ref{lem:mart_diff_concen} and yield
    \begin{align*}
        &\mbP\Big(\Big|\sum_{l=1}^{m_k/4}(\vect{\xi}_l-\mbE[\vect{\xi}_l|\mcG_{l-1}])\Big|_2\geq C(d^{3/2}\log^2k+d\log^2k)\Big)=O(k^{-\nu d}).
    \end{align*}
    Combining it with \eqref{eq:exp_remain_bound} and \eqref{eq:xi_remain_exp_bound}, we have that
    \begin{equation*}
        \mbP\Big(\Big|\frac{1}{k+1}\sum_{l=1}^{m_k/2}\vect{\xi}_l\Big|_2\geq C\frac{d^{3/2}\log^2k}{k},\cap_{i=n_0}^kE_i\Big)=O((k+1)^{-\nu})
    \end{equation*}
    A similar result holds for the average of $\vect{\xi}_l = \sum_{i\in H_{4l+q}}(\vect{Y}_i-\mbE[\vect{Y}_i])$, where $q=0,2,3$. Substitute it into \eqref{eq:remain_main}, we have that
    \begin{equation*}    
    \begin{aligned}
        &\mbP\Big(\Big|\frac{1}{k+1}\sum_{i=1}^{k+1}(\vect{A}_i-\mbE[\vect{A}_i])(\hat{\vect{\theta}}_{i-1}-\vect{\theta}^*)-\frac{1}{k+1}\sum_{i=1}^{k+1}(\vect{A}_i-\mbE[\vect{A}_i])\tilde{\vect{\theta}}_{i-1}\Big|_2\geq C\Psi^2de_k^2,\cap_{i=n_0}^kE_i\Big)\\
        =&O((k+1)^{-\nu}),
    \end{aligned}
    \end{equation*}
    which proves the lemma.
\end{proof}


\subsection{Proof of Results in Section \ref{sec:bahadur}}

\begin{proof}[Proof of Theorem \ref{thm:asymp_norm}]
	From the proof of Theorem \ref{thm:contam_rate}, as $n_0$ tends to infinity, we can obtain that
	\begin{equation}	\label{eq:normal_term1}
       		 \hat{\vect{\theta}}_{n}-\vect{\theta}^*=\vect{H}^{-1}\frac{1}{n}\sum_{i\notin\mcQ_{n}}\vect{X}_{i}g_{\tau_i}(\vect{Z}_{i}^{\tp}\vect{\theta}^*-b_{i})+O_{\mbP}\Big(\sqrt{d}\tau_n\alpha_n+de_{n-1}^2\log^2n\Big),
   	\end{equation}
	where $\vect{H} = \mbE[\vect{X}\vect{Z}^{\tp}]$. Denote $\epsilon_i = \vect{Z}_i^{\tp}\vect{\theta}^*-b_i$, we first bound the term
	\begin{align*}
		&\Big|\frac{1}{n}\sum_{i\notin\mcQ_n}\vect{X}_i(g_{\tau_i}(\epsilon_i)-\epsilon_i)\Big|_2\stepcounter{equation}\tag{\theequation}\label{eq:normal_term2}\\
		\leq&\Big|\frac{1}{n}\sum_{i\notin\mcQ_n}\vect{X}_i(g_{\tau_i}(\epsilon_i)-\epsilon_i) - \frac{1}{n}\sum_{i\notin\mcQ_n}\mbE[\vect{X}_i(g_{\tau_i}(\epsilon_i)-\epsilon_i)]\Big|_2+\frac{1}{n}\sum_{i\notin\mcQ_n}\Big|\mbE[\vect{X}_i(g_{\tau_i}(\epsilon_i)-\epsilon_i)]\Big|_2\\
		=&\Big|\frac{1}{n}\sum_{i\notin\mcQ_n}\vect{X}_i(g_{\tau_i}(\epsilon_i)-\epsilon_i) - \frac{1}{n}\sum_{i\notin\mcQ_n}\mbE[\vect{X}_i(g_{\tau_i}(\epsilon_i)-\epsilon_i)]\Big|_2+O(\sqrt{d}\tau_n^{-2}),
	\end{align*}
	where the last inequality uses the moment assumption on $\epsilon$ and similar argument as in \eqref{eq:huber_grad_exp}. To bound the first term, we directly compute its variance and use Chebyshev's inequality. More precisely, by equation (20.23) of \cite{billingsley1968convergence}, for arbitrary $i,j$ and each coordinate $l$, there is
	\begin{align*}
		&\Big|\mbE\Big[\big\{X_{i,l}(g_{\tau_i}(\epsilon_i)-\epsilon_i)-\mbE[X_{i,l}(g_{\tau_i}(\epsilon_i)-\epsilon_i)]\big\}\big\{X_{j,l}(g_{\tau_j}(\epsilon_j)-\epsilon_j)-\mbE[X_{j,l}(g_{\tau_j}(\epsilon_j)-\epsilon_j)]\big\}\Big]\Big|\\
		\leq&2\rho^{|i-j|/2}\sqrt{\mbE[X_{i,l}^2(g_{\tau_i}(\epsilon_i)-\epsilon_i)^2]\mbE[X_{j,l}^2(g_{\tau_j}(\epsilon_j)-\epsilon_j)^2]}\\
		\leq&2C_{0}\rho^{|i-j|/2}\frac{1}{\tau^2_i\tau^2_j}\sqrt{\mbE[|\epsilon_i|^6]\mbE[|\epsilon_j|^6]}\leq C_1\frac{\rho^{|i-j|/2}}{\tau^2_i\tau^2_j}.
	\end{align*} 
	Therefore we have that for every unit vector $\vect{v}\in\mbS^{d-1}$, there is
	\begin{align*}
		&\var\Big[\frac{1}{n}\sum_{i\notin\mcQ_n}\vect{X}_i(g_{\tau_i}(\epsilon_i)-\epsilon_i)\Big] \leq \frac{C_1}{n^2}\sum_{i,j=1}^n\frac{\rho^{|i-j|/2}}{\tau_i^2\tau_j^2}\\
		\leq&\frac{C_1}{n^2}\sum_{i=1}^n\frac{2}{\tau_i^2}\Big(\sum_{j=0}^n\rho^{j/2}\Big)\leq C_2\frac{1}{n\tau_n^2},
	\end{align*}
	for some constant $C_2>0$. Therefore, by Chebyshev's inequality, we have that 
	\begin{equation}	\label{eq:normal_term3}
		\Big|\frac{1}{n}\sum_{i\notin\mcQ_n}\vect{X}_i(g_{\tau_i}(\epsilon_i)-\epsilon_i) - \frac{1}{n}\sum_{i\notin\mcQ_n}\mbE[\vect{X}_i(g_{\tau_i}(\epsilon_i)-\epsilon_i)]\Big|_2 = O_{\mbP}\Big(\frac{\sqrt{d}}{(n\tau_n^2)^{2/5}}\Big).
	\end{equation}
	Combining \eqref{eq:normal_term1}, \eqref{eq:normal_term2} and \eqref{eq:normal_term3}, given a unit vector $\vect{v}\in\mbS^{d-1}$, we have
    \begin{align*}
    \vect{v}^{\tp}(\hat{\vect{\theta}}_{n}-\vect{\theta}^*) =& \vect{v}^{\tp}\vect{H}^{-1}\frac{1}{n}\sum_{i\notin\mcQ_{n}}\vect{X}_{i}\epsilon_i + O_{\mbP}\Big(\sqrt{d}\tau_n\alpha_n+de_{n-1}^2\log^2n+\sqrt{d}\tau_n^{-2}+\frac{\sqrt{d}}{(n\tau_n^2)^{2/5}}\Big),\\
    =&\frac{1}{n}\sum_{i\notin\mcQ_{n}}\vect{v}^{\tp}\vect{H}^{-1}\vect{X}_{i}\epsilon_i + o_{\mbP}\Big(\frac{1}{\sqrt{n}}\Big),
    \end{align*}
    where the remainder becomes $o_{\mbP}(1/\sqrt{n})$ under some rate constraints. Here the main term is the average of strict stationary and strong mixing sequence $\vect{v}^{\tp}\vect{H}^{-1}\vect{X}_{i}\epsilon_i$. By Lemma \ref{lem:doukhan_condition}, the conditions in Theorem 1 of \cite{doukhan_etal.1994} is fulfilled. Then we can apply Theorem 1 of \cite{doukhan_etal.1994} and yield
    \begin{equation*}
        \frac{\sqrt{n}}{\sigma_{\vect{v}}}(\hat{\vect{\theta}}_{n}-\vect{\theta}^*)\rightarrow\mcN(0,1),
    \end{equation*}
    where
    \begin{align*}
        \sigma^2_{\vect{v}} =& \sum_{i=-\infty}^{\infty}\cov\big(\vect{v}^{\tp}\vect{H}^{-1}\vect{X}_0\epsilon_0,\vect{v}^{\tp}\vect{H}^{-1}\vect{X}_i^{\tp}\epsilon_i\big)\\
        =&\vect{v}^{\tp}\vect{H}^{-1}\vect{\Sigma}(\vect{H}^{\tp})^{-1}\vect{v},
    \end{align*}
    and $\vect{\Sigma}$ is defined in \eqref{eq:longrun_cov}. Therefore the theorem is proved.
\end{proof}

\begin{lemma}	\label{lem:doukhan_condition}
	Let the random variable $X$ satisfies $\mbE[|X|^{1+\delta}]$ for some $\delta>1$, and $\phi(k)=O(\rho^k)$ (where $\rho<1$). Then we have that
	\begin{equation*}
		\int_0^1\phi^{-1}(u)Q_{X}^2(u)\diff u<\infty,
	\end{equation*}
	where $\phi^{-1}(u)$ is the inverse function of $\phi(k)$, \ie, $\phi^{-1}(u)=\log u/\log\rho$, and $Q_{X}(u)=\inf\{t: \mbP(|X|\geq t)\leq u\}$.
\end{lemma}

\begin{proof}
	By Markov's inequality, we have that
	\begin{align*}
		&\mbP(|X|\geq Q_{X}(u))\leq u\leq \frac{\mbE[|X|^{1+\delta}]}{(Q_X(u))^{1+\delta}},\\
		\Rightarrow&Q_X(u)\leq \Big(\frac{\mbE[|X|^{1+\delta}]}{u}\Big)^{1/(1+\delta)}.
	\end{align*}
	Therefore, we have that
	\begin{align*}
		\int_0^1\phi^{-1}(u)Q_{X}^2(u)\diff u\leq \int_0^1\frac{\log u}{\log\rho}\Big(\frac{\mbE[|X|^{1+\delta}]}{u}\Big)^{2/(1+\delta)}\diff u<\infty,
	\end{align*}
	as long as $\delta>1$, which proves the lemma.
\end{proof}

\begin{proposition} \label{prop:contam_rate_d3}
    Suppose the conditions in Theorem \ref{thm:contam_rate} hold. When $\delta>4$ and $\sqrt{i/\log^3i}=O(\tau_i)$, for every $\nu>0$, there exists constants $C,c>0$ such that
        \begin{equation*}
            \mbP\Big(\cap_{i=n_0}^n\big\{|\hat{\vect{\theta}}_{i}-\vect{\theta}^*|_2\geq C\sqrt{d}e_i\big\}\Big) \geq 1-cn_0^{-\min\{\nu,(\delta-4)/3\}} ,
        \end{equation*}
        where 
        \begin{equation}    \label{eq:conv_g2}
            e_{n} = \alpha_n\tau_{n} + \sqrt{\frac{\log n}{n}}+\frac{1}{\sqrt{d}}(c_0)^{2^{n-n_0}}.
        \end{equation}
     Here $\delta$ is defined in the moment condition in \ref{cond:Bbound}.
\end{proposition}

\begin{proof}
     When $\delta>4$ and $\sqrt{k}=O(\tau_k)$, we denote
     \begin{equation*}
     	e_k = \alpha_k\tau_k + \sqrt{\frac{\log k}{k}}+\frac{1}{\sqrt{d}}(c_0)^{2^{k-n_0}},
     \end{equation*}
     for $k\geq n_0$. We continue from equation \eqref{eq:theta_error_bound} in the proof of Theorem \ref{thm:contam_rate}.

    By ii) of Lemma \ref{lem:huber_grad_concen}, there is
    \begin{equation*}
        \mbP\Big(\Big|\hat{\vect{H}}_{k+1}^{-1}\frac{1}{k+1}\sum_{i=1}^{k+1}\big\{\vect{X}_{i}g(\vect{Z}_{i}^{\tp}\vect{\theta}^*-b_{i})\big\}\Big|_2\geq C_1\sqrt{d}e_{k+1}, \cap_{i=n_0}^kE_i\Big)\leq C_3(k+1)^{-(\delta-1)/3}.
    \end{equation*}
    Then we have that
    \begin{align*}
        \mbP\Big(|\hat{\vect{\theta}}_{k+1}-\vect{\theta}^*|_2\geq \Psi\sqrt{d}e_{k+1},\cap_{i=n_0}^kE_i\Big)\leq 3C_3(k+1)^{-\min\{\nu,(\delta-1)/3\}},
    \end{align*}
    which yields
    \begin{align*}
        \mbP\Big(\cap_{i=n_0}^kE_i\Big)\geq& 1-\sum_{i=n_0+1}^k3C_3i^{-\min\{\nu,(\delta-1)/3\}}\\
        \geq& 1-3C_3\max\Big\{\frac{1}{\nu-1},\frac{2}{\delta-4}\Big\}n_0^{-\min\{\nu-1,(\delta-4)/3\}}.
    \end{align*}
    Therefore, the theorem is proved.
\end{proof}

\begin{proof}[Proof of Proposition \ref{prop:baha_remain}]
    
    By Proposition \ref{prop:contam_rate_d3} and Theorem \ref{thm:asymp_norm}, we know that $\hat{\vect{\theta}}_n$ admits the Bahadur representation 
    \begin{equation*}
        \vect{v}^{\tp}(\hat{\vect{\theta}}_{n}-\vect{\theta}^*) = \vect{v}^{\tp}\vect{H}^{-1}\frac{1}{n}\sum_{i\notin\mcQ_{n}}\vect{X}_{i}(\vect{Z}^{\tp}_{i}\vect{\theta}^*-b_i) + O_{\mbP}\Big(de_{n-1}^2\log^2n+\sqrt{d}\tau_n^{-2}+\frac{\sqrt{d}}{(n\tau_n^2)^{2/5}}\Big),
    \end{equation*}
    where $\alpha_n=0$ and $e_{n-1}$ is given in \eqref{eq:conv_g2}. When $\tau_i=Ci^{\beta}$ for $\beta\geq3/4$ we have that the remainder term has an order $O_{\mbP}(d\log n/n)$ when $n_0\rightarrow\infty$, which proves the proposition.
\end{proof}


\subsection{Proof of Results in Section \ref{sec:online_infer}}

\begin{proof}[Proof of Theorem \ref{prop:cov_est}]
    For simplicity we denote $\vect{\Gamma}_k=\cov(\vect{X}_0\epsilon_0,\vect{X}_{-k}\epsilon_{-k})$, $\vect{Y}_{i,k}=\vect{X}_ig_{\tau_i}(\epsilon_i)\vect{X}^{\tp}_{i-k}g_{\tau_{i-k}}(\epsilon_{i-k})$. Then by \eqref{eq:longrun_cov} we know
    \begin{equation*}
        \vect{\Sigma} = \vect{\Gamma}_0+\sum_{k=1}^{\infty}(\vect{\Gamma}_k+\vect{\Gamma}^{\tp}_k).
    \end{equation*}
    Under event \eqref{eq:huber_ind_event}, for $k=0,\dots,\lceil\lambda\log n\rceil$, we have that
    \begin{align*}
        &\Big\|\frac{1}{n}\sum_{i=\lceil e^{k/\lambda}\rceil}^n\vect{Y}_{i,k}-\frac{1}{n}\sum_{i=\lceil e^{k/\lambda}\rceil}^{n}\vect{X}_{i-k}g_{\tau_{i-k}}(\vect{Z}^{\top}_{i-k}\hat{\vect{\theta}}_{i-k-1}-b_{i-k})\vect{X}^{\top}_{i}g_{\tau_i}(\vect{Z}^{\top}_{i}\hat{\vect{\theta}}_{i-1}-b_{i})\Big\|\\
        =&O\Big(\frac{1}{n}\sum_{i=\lceil e^{k/\lambda}\rceil}^n\tau_{i}|\hat{\vect{\theta}}_{i}-\vect{\theta}^*|_2\Big)=O_{\mbP}\Big(\sqrt{d}\tau_ne_n\Big).
    \end{align*}
    Therefore we have that
    \begin{equation}    \label{eq:sigma_hat_diff}
        \Big\|\hat{\vect{\Sigma}}_n - \Big(\frac{1}{n}\sum_{i=1}^{n}\vect{Y}_{i,0}+\frac{1}{n}\sum_{i=1}^n\sum_{k=1}^{\lceil\lambda\log i\rceil}(\vect{Y}_{i,k}+\vect{Y}_{i,k}^{\tp})\Big)\Big\|=O_{\mbP}\Big(\sqrt{d}\tau_ne_n\Big).
    \end{equation}
    We first prove
    \begin{equation}    \label{eq:trunc_cov}
        \big\|\vect{\Sigma} - \big(\vect{\Gamma}_0+\sum_{k=1}^{\lceil\lambda\log n/2\rceil}(\vect{\Gamma}_k+\vect{\Gamma}_k^{\tp})\big)\big\| = O(n^{-\nu}),
    \end{equation}
    for some $\nu>0$. By equation (20.23) of \cite{billingsley1968convergence}, for each $k$, there is
    \begin{align*}
        \|\vect{\Gamma}_k\|=&\sup_{\vect{u},\vect{v}\in\mbS^{d-1}}\mbE\Big[|\vect{u}^{\tp}\vect{X}_0\epsilon_0\vect{v}^{\tp}\vect{X}_{-k}\epsilon_{-k}|\Big]\\
        \leq&2\sqrt{\phi(|k|)}\sup_{\vect{u},\vect{v}\in\mbS^{d-1}}\sqrt{\mbE[|\vect{u}^{\tp}\vect{X}_0\epsilon_0|^2]\mbE[|\vect{v}^{\tp}\vect{X}_{-k}\epsilon_{-k}|^2]} = O(\rho^{|k|/2}).
    \end{align*}
    Therefore we can obtain that
    \begin{align*}
        &\big\|\vect{\Sigma} - \big(\vect{\Gamma}_0+\sum_{k=1}^{\lceil\lambda\log n/2\rceil}(\vect{\Gamma}_k+\vect{\Gamma}_k^{\tp})\big)\big\|\\
        \leq&2\sum_{k=\lceil\lambda\log n/2\rceil+1}^{\infty}\|\vect{\Gamma}_k\|=O(\sum_{k=\lceil\lambda\log n/2\rceil+1}^{\infty}\rho^{|k|/2})=O(\rho^{\lceil\lambda\log n\rceil/4})=O(n^{-\nu}),
    \end{align*}
    for $\lambda\geq 4\nu/|\log\rho|$, which proves \eqref{eq:trunc_cov}. Next we prove that for $k=0,\dots,\lceil\lambda\log n\rceil$, there holds
    \begin{equation}    \label{eq:gamma_concen}
    \begin{aligned}
        &\Big\|\frac{1}{n}\sum_{i=\lceil e^{k/\lambda}\rceil}^n\vect{Y}_{i,k}-\frac{1}{n}\sum_{i,i-k\notin\mcQ_n}\mbE[\vect{Y}_{i,k}]\Big\| \\
        =& \Big\|\frac{1}{n}\sum_{i,i-k\notin\mcQ_n}\vect{Y}_{i,k}-\frac{1}{n}\sum_{i,i-k\notin\mcQ_n}\mbE[\vect{Y}_{i,k}]\Big\| + O_{\mbP}\big(\tau_n^2\alpha_n\big)\\
        =& O_{\mbP}\Big(\tau_n^2\alpha_n+\sqrt{\frac{d\log^2 n}{n}} + \frac{d\tau_n^2\log^2n}{n}\Big).
    \end{aligned}
    \end{equation}
    By the proof of Lemma \ref{lem:huber_concen_hess}, we know that
    \begin{equation*}
        \Big\|\frac{1}{n}\sum_{i,i-k\notin\mcQ_n}\vect{Y}_{i,k}-\frac{1}{n}\sum_{i,i-k\notin\mcQ_n}\mbE[\vect{Y}_{i,k}]\Big\|\leq 2\sup_{\vect{u},\vect{v}\in\mfN}\Big|\frac{1}{n}\sum_{i,i-k\notin\mcQ_n}\vect{u}^{\tp}(\vect{Y}_{i,k}-\mbE[\vect{Y}_{i,k}])\vect{v}\Big|,
    \end{equation*}
    where $\mfN$ is a $1/4$-net of $\mbS^{d-1}$. For $k=0,\dots,\lceil\lambda\log n\rceil-1$, Denote $\tilde{\mcF}_{k,a}^b=\sigma(\vect{Y}_{i,k},a\leq i\leq b)$, then by \eqref{eq:phi_mixing} we know that
    \begin{equation*}
        |\mbP(B|A)-\mbP(B)|\leq\phi((j-k)_+),
    \end{equation*}
    for all $A\in\tilde{\mcF}_{k,1}^n,B\in\tilde{\mcF}_{k,n+j}^{\infty}$ for all $n,j\geq 0$.
We basically rehash the proof in Lemma \ref{lem:concen_mix} and obtain that
    \begin{equation*}
        \sup_{\vect{u},\vect{v}\in\mfN}\mbP\Bigg(\Big|\frac{1}{n}\sum_{i,i-k\notin\mcQ_n}\vect{u}^{\tp}(\vect{Y}_{i,k}-\mbE[\vect{Y}_{i,k}])\vect{v}\Big|\geq C\Big(\sqrt{\frac{d\log^2n}{n}}+\frac{d\tau_n^2\log^2n}{n}\Big)\Bigg)=O(n^{-(\gamma+2\log 9)d}).
    \end{equation*}
    Here the only difference is that $\var(\eta_l)=O(1)$ in \eqref{eq:var_eta}, and $\vect{u}^{\tp}\vect{Y}_{i,k}\vect{v}$ is bounded by $\tau_n^2$. Therefore \eqref{eq:gamma_concen} can be proved.
    
    Last, we prove that
    \begin{equation}	\label{eq:cov_drift}
    	\Big\|\frac{1}{n}\sum_{i,i-k\notin\mcQ_n}\mbE[\vect{Y}_{i,k}] - \vect{\Gamma}_k\Big\| = O\Big(\alpha_n + \tau_{n}^{-2}\Big)
    \end{equation}
    To see this, we compute that
    \begin{align*}
    	&\Big\|\mbE[\vect{Y}_{i,k}]-\vect{\Gamma}_k\Big\|\\
	\leq&\Big\|\mbE[\vect{X}_ig_{\tau_i}(\epsilon_i)\vect{X}^{\tp}_{i-k}g_{\tau_{i-k}}(\epsilon_{i-k})] - \mbE[\vect{X}_0\epsilon_0\vect{X}_{-k}g_{\tau_{i-k}}(\epsilon_{-k})]\Big\|\\
	&+\Big\|\mbE[\vect{X}_i\epsilon_i\vect{X}^{\tp}_{i-k}g_{\tau_{i-k}}(\epsilon_{i-k})] - \mbE[\vect{X}_0\epsilon_0\vect{X}_{-k}\epsilon_{-k}]\Big\|
    \end{align*}
    For the second term,
    \begin{align*}
    	&\Big\|\mbE[\vect{X}_i\epsilon_i\vect{X}^{\tp}_{i-k}g_{\tau_{i-k}}(\epsilon_{i-k})] - \mbE[\vect{X}_0\epsilon_0\vect{X}_{-k}\epsilon_{-k}]\Big\|\\
	=&\sup_{\vect{u},\vect{v}\in\mbS^{d-1}}\Big| \mbE[\vect{u}^{\tp}\vect{X}_i\epsilon_i\vect{X}^{\tp}_{i-k}\vect{v}\big\{g_{\tau_{i-k}}(\epsilon_{i-k})-\epsilon_{i-k}\big\}] \Big|\\
	\leq&\sup_{\vect{u},\vect{v}\in\mbS^{d-1}}\sqrt{\mbE[|\vect{u}^{\tp}\vect{X}_i\epsilon_i|^2]\mbE[|\vect{X}^{\tp}_{i-k}\vect{v}\big\{g_{\tau_{i-k}}(\epsilon_{i-k})-\epsilon_{i-k}\big\}|^2]}\leq C_{1}\tau_{i-k}^{-2},
    \end{align*}
    for some constants $C_1>0$. Therefore \eqref{eq:cov_drift} is proved.

    Combining \eqref{eq:sigma_hat_diff}, \eqref{eq:trunc_cov}, \eqref{eq:gamma_concen} and \eqref{eq:cov_drift} we have that
    \begin{eqnarray*}
        &&\|\hat{\vect{\Sigma}}_n-\vect{\Sigma}\|\cr
        \leq&&\Big\|\hat{\vect{\Sigma}}_n - \Big(\frac{1}{n}\sum_{i=1}^{n}\vect{Y}_{i,0}+\frac{1}{n}\sum_{i=1}^n\sum_{k=1}^{\lceil\lambda\log i\rceil}(\vect{Y}_{i,k}+\vect{Y}_{i,k}^{\tp})\Big)\Big\|+\big\|\vect{\Sigma} - \big(\vect{\Gamma}_0+\sum_{k=1}^{\lceil\lambda\log n/2\rceil}(\vect{\Gamma}_k+\vect{\Gamma}_k^{\tp})\big)\big\|\cr
        &&+2\sum_{k=0}^{\lceil\lambda\log n/2\rceil}\Big\|\frac{1}{n}\sum_{i=\lceil e^{k/\lambda}\rceil}^n\vect{Y}_{i,k}-\frac{1}{n}\sum_{i,i-k\notin\mcQ_n}\mbE[\vect{Y}_{i,k}]\Big\|+2\sum_{k=0}^{\lceil\lambda\log n/2\rceil}\Big\|\frac{1}{n}\sum_{i,i-k\notin\mcQ_n}\mbE[\vect{Y}_{i,k}] - \vect{\Gamma}_k\Big\| \cr
        &&+2\sum_{k=\lceil\lambda\log n/2\rceil+1}^{\lceil\lambda\log n\rceil}\Big\|\frac{1}{n}\sum_{i=\lceil e^{k/\lambda}\rceil}^n\vect{Y}_{i,k}-\frac{1}{n}\sum_{i,i-k\notin\mcQ_n}\mbE[\vect{Y}_{i,k}]\Big\|+2\sum_{k=\lceil\lambda\log n/2\rceil+1}^{\lceil\lambda\log n\rceil}\Big\|\frac{1}{n}\sum_{i,i-k\notin\mcQ_n}\mbE[\vect{Y}_{i,k}] - \vect{\Gamma}_k\Big\| \cr
        &&+2\sum_{k=\lceil\lambda\log n/2\rceil+1}^{\lceil\lambda\log n\rceil}\Big\| \vect{\Gamma}_k\Big\|\cr
        =&&O_{\mbP}\Big(\sqrt{d}\tau_ne_n+n^{-\nu}+\tau_n^2\alpha_n+\sqrt{\frac{d\log^2 n}{n}} + \frac{d\tau_n^2\log^2n}{n}+\alpha_n + \tau_{n}^{-2}\Big)\cr
        =&&O_{\mbP}\Big(\sqrt{d}\tau_n^2\alpha_n+  \sqrt{d}\tau_n^{-1}+\tau_n\sqrt{\frac{d\log n}{n}}+ \frac{d\tau_n^2\log^2n}{n}\Big),
    \end{eqnarray*}
    which proves the theorem.
\end{proof}

\end{document}

%% file: notations.tex
\newcommand{\Norm}[1]{\left\Vert #1\right\Vert}			
\newcommand{\vect}[1]{\boldsymbol{#1}}				

\newcommand{\tp}{{\top}}						
\newcommand{\cov}{{\rm Cov}}						
\newcommand{\var}{{\rm Var}}						
\newcommand{\diff}{\mathrm{d}}					
\newcommand{\argmin}[1]{\mathop{\rm{argmin}}_{#1}}	
											
\newcommand*\widebar[1]{\,						
	\hbox{%
   		\kern0.01em%
		\vbox{%
			\hrule height 0.5pt		
			\kern0.33ex			
			\hbox{%
				\kern-0.1em		
				\ensuremath{#1}%
				\kern-0.1em		
			}%
		}%
	}%
\,}%
\let\hat\widehat
\let\tilde\widetilde
\let\bar\widebar

\newcommand{\ie}{\mbox{\sl i.e.\;}}					

\newcommand{\mcF}{{\mathcal F}}
\newcommand{\mcG}{{\mathcal G}}

\newcommand{\mcN}{{\mathcal N}}					

\newcommand{\mcP}{{\mathcal P}}
\newcommand{\mcQ}{{\mathcal Q}}
\newcommand{\mcR}{{\mathcal R}}
\newcommand{\mcS}{{\mathcal S}}

\newcommand{\mcU}{{\mathcal U}}

\newcommand{\mbB}{{\mathbb B}}					

\newcommand{\mbE}{{\mathbb E}}					

\newcommand{\mbI}{{\mathbb I}}					

\newcommand{\mbP}{{\mathbb P}}					
\newcommand{\mbR}{{\mathbb R}}					
\newcommand{\mbS}{{\mathbb S}}					

\newcommand{\mfN}{{\mathfrak N}}					

\DeclareMathAlphabet\EuScriptBF{U}{eus}{b}{n}
